\PassOptionsToPackage{table}{xcolor}

\documentclass[11pt]{article}

\usepackage[final]{acl}

\usepackage{times}
\usepackage[T1]{fontenc}
\usepackage[utf8]{inputenc}
\usepackage{microtype}
\usepackage{inconsolata}
\usepackage{graphicx}
\usepackage{paralist}
\usepackage{hyperref}
\usepackage{xcolor}

\usepackage{anyfontsize}

\DeclareFontShape{T1}{ptm}{m}{scit}{<-> ptmrc8t}{}
\DeclareFontShape{T1}{ptm}{b}{scit}{<-> ptmbc8t}{}
\DeclareFontShape{T1}{ptm}{bx}{scit}{<->ssub * ptm/b/scit}{}

\usepackage{cuted}        
\usepackage{caption}      

\usepackage{tikz}
\usetikzlibrary{shapes.geometric,calc,arrows.meta}
\usepackage{fontawesome5} 
\usepackage{pifont}       

\newcommand{\HFsmiley}[1][0.32]{%
\begin{tikzpicture}[scale=#1, line cap=round, baseline={(0,0)}]
  \fill[yellow!80!orange] (-0.95,-0.55) circle (0.32);
  \fill[yellow!80!orange] ( 0.95,-0.55) circle (0.32);
  \draw[orange!85!red, line width=0.5pt] (-0.95,-0.55) circle (0.32);
  \draw[orange!85!red, line width=0.5pt] ( 0.95,-0.55) circle (0.32);
  \fill[yellow!80!orange] (0,0) circle (1);
  \draw[orange!85!red, line width=0.5pt] (0,0) circle (1);
  \draw[black, line width=1.4pt] (-0.42,0.15) .. controls (-0.32,0.40) and (-0.18,0.40) .. (-0.08,0.15);
  \draw[black, line width=1.4pt] ( 0.08,0.15) .. controls ( 0.18,0.40) and ( 0.32,0.40) .. ( 0.42,0.15);
  \draw[black, line width=1.6pt]
    (-0.45,-0.25) .. controls (-0.2,-0.58) and (0.2,-0.58) .. (0.45,-0.25);
\end{tikzpicture}%
}

\usepackage{amsmath,amssymb,amsfonts,amsthm,mathtools}

\DeclareMathOperator*{\argmax}{arg\,max}
\usepackage{bm}

\usepackage{booktabs}
\usepackage{multirow}
\usepackage{array}
\usepackage{colortbl}

\usepackage{microtype}
\usepackage{mathtools}

\definecolor{tableheader}{HTML}{2C3E50}      
\definecolor{tablezebra}{HTML}{F4F6F8}       
\definecolor{tablehighlight}{HTML}{FFF3CD}   
\definecolor{tablerule}{HTML}{1F2D3D}        
\newcommand{\hdr}[1]{\textcolor{white}{\textbf{#1}}}
\newcommand{\hdrrow}{\rowcolor{tableheader}}

\usepackage{algorithm}
\usepackage{algpseudocode}

\usepackage[most]{tcolorbox}
\tcbuselibrary{skins,breakable}

\definecolor{boxBlueBar}{HTML}{2C3E66}     
\definecolor{boxBlueBg}{HTML}{EEF2FA}
\definecolor{boxOrangeBar}{HTML}{C25A1F}   
\definecolor{boxOrangeBg}{HTML}{FCEFE3}
\definecolor{boxGreenBar}{HTML}{2F7A4D}    
\definecolor{boxGreenBg}{HTML}{E8F3EC}
\definecolor{boxTealBar}{HTML}{2A6F7A}     
\definecolor{boxTealBg}{HTML}{E5F1F3}

\newtcolorbox{insightblue}[1]{%
  enhanced, breakable, sharp corners,
  colback=boxBlueBg, colframe=boxBlueBar,
  boxrule=0pt, leftrule=3pt,
  left=10pt, right=10pt, top=7pt, bottom=7pt,
  before upper={\textbf{\textcolor{boxBlueBar}{Principle.~#1}}\par\smallskip},
}
\newtcolorbox{insightorange}[1]{%
  enhanced, breakable, sharp corners,
  colback=boxOrangeBg, colframe=boxOrangeBar,
  boxrule=0pt, leftrule=3pt,
  left=10pt, right=10pt, top=7pt, bottom=7pt,
  before upper={\textbf{\textcolor{boxOrangeBar}{Why.~#1}}\par\smallskip},
}
\newtcolorbox{insightgreen}[1]{%
  enhanced, breakable, sharp corners,
  colback=boxGreenBg, colframe=boxGreenBar,
  boxrule=0pt, leftrule=3pt,
  left=10pt, right=10pt, top=7pt, bottom=7pt,
  before upper={\textbf{\textcolor{boxGreenBar}{Finding.~#1}}\par\smallskip},
}
\newtcolorbox{insightteal}[1]{%
  enhanced, breakable, sharp corners,
  colback=boxTealBg, colframe=boxTealBar,
  boxrule=0pt, leftrule=3pt,
  left=10pt, right=10pt, top=7pt, bottom=7pt,
  before upper={\textbf{\textcolor{boxTealBar}{Takeaway.~#1}}\par\smallskip},
}

\usepackage{hyperref}
\usepackage[nameinlink,capitalize]{cleveref}
\crefname{assumption}{Assumption}{Assumptions}
\Crefname{assumption}{Assumption}{Assumptions}
\crefname{theorem}{Theorem}{Theorems}
\Crefname{theorem}{Theorem}{Theorems}
\crefname{definition}{Definition}{Definitions}
\Crefname{definition}{Definition}{Definitions}
\crefname{appendix}{Appendix}{Appendices}
\Crefname{appendix}{Appendix}{Appendices}

\newtheorem{theorem}{Theorem}
\newtheorem{proposition}{Proposition}
\newtheorem{assumption}{Assumption}
\theoremstyle{definition}
\newtheorem{definition}{Definition}

\newcommand{\OPR}{\ensuremath{\mathrm{OPR}}}
\newcommand{\PTC}{\textsc{PTC}}
\newcommand{\TAS}{\textbf{\textsc{TAS}}}
\newcommand{\EG}{\ensuremath{\mathrm{EG}}}
\newcommand{\PA}{\ensuremath{\mathrm{PA}}}
\newcommand{\AUPRC}{\ensuremath{\mathrm{AUPRC}}}
\newcommand{\Ptheta}{\ensuremath{P_{\theta}}}
\newcommand{\ours}{\TAS}

\title{Right Knowledge, Wrong Answer: Characterizing Parametric Temporal Conflict in Open-Weight Language Models}

\newcommand{\linkicon}[2][11pt]{\raisebox{-0.32\height}{\resizebox{!}{#1}{#2}}}
\newcommand{\linkbadge}[3]{%
  \href{#1}{\shortstack[c]{#2\\[1pt]{\fontsize{7}{8}\selectfont\textsf{#3}}}}}

\author{
\fontsize{9.5}{11}\selectfont\textbf{Elias Hossain}$^{1\dagger}$,
\fontsize{9.5}{11}\selectfont\textbf{Sourav Saha}$^{1\dagger}$,
\fontsize{9.5}{11}\selectfont\textbf{Tasfia Nuzhat Ornee}$^{1\dagger}$,
\fontsize{9.5}{11}\selectfont\textbf{Sanjeda Sara Jennifer}$^{1}$
\\
\fontsize{9.5}{11}\selectfont\textbf{Umesh Chandra Biswas}$^{2}$,
\fontsize{9.5}{11}\selectfont\textbf{Shubhashis Roy Dipta}$^{3}$,
\fontsize{9.5}{11}\selectfont\textbf{Rajib Rana}$^{4}$,
\fontsize{9.5}{11}\selectfont\textbf{Niloofar Yousefi}$^{1}$
\\[0.6em]
{\small $^{1}$University of Central Florida}
\qquad
{\small $^{2}$Mississippi State University}
\\
{\small $^{3}$University of Maryland, Baltimore County}
\qquad
{\small $^{4}$University of Southern Queensland}
\\[0.7em]
{\small $\dagger$ Equal contribution}
\\[0.6em]
{\small\texttt{\textbf{mdelias.hossain@ucf.edu}}}
\\[0.8em]
\linkbadge{https://eliashossain001.github.io/tas/}%
  {\linkicon{\textcolor{teal!65!black}{\faGlobe}}}{Project}%
\hspace{1.8em}%
\linkbadge{https://huggingface.co/datasets/EliasHossain/ptc-benchmark}%
  {\linkicon[8pt]{\HFsmiley}}{Dataset}%
\hspace{1.8em}%
\linkbadge{https://github.com/eliashossain001/temporal-attractor-steering}%
  {\linkicon{\faGithub}}{Code}%
\\[1.2em]
}

\begin{document}
\maketitle

\begin{abstract}
Language models are trained on text spanning many years, so a model can encode both an outdated fact and the newer fact that superseded it. When it answers with the outdated one, is the newer fact missing, or merely dispreferred? We define \emph{Parametric Temporal Conflict} (\PTC{}): the newer fact is present and elicitable, yet the default forward pass favors the older one. On a deterministically verified $8{,}746$-record benchmark of single-holder Wikidata position-holder transitions, we characterize \PTC{} in four open-weight base LMs across three families (Qwen-2.5-1.5B/7B, Mistral-7B-v0.3, Llama-3.1-8B). Three findings indicate a representational preference rather than a knowledge gap: a date-prefix prompt recovers the newer fact in $61$--$81\%$ of \PTC{} cases; single-layer activation patching flips the answer in $72$--$85\%$, inside a model-specific upper-layer causal region; and a single-direction residual-stream edit flips it substantially more often than a norm-matched random direction of equal magnitude ($29$--$62\%$ vs.\ $18$--$31\%$). All recovery figures are measured on \emph{oracle-identified} instances, because automatic conflict detection remains unsolved (AUROC $0.47$--$0.66$), and the steering edit we use is an instrument rather than a proposed method. We release the benchmark, code, and statistics.
\end{abstract}

\section{Introduction}
\label{sec:intro}

Facts change, and language models are trained on text spanning many years. A model may therefore encode both an earlier fact and its later replacement: an older CEO and a newer one, an earlier head of state and a later one. Asked who currently holds such a position, the model often answers with the outdated fact. The standard remedies all treat this as \emph{missing information}: retrieval injects the newer fact from an external store, while knowledge editing and continual learning write it into the weights. Each assumes the model does not already know the answer.

That assumption is sometimes false. Adding a temporal cue such as \emph{``As of 2024, who is\dots''} frequently elicits the newer fact from the very same weights \citep{vu2024freshllms,zhao2024set,herel2024time}, indicating it was encoded all along. The failure is then not one of knowledge but of \emph{preference}: two competing facts coexist in parametric memory and the default forward pass surfaces the wrong one. The distinction matters, because a preference and a gap call for different remedies, and a benchmark that scores only temporal correctness cannot tell them apart.

We name and formalize this phenomenon as \emph{Parametric Temporal Conflict} (\PTC{}, \cref{def:ptc}): the newer fact is recoverable from the model, yet the default response favors the outdated one. To separate genuine conflict from simple absence, we pair the definition with a knowledge-recovery filter that admits a record only when the model demonstrably can produce the newer fact. Prior work does not target this setting. Temporal benchmarks \citep{chen2021dataset,zhang2021situatedqa,vu2024freshllms,chu2024timebench,uddin2025unseentimeqa,zhu2025evolvebench} score temporal correctness without asking whether a wrong answer means the knowledge is \emph{missing} or merely \emph{dispreferred}, and continual learning \citep{dhingra2022time,zhao2022can,kim2024carpe,fierro2024mulan} updates weights without a test-time means of choosing between coexisting versions.

Our thesis is that the outdated answer is a \emph{localized, direction-specific} property of the residual stream rather than a knowledge gap. We investigate four questions: \textbf{(RQ1)} is the newer fact absent or merely dispreferred? \textbf{(RQ2)} is the dispreference localized to identifiable layers? \textbf{(RQ3)} is a specific direction in activation space causally implicated in it? and \textbf{(RQ4)} does it persist across the evaluated model sizes and families?

RQ2 and RQ3 are causal questions: answering them requires intervening on the model, not merely observing it. We therefore analyze \PTC{} with two causal interventions. Activation patching asks \emph{where} a temporal-cued state can change the answer. A single-direction residual-stream edit, \textbf{\TAS{}} (Temporal Attractor Steering), asks \emph{whether a single direction is causally implicated in the preference}: if editing along the conflict direction flips the answer substantially more often than a norm-matched random direction does, that direction is implicated in shifting it.

\TAS{} is an analysis instrument, not a proposed system, and we state plainly what we do \emph{not} claim. It confers no deployable advantage: a one-line date-prefix prompt recovers more \PTC{} cases (\cref{tab:baselines}), \TAS{} only ties a matched ITI baseline on held-out data (\cref{tab:tas-vs-iti}), and its detector-gated variant is a negative result (\S\,\emph{Can \PTC{} Be Identified Automatically?}). None of this weakens the mechanistic findings, and the distinction is the one this paper turns on. Prompting shows \emph{that} the newer fact can be recovered. The causal edit shows \emph{where} the preference is represented, and that a \emph{single direction} is causally implicated in shifting it. Prompting can reveal neither. We make four contributions:

\begin{itemize}
    \item \textbf{A new phenomenon: Parametric Temporal Conflict (RQ1).}
    We identify and formalize \PTC{} (\cref{def:ptc}), a failure mode in which
    a superseding fact is encoded but dispreferred, and pair it with a
    knowledge-recovery filter that separates conflict from absence. A
    date-prefix prompt recovers the newer fact in $61$--$81\%$ of \PTC{} cases
    (\cref{tab:baselines}): the knowledge is present, not missing.

    \item \textbf{A verified multi-model benchmark (RQ4).}
    We release an $8{,}746$-record deterministically verified benchmark over
    five single-holder Wikidata position-holder relations with a released
    stratified audit sample, and characterize \PTC{} across four open-weight base
    LMs in three families, with bootstrap confidence intervals and paired
    tests throughout.

    \item \textbf{Mechanistic localization (RQ2).}
    Single-layer activation patching flips the outdated answer on $72$--$85\%$
    of cases within a model-specific upper-layer causal region
    (\cref{tab:locator}), causally implicating an identifiable band of layers
    rather than a diffuse computation.

    \item \textbf{Direction-specific representation evidence (RQ3).}
    Holding magnitude fixed and varying only the direction, a
    single-direction edit recovers $29$--$62\%$ of oracle-identified \PTC{}
    cases where a norm-matched random direction recovers only $18$--$31\%$,
    and the same direction stays effective at other layers
    (\cref{tab:layer-localization}). This causally implicates the direction,
    rather than a unique layer, in mediating the preference.
\end{itemize}

Filtered \PTC{} rates increase across our four models ($0.041$ to $0.103$), but those models differ in family, tokenizer, and cutoff as well as size, so we report this as an observed trend, not a controlled scaling result. Together these results characterize outdated-answer failures, in the models and relations we evaluate, as a localized, direction-specific property of parametric memory rather than an absence of knowledge.

\section{Preliminaries and Problem Formulation}
\label{sec:ptc}


We use three symbols throughout: $\theta$ is a pretrained autoregressive LM,
$\Ptheta(y\mid x)$ its probability for completion $y$ given prompt $x$, and
$h_\ell(x)$ its residual-stream representation at layer $\ell$. A
\emph{temporal fact quadruple}
$\mathcal{Q}=(q,a_{\mathrm{old}},a_{\mathrm{new}},t_{\mathrm{update}})$ pairs a
query $q$ with an answer $a_{\mathrm{old}}$ correct before
$t_{\mathrm{update}}$ and a superseding answer $a_{\mathrm{new}}$ correct
after it.

\begin{definition}[Parametric Temporal Conflict]
\label{def:ptc}
An LLM $\theta$ exhibits \PTC{} on $\mathcal{Q}$ if
\begin{align}
\Ptheta(a_{\mathrm{old}} \mid q) &> \Ptheta(a_{\mathrm{new}} \mid q), \\
\Ptheta(a_{\mathrm{new}} \mid q, c_{\mathrm{temp}}) &\gg
\Ptheta(a_{\mathrm{old}} \mid q, c_{\mathrm{temp}}),
\end{align}
where $c_{\mathrm{temp}}$ is a temporal cue and both answers are verified
as parametrically plausible completions.
\end{definition}

\paragraph{Knowledge-recovery filter.}
To ensure that \PTC{} reflects conflict rather than knowledge absence, as
benchmarks certify a property before probing it \citep{nazi2026omni}, we
retain a record only if
$\overline{\log P}_{\theta}(a_{\mathrm{new}}\mid q,c_{\mathrm{temp}})
\ge \tau_{\mathrm{rec}}$, with $\tau_{\mathrm{rec}}=-3$ applied to the
length-normalized answer log-probability (our pre-specified code default,
fixed before all downstream analysis and not tuned for Recovery). Records
below this threshold are marked \emph{knowledge-absent for that model} and
excluded from filtered \PTC{} rates, which we report as \PTC{} count over
retained records. We report both raw and filtered rates throughout
(\cref{fig:phase1-headline}). A sweep over $\tau_{\mathrm{rec}}\in\{-2.5,
-3.0,-3.5\}$ leaves every qualitative conclusion unchanged (the four-model
ordering is preserved at Spearman $\rho=1$); absolute rates vary smoothly
with coverage. Full sensitivity analysis, counts, and confidence intervals
are in Appendix~G.

\paragraph{Frequency-recency hypothesis.}
Older facts may form stronger attractors through greater cumulative
pre-training exposure \citep{kim2024carpe,fierro2024mulan}; we test this only
indirectly, by varying model scale and family, and defer within-model
checkpoint analysis to future work. That temporal signals are separable in
weight space \citep{nylund2024time} motivates analyzing \PTC{} at the
activation level.

\section{Causal Analysis Method}
\label{sec:tas}

To ask \emph{where} the outdated preference lives and \emph{whether one
direction moves it}, observation is not enough: both are causal questions,
so we must intervene on the model. We use two interventions on the residual
stream, both instruments for the experiments below rather than proposed
systems. \emph{Activation patching} asks \emph{where} a
temporal-cued state can causally change the answer. A \emph{single-direction
edit}, \TAS{} (Temporal Attractor Steering), asks \emph{whether a single
direction is causally implicated in the preference}. Full algorithms are in Appendix~B and
implementation details in Appendix~O.

\paragraph{Locating the preference.}
On held-out verified \PTC{} tuples we compute, for each layer $\ell$, the
answer-flip rate $\mathrm{AFR}(\ell)$ induced by patching $h_\ell$ from the
temporal-prompt forward pass into the standard-prompt forward pass at the
last prompt token. The peak conflict-critical layer is $\ell^{*} =
\argmax_\ell \mathrm{AFR}(\ell)$. We also report the contiguous plateau of
layers whose $\mathrm{AFR}$ falls within $0.05$ of peak, which is what
distinguishes a causal \emph{region} from a unique layer.

\paragraph{Testing directionality.}
At $\ell^{*}$ we form the direction
\[
\Delta = \mathrm{mean}_q\bigl[h_{\ell^*}(q,c_{\mathrm{temp}}) - h_{\ell^*}(q)\bigr],
\]
the mean shift the temporal cue induces, averaged over verified \PTC{}
instances, and intervene as
\begin{equation}
\label{eq:tas-update}
h'_{\ell^*}(q) = h_{\ell^*}(q) + \alpha \cdot c \cdot \Delta,
\end{equation}
with strength $\alpha>0$ and detector score $c$. Contrasting this edit with
a norm-matched random direction is our test for RQ3. We average $\Delta$
per relation; global and per-domain variants are statistically comparable
on held-out data (Appendix~K), and $\alpha$ is selected on validation
(Appendices~D and~E).

\paragraph{Detection.}
A linear probe on $h_{\ell^*}(q)$ scores whether a query exhibits \PTC{},
gating the edit at a threshold $\tau$. All learning and evaluation use one
\emph{subject-disjoint} split ($60/15/10/15$ by subject, stratified by
relation), so no subject, and hence no near-duplicate transition, crosses
partitions: the probe is fit on train, calibrated on calibration,
thresholded on validation, and reported on test. We include the gate only as
a diagnostic. Under this protocol the detector is near chance, so it is a
negative result rather than a working system, and our mechanistic claims
instead rest on the always-on evaluation over verified \PTC{} records. Full
protocol, leakage audit, and operating points are in Appendix~C.

\paragraph{Geometric motivation for a bounded edit.}
A local sufficient condition gives a geometric picture of when the edit
should flip the preference. Let the \emph{margin} be the amount by which the
model prefers $a_{\mathrm{new}}$ over $a_{\mathrm{old}}$ at a given state.
\PTC{} means the margin is negative under the plain prompt and positive
under the temporal cue, so if it increases along $\Delta$, a bounded step
carries it past zero. This is a motivation, not a guarantee: the condition is
local and sufficient rather than necessary, holds only where a first-order
approximation is valid, and presumes nothing about detector reliability or a
privileged layer. It is compatible with the partial Recovery we observe
($0.32$--$0.54$), which it neither predicts nor explains. Assumptions and
proofs are in Appendix~A.

\section{Benchmark and Experimental Setup}
\label{sec:setup}

\paragraph{Benchmark.}
The benchmark is built in four steps: source, verify, filter, render.
\emph{(1) Source.} We mine superseding position-holder updates from Wikidata
by SPARQL over five single-valued position-holder relations, taking
consecutive $(a_{\mathrm{old}}, a_{\mathrm{new}})$ transitions within each
subject. \emph{(2) Verify.} A transition is admitted only if it passes four
deterministic checks: \emph{abutting validity windows} (a genuine hand-off,
not a vacancy), \emph{minimum tenure} (so the outdated fact could establish
in pre-training text), \emph{single-holder and vacancy guards}, and
\emph{canonical English-label resolution}. $t_{\mathrm{update}}$ is the new
statement's start, restricted to $[2010,2024]$.

\emph{(3) Filter.} A verified record is admitted \emph{for a given model}
only if it also passes that model's knowledge-recovery filter; because this
step is model-dependent, raw and filtered \PTC{} rates are reported as a
pair throughout (\cref{fig:phase1-headline}). \emph{(4) Render.} Each of the
resulting $8{,}746$ records pairs its
$(a_{\mathrm{old}}, a_{\mathrm{new}}, t_{\mathrm{update}})$ triple with a
standard prompt (\emph{``Who is the $\langle$role$\rangle$ of
$\langle$subject$\rangle$?''}) and a temporal prompt (\emph{``As of
$\langle t \rangle$, who was the $\langle$role$\rangle$ of
$\langle$subject$\rangle$?''}), with $t$ inside the new fact's validity
window. Here ``verified'' means every record passes the deterministic checks,
not that any record was hand-checked; to enable independent verification we
release a stratified $200$-record audit sample with annotation guidelines. Exact
queries, thresholds, the audit protocol, and rejection breakdowns are in
Appendix~H.

\paragraph{Relation distribution.}
The benchmark spans five relations across three domains and is unbalanced,
reflecting each position-holder population at source: head of state (P35,
$n{=}183$) and CEO (P169, $n{=}208$) are smallest; head of government (P6),
head coach (P286), and chairperson (P488) largest. The smallest two carry
the largest per-record \PTC{} signal, so we report relation-level results
rather than aggregate rates alone (Appendix~I).

\paragraph{Models.}
We evaluate four open-weight base (non-instruct) LMs across three
families. \textbf{Qwen-2.5-1.5B}/\textbf{7B} share architecture, data, and
tokenizer, isolating parameter count; \textbf{Mistral-7B-v0.3} and
\textbf{Llama-3.1-8B} add two families at the matched $7$--$8$B class,
a limited matched-size cross-family comparison, not a separation of scale
from family: corpus, tokenizer, architecture, and cutoff remain
uncontrolled. We exclude instruction-tuned variants
because chat alignment confounds the parametric-memory signal
\citep{herel2024time}.

\paragraph{Operational metrics.}
Every metric compares answers by $s(a,x)$, the \emph{length-normalized}
log-probability of $a$ given $x$ (the per-token mean of $\log \Ptheta$, which
keeps answers of different lengths comparable). The first four differ mainly
in their \emph{denominator}, which is what to watch when reading the tables:
\begin{compactitem}
    \item \textbf{\OPR{}} (outdated-preference rate): of \emph{all} records,
    the fraction with $s(a_{\mathrm{old}},q)>s(a_{\mathrm{new}},q)$.
    \item \textbf{Raw \PTC{}}: of \emph{all} records, the fraction meeting
    both conditions of \cref{def:ptc}.
    \item \textbf{Filtered \PTC{}}: the same count over \emph{only records
    passing the knowledge-recovery filter}. This is our headline rate, as it
    excludes records the model cannot recover.
    \item \textbf{Recovery}: of records \emph{already known} to be \PTC{},
    the fraction flipping to $a_{\mathrm{new}}$ under an intervention
    (temporal cue, prompt baseline, or \TAS{}).
    \item \textbf{\EG{}} (elicitation gap): mean $s(a_{\mathrm{new}},q,c_{\mathrm{temp}})-s(a_{\mathrm{new}},q)$.
    \item \textbf{\PA{}} (preservation accuracy): fraction of matched non-conflict records whose argmax over $\{a_{\mathrm{old}},a_{\mathrm{new}}\}$ is unchanged.
\end{compactitem}
The full scoring equation, metric definitions, and ITI construction are in
Appendix~D.

\paragraph{Statistical protocol.}
All confidence intervals are $95\%$ percentile bootstraps ($B{=}10{,}000$,
fixed seed) computed by one seeded routine from the released per-instance
decisions, so every interval is reproducible rather than transcribed.
Paired comparisons use an exact McNemar test; the edit shifts the
outdated-answer preference at $p<10^{-9}$ on every model, establishing a
causal effect (its size relative to prompting is in \cref{tab:baselines}).
\section{Results}
\label{sec:results}

\subsection{Is the Newer Fact Absent or Dispreferred? (RQ1)}
\label{sec:results-screening}

Before asking why the newer fact loses, we must show it is present. Screening
under both prompts and keeping only records passing the knowledge-recovery
filter isolates cases where the model can produce the newer fact under a
temporal cue yet prefers the older one. Three patterns emerge
(\cref{fig:phase1-headline}). First, \OPR{} is nearly constant across the
$4$ evaluated models ($0.530$--$0.544$). Second, \EG{} is positive on each
of them ($+0.090$ to $+0.180$ nats), indicating that the temporal cue moves
probability mass toward $a_{\mathrm{new}}$ in every model. Third,
filtered \PTC{} increases across the evaluated model set
($0.041\!<\!0.071\!<\!0.085\!<\!0.103$; $95\%$ bootstrap CIs $[.033,.049]$,
$[.063,.080]$, $[.078,.093]$, $[.096,.111]$). Only the endpoints are clearly
separated (adjacent Qwen-7B and Mistral intervals overlap and span $2$
families), and the models differ in tokenizer, corpus, and cutoff as well as
size, so we read this as an observed trend, not a controlled effect of
capacity. The filter is essential here: raw \PTC{} partly reflects coverage
differences (Qwen-1.5B retains $28.2\%$ vs.\ Llama-3.1-8B's $66.1\%$), and
only the filtered rate separates the $4$ models.

\begin{figure*}[t]
    \centering
    \includegraphics[width=0.72\textwidth]{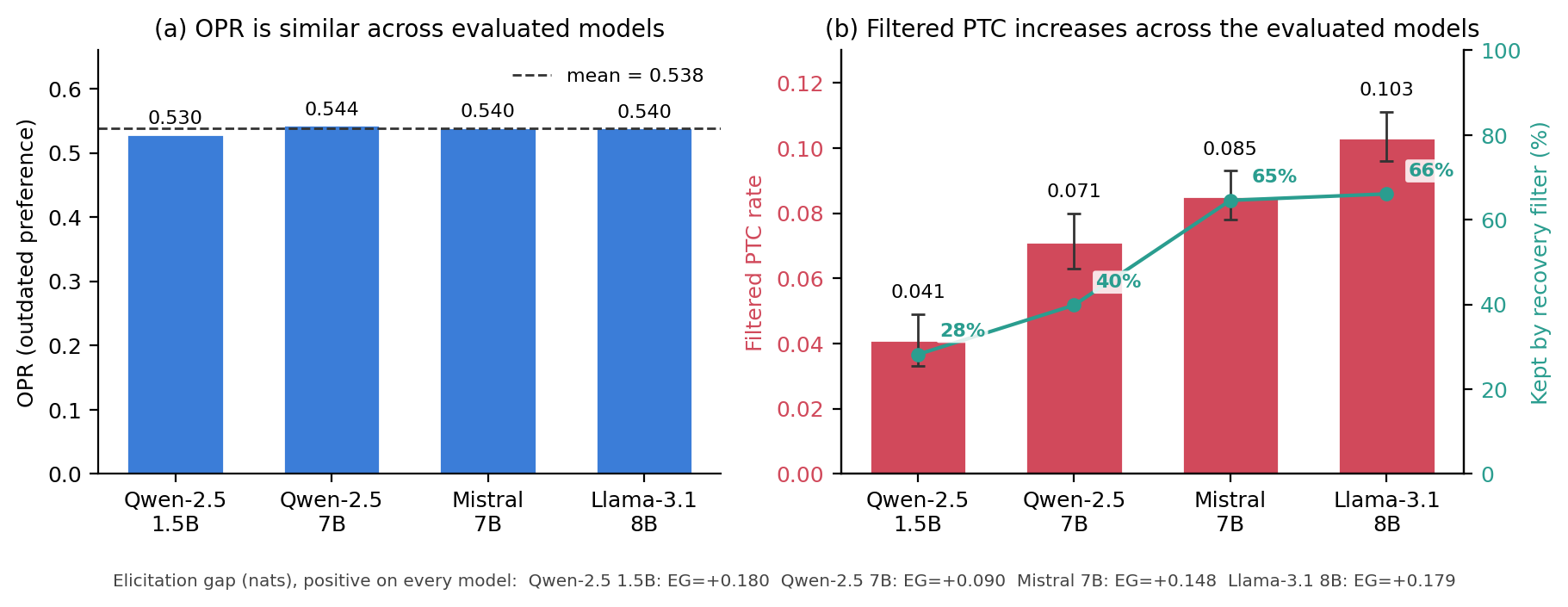}
    \caption{Phase~1 screening on the $8{,}746$-record benchmark.
    \textbf{(a)} The outdated-preference rate (\OPR{}) is nearly constant
    across the four evaluated models (all within $[0.530, 0.544]$).
    \textbf{(b)} The filtered \PTC{} rate \emph{increases across the
    evaluated model set} ($0.041 \to 0.103$); the models differ in family
    and cutoff as well as size, so this is an observed trend, not a
    controlled scaling result. \emph{Kept} is the
    fraction passing the knowledge-recovery filter
    ($\overline{\log P}(a_{\mathrm{new}}\mid q, c_{\mathrm{temp}}) \ge -3$;
    $28.2$--$66.1\%$). The elicitation gap is positive on every model
    ($\EG{} \in [+0.090, +0.180]$ nats).}
    \label{fig:phase1-headline}
\end{figure*}

That the filtered rate is lowest on the smallest model is at least
inconsistent with \PTC{} being a small-model artefact, which motivates the
mechanistic analysis below. Having established that the newer fact is
recoverable, we next ask where this preference is represented.

\subsection{Which Facts Conflict, and When?}
\label{sec:results-relations}

Where \PTC{} concentrates is informative about its origin, though these
analyses are observational, so we read them as associations, not causes.
Per-relation \PTC{} is uneven: head of state (P35) is strongest in every
model ($0.066$--$0.235$, $2$--$3\times$ the next relation), CEO (P169) shows
the next-largest increase, and the three high-volume relations contribute
most records but moderate per-record signal.

Bucketing by $t_{\mathrm{update}}$ year gives an inverted-U peaking at
$[2020, 2021]$ for the three $7$--$8$B models. Llama-3.1-8B alone retains a
substantial $2022$--$2024$ rate, consistent with its later cutoff, which
suggests a pre-training-recency association rather than a representational
one. The politics~$>$~corporate~$>$~sports ordering tracks the relation
distribution (P35 and P6 are both political), so we do not read it as
domain-specific structure. Full breakdowns are in Appendix~I.

\subsection{Where Is the Preference Represented? (RQ2)}
\label{sec:results-locator}

If the outdated answer is a representational preference, we should be able to
say \emph{where} conflict-relevant information becomes causally effective. We
ask, layer by layer, whether transplanting a temporal-cued state suffices to
change the answer (\cref{tab:locator}). Single-layer single-token activation
patching, drawn from the temporal-prompt forward pass and injected into the
standard-prompt forward pass at the last prompt token, flips the answer
argmax on $72.3$--$85.1\%$ of instances at the \emph{peak} conflict-critical
layer $\ell^{*}$. Patching identifies a family-specific \emph{upper-layer
causal region} rather than a unique layer: the peak sits inside a plateau of
contiguous near-peak layers. Recomputing AFR on held-out test subjects alone
reproduces the full-data peaks ($0.86$/$0.87$/$0.82$ for
Qwen-7B/Mistral/Llama; Qwen-1.5B $0.60$ at $n{=}10$), so the localization is
not an artefact of pooling splits.

\begin{table}[t]
\centering
\small
\setlength{\tabcolsep}{4pt}
\renewcommand{\arraystretch}{1.0}
\begin{tabular}{lccccc}
\toprule
\hdrrow
\hdr{Model} & \hdr{$L$} & \hdr{$\ell^*$} &
\hdr{Peak AFR} & \hdr{Plateau} & \hdr{Range} \\
\midrule
Qwen-2.5-1.5B   & 28 & 23 & 0.723 & 6  & 22--27 \\
\rowcolor{tablezebra}
Qwen-2.5-7B     & 28 & 23 & \cellcolor{tablehighlight}\textbf{0.851} & 6  & 22--27 \\
Mistral-7B-v0.3 & 32 & 30 & 0.824 & \cellcolor{tablehighlight}\textbf{14} & 18--31 \\
\rowcolor{tablezebra}
Llama-3.1-8B    & 32 & 31 & 0.816 & 9  & 23--31 \\
\bottomrule
\end{tabular}
\caption{Stage~2 layer localisation (full-data patching diagnostic; the
held-out test AFR reproduces these peaks). $L$: model depth;
$\ell^{*}$: \emph{peak} conflict-critical layer; \emph{Peak AFR}:
answer-flip rate at $\ell^{*}$ for a single-token activation patch from
the temporal-prompt into the standard-prompt forward pass;
\emph{Plateau}: contiguous layers with AFR within $0.05$ of peak (width
stable under $0.03$/$0.10$ too); \emph{Range}: layer indices of that
plateau. The two Qwens share a tight $6$-layer mid-network plateau,
Mistral's is unusually wide ($14$ layers), and Llama's ($9$) sits at the
final decoder layers.}
\label{tab:locator}
\end{table}

The important column in \cref{tab:locator} is \emph{Plateau}: every model has
several contiguous near-peak layers, so patching implicates a causal
\emph{region}, not one privileged layer.

\paragraph{Patch-localization is not edit-localization.}
Where a fixed learned edit remains effective is a different question from
where a temporal-cued state can causally alter the answer, and the two
profiles diverge (\cref{tab:layer-localization}). Applying the same
train-only $\Delta_{\ell^*}$ at an early off-peak layer (patching AFR near
zero, $0.01$--$0.09$, outside every plateau) recovers \emph{at least as
many} held-out test cases as steering at $\ell^*$ on all four models, and
\emph{significantly more} on Mistral ($0.451$ vs.\ $0.324$; paired McNemar
$p{=}0.02$). Across evaluated layers, held-out Recovery is
\emph{negatively} correlated with patching AFR (Spearman $\rho=-0.67$ to
$-0.98$) but \emph{positively} with the geometric coherence of the temporal
shift ($\rho=0.89$--$0.98$): coherence is highest at early layers
($\cos\approx0.95$) and lowest at $\ell^{*}$ ($\cos\approx0.45$). This is
consistent with a fixed edit being \emph{portable} across the layers we
evaluate, while $\ell^{*}$ is where the patch \emph{reads out} most reliably
and selectively (highest PA). We therefore claim direction-specificity, not
layer-uniqueness; cheap importance signals likewise mistrack causal effect
across tokens \citep{jiang2026rock}. Curves, geometry, and paired tests are
in Appendices~F
and~J.

\begin{table}[t]
\centering
\small
\setlength{\tabcolsep}{4pt}
\renewcommand{\arraystretch}{1.0}
\resizebox{\columnwidth}{!}{%
\begin{tabular}{lcccccc}
\toprule
\hdrrow
\hdr{Model} & \hdr{$\ell^*$} & \hdr{Peak AFR} & \hdr{plateau} &
\hdr{off-peak} & \hdr{$\ell^*$ steer} & \hdr{off-peak steer} \\
 & & & (width) & (AFR) & Rec. & Rec. \\
\midrule
Qwen-2.5-1.5B   & 23 & 0.723 & 22--27 (6)  & 2 (0.09) & 0.469 & 0.594 \\
Qwen-2.5-7B     & 23 & 0.851 & 22--27 (6)  & 2 (0.05) & 0.538 & 0.667 \\
Mistral-7B-v0.3 & 30 & 0.824 & 18--31 (14) & 3 (0.01) & 0.324 & 0.451 \\
Llama-3.1-8B    & 31 & 0.816 & 23--31 (9)  & 3 (0.03) & 0.356 & 0.356 \\
\bottomrule
\end{tabular}}
\caption{Patch- vs.\ edit-localization. \emph{Peak AFR} and \emph{plateau}
are full-data patching diagnostics (\cref{tab:locator}); \emph{off-peak}
is the early control layer, its (near-zero) patching AFR in parentheses.
\emph{$\ell^*$ steer} / \emph{off-peak steer} are held-out test Recovery
from applying the same train-only $\Delta_{\ell^*}$ at $\ell^{*}$ vs.\ at
the off-peak layer (validation-selected $\alpha$; portability experiment,
not a re-selection of $\ell^{*}$). Off-peak steering matches or exceeds
$\ell^{*}$ steering everywhere despite near-zero off-peak patching AFR.}
\label{tab:layer-localization}
\end{table}

\subsection{Can a Single Direction Move the Preference? (RQ3)}
\label{sec:results-e2e}

Having located \emph{where} the temporal-cued state most influences the
answer, we ask whether a
single direction is implicated in it. This is our primary held-out evidence: the
direction is built on \emph{train} only, $\alpha$ selected on
\emph{validation}, and the edit evaluated once on untouched \emph{test}
subjects. Applied \emph{always-on}, it moves $0.32$--$0.54$ of verified
\PTC{} cases to the newer fact while leaving $0.79$--$0.85$ of clean
non-conflict answers intact (\cref{tab:tas-main}). \emph{Always-on} means
every steered record is already known to be a \PTC{} instance, so this is an
oracle-gated diagnostic: it measures how much of the preference one
direction accounts for, and upper-bounds any pipeline that must first find
the conflicts itself (\cref{tab:detector-gated}). Why the effect is largest
on Qwen-2.5-7B is analysed in Appendix~L.

\begin{table}[t]
\centering
\small
\setlength{\tabcolsep}{5pt}
\renewcommand{\arraystretch}{1.0}
\begin{tabular}{lcccc}
\toprule
\hdrrow
\hdr{Model} & \hdr{$n$} & \hdr{$\alpha^{*}$} &
\hdr{Recovery} & \hdr{PA (clean)} \\
\midrule
Qwen-2.5-1.5B   & 32 & 2 & 0.469\,{\scriptsize[.31,.63]} & 0.820 \\
Qwen-2.5-7B     & 39 & 2 & \cellcolor{tablehighlight}\textbf{0.538}\,{\scriptsize[.38,.69]} & 0.795 \\
Mistral-7B-v0.3 & 71 & 4 & 0.324\,{\scriptsize[.23,.44]} & 0.850 \\
Llama-3.1-8B    & 87 & 6 & 0.356\,{\scriptsize[.25,.46]} & 0.785 \\
\bottomrule
\end{tabular}
\caption{\textbf{Always-on \TAS{}} (V2 per-relation $\Delta$) on the
held-out subject-disjoint test split. Direction built on \emph{train}
only; $\alpha^{*}$ selected on \emph{validation}; evaluated once on the
$n$ test verified-\PTC{} records and the $200$ clean controls. Recovery
carries a $95\%$ percentile-bootstrap CI. The detector-gated pipeline and
corrected detector metrics are in \cref{tab:detector-gated}, and the
paired comparison with ITI in \cref{tab:tas-vs-iti}.}
\label{tab:tas-main}
\end{table}

\cref{tab:tas-main} shows one direction accounts for roughly a third to a
half of the preference while leaving most clean answers intact; the
random-direction comparison that makes this evidence for
\emph{direction-specificity} comes below.

\subsection{Can \PTC{} Be Identified Automatically?}
\label{sec:results-tau}

Every recovery number so far assumes an oracle told us which queries are in
conflict. Can the model find them itself? This experiment tests that, and
the answer is a negative result. Under the subject-disjoint
protocol the detector is a weak discriminator,
and its quality \emph{varies across models}: held-out AUROC is only
$0.47$--$0.66$ (near chance for Qwen-2.5-1.5B, \cref{tab:detector-gated})
and AUPRC $0.03$--$0.10$ at the true low prevalence ($2$--$7\%$); good
calibration reflects predicting the base rate, not separating conflicts. At
$\tau{=}0.15$ the gate fires on $0.5$--$10.9\%$ of queries yet catches just
$3$--$23\%$ of true conflicts, so detector-gated Recovery collapses to
$0.03$--$0.08$
against an \emph{always-on} (oracle-gated) $0.32$--$0.54$, and the high
preservation ($0.97$--$1.00$) partly reflects this low firing rate rather
than strong discrimination. We therefore treat the gated pipeline as a
negative result and the always-on verified-\PTC{} evaluation as the
mechanistic analysis: the edit has a measurable causal effect on
oracle-identified conflicts, but conflict detection remains unresolved, as
does prospective self-assessment generally \citep{nazi2026triage}. Full
operating points, curves, and score distributions are in Appendix~C.

\begin{table}[t]
\centering
\small
\setlength{\tabcolsep}{2pt}
\renewcommand{\arraystretch}{1.0}
\begin{tabular}{@{}lccccc@{}}
\toprule
\hdrrow
\hdr{Model} & \hdr{AUROC} & \hdr{AUPRC} & \hdr{FPR} &
\hdr{Gated} & \hdr{Always-} \\
\hdrrow
& & & & \hdr{Rec.} & \hdr{on Rec.} \\
\midrule
Qwen-2.5-1.5B   & 0.467 & 0.028 & 0.005 & 0.031 & 0.469 \\
Qwen-2.5-7B     & 0.571 & 0.056 & 0.025 & 0.077 & 0.538 \\
Mistral-7B-v0.3 & 0.657 & 0.092 & 0.103 & 0.056 & 0.324 \\
Llama-3.1-8B    & 0.598 & 0.097 & 0.058 & 0.069 & 0.356 \\
\bottomrule
\end{tabular}
\caption{Corrected detector-gated \TAS{} on the held-out test split.
\emph{AUROC/AUPRC}: corrected detector on test (train-fit,
calibration-calibrated); AUPRC at the true test base rate. \emph{FPR}:
false-positive rate at $\tau{=}0.15$. \emph{Gated}: end-to-end Recovery of
the detect$\to$steer pipeline at the most permissive validation threshold
($\tau{=}0.15$); \emph{Always-on}: Recovery when every verified-\PTC{}
record is steered (the mechanistic upper bound, \cref{tab:tas-vs-iti}).
The gated pipeline recovers almost nothing because the corrected detector
is near chance; full per-threshold operating points with $95\%$ CIs are in
Appendix~C.}
\label{tab:detector-gated}
\end{table}

\subsection{Does the Direction Need Verified Conflicts? (RQ3)}
\label{sec:results-iti}

Does verifying conflicts do any work, or would a population-level
behavioural contrast find the same direction? We test this against an
ITI~\citep{li2023inference} baseline that shares $\ell^{*}$, the
single-direction parameterisation, and the $\alpha$ sweep, so only the
construction of $\Delta$ varies: $\Delta_{\mathrm{ITI}}$ is the prefers-new
minus prefers-old mean of $h_{\ell^{*}}$ under the standard prompt. That
contrast is itself \PTC{}-relevant, making this a strong test of
\emph{verified-conflict} vs.\ population-level construction (the shared
operating-point rule is in Appendix~D).

\begin{table*}[t]
\centering
\small
\setlength{\tabcolsep}{5pt}
\renewcommand{\arraystretch}{1.0}
\begin{tabular}{lc|cccc|ccc}
\toprule
\hdrrow
& & \multicolumn{4}{c}{\hdr{Recovery (verified-\PTC{} test)}}
& \multicolumn{3}{c}{\hdr{PA (clean-control test)}} \\
\hdrrow
\hdr{Model} & \hdr{$n$} &
\hdr{TAS} & \hdr{ITI} & \hdr{T/I disc.} & \hdr{$p_{\mathrm{Holm}}$} &
\hdr{TAS} & \hdr{ITI} & \hdr{$p_{\mathrm{Holm}}$} \\
\midrule
Qwen-2.5-1.5B   & 32 & 0.469 & 0.406 & 5/3   & 1.00 & 0.820 & 0.895 & \textbf{0.006} \\
Qwen-2.5-7B     & 39 & 0.538 & 0.538 & 2/2   & 1.00 & 0.795 & 0.845 & 0.23 \\
Mistral-7B-v0.3 & 71 & 0.324 & 0.380 & 6/10  & 1.00 & 0.850 & 0.830 & 0.56 \\
Llama-3.1-8B    & 87 & 0.356 & 0.287 & 19/13 & 1.00 & 0.785 & 0.835 & 0.24 \\
\bottomrule
\end{tabular}
\caption{\textbf{Held-out paired \TAS{} vs.\ ITI}, subject-disjoint split.
Both directions built on \emph{train} only, $\alpha$ selected independently
on \emph{validation}, each evaluated \emph{once} on the identical test
verified-\PTC{} records ($n$) and $200$ clean controls; only $\Delta$
construction differs. \emph{T/I disc.}: McNemar discordant pairs;
$p_{\mathrm{Holm}}$: exact two-sided McNemar $p$ after Holm correction.
\textbf{No Recovery difference is significant}; ITI preserves better on
Qwen-2.5-1.5B. Oracle full-data comparison in Appendix~E.}
\label{tab:tas-vs-iti}
\end{table*}

Under the held-out protocol (\cref{tab:tas-vs-iti}), the verified-conflict
$\Delta$ shows \emph{no significant Recovery advantage} over the
behavioural-contrast $\Delta$: it leads numerically on two models, ties on
one, trails on one, and every paired McNemar test is non-significant after
Holm correction ($p_{\mathrm{Holm}}=1.00$ throughout). The large Qwen
advantages in the earlier oracle, full-data comparison ($+0.142$, $+0.152$)
do not survive held-out evaluation: they reflected test-set-optimal $\alpha$
and directions built from the whole benchmark (Appendix~E). On preservation,
ITI is the more conservative method ($0.895$ vs.\ $0.820$ on Qwen-2.5-1.5B;
$p_{\mathrm{Holm}}=0.006$). We read this as a model-dependent,
non-significant difference in $\Delta$ construction rather than superiority
of either: verifying conflicts does not reliably beat a population-level
contrast at matched held-out cost.

\subsection{Is the Effect Direction-Specific? (RQ3)}
\label{sec:results-baselines}

Our central causal claim is that a \emph{specific} direction, not generic
perturbation, is causally implicated in shifting the preference. The decisive test
fixes the intervention magnitude and varies only \emph{what} direction we
move along; a norm-matched random direction is the control this claim must
survive. The same experiment positions the intervention against cheap
prompting (\cref{tab:baselines}). The prompt baselines prepend a minimal
date cue (\emph{``In $\langle t\rangle$: \dots''}) or an instruction
(\emph{``Answer with the current, most up-to-date fact.\dots''}) to the
standard question. The controls hold magnitude fixed at $\alpha^{*}$ and
change only what or where we steer: a random unit direction rescaled to
$\lVert\Delta\rVert$ at $\ell^{*}$, and the genuine $\ell^{*}$ direction
applied at an early null layer.

\begin{table*}[t]
\centering
\small
\setlength{\tabcolsep}{5pt}
\renewcommand{\arraystretch}{1.0}
\begin{tabular}{lccccc|cc}
\toprule
\hdrrow
& & \multicolumn{4}{c}{\hdr{Recovery}} & \multicolumn{2}{c}{\hdr{Steering controls}} \\
\hdrrow
\hdr{Model} & \hdr{$n_{\PTC{}}$} &
\hdr{Date-prefix} & \hdr{Instruction} & \hdr{\TAS{} (oracle)} &
\hdr{Standard} & \hdr{Random-dir} & \hdr{Off-peak} \\
\midrule
Qwen-2.5-1.5B   & 101 & 0.614\,{\scriptsize[.51,.70]} & 0.416 & 0.515 & 0.000 & 0.257\,{\scriptsize[.18,.35]} & 0.446 \\
Qwen-2.5-7B     & 248 & 0.782\,{\scriptsize[.73,.83]} & 0.278 & 0.617 & 0.000 & 0.177\,{\scriptsize[.13,.23]} & 0.448 \\
Mistral-7B-v0.3 & 481 & 0.794\,{\scriptsize[.76,.83]} & 0.154 & 0.293 & 0.000 & 0.241\,{\scriptsize[.20,.28]} & 0.412 \\
Llama-3.1-8B    & 596 & 0.805\,{\scriptsize[.77,.84]} & 0.193 & 0.382 & 0.000 & 0.312\,{\scriptsize[.28,.35]} & 0.456 \\
\bottomrule
\end{tabular}
\caption{Prompt baselines and steering controls on the identical
oracle-identified \PTC{} subset ($n_{\PTC{}}$ per model), same
length-normalised log-prob protocol. \emph{Standard} Recovery is $0$ by
construction. Brackets are $95\%$ percentile-bootstrap CIs
($B{=}10{,}000$). A date-prefix prompt recovers \emph{more} than \TAS{} on
every model; the norm-matched random-direction control collapses well
below it, so recovery depends on the conflict direction, not generic
perturbation. \emph{Off-peak} applies the same $\Delta$ at an early layer.
This is a full-data oracle diagnostic; the held-out portability analysis is
\cref{tab:layer-localization}.}
\label{tab:baselines}
\end{table*}

Two conclusions follow, and both sharpen rather than inflate the claim.
First, \emph{a date-prefix prompt out-recovers \TAS{} on all four models}
($0.61$--$0.81$ vs.\ $0.29$--$0.62$) at comparable or better preservation
(PA $0.93$--$0.98$ vs.\ $0.85$--$0.94$; both oracle diagnostics on the same
full-data subset). \TAS{} is therefore \emph{not} a replacement for
prompting; its role is to show causally that the preference can be shifted
by a targeted edit, which prompting neither localises nor explains. Second,
with magnitude held fixed, a norm-matched random vector drops Recovery to
$0.18$--$0.31$, so the effect depends on the \emph{direction}, while the
true $\Delta$ stays effective at an early off-peak layer. The intervention
is direction-specific, not tied to a unique layer.

\subsection{Do Models Conflict on the Same Facts? (RQ4)}
\label{sec:results-overlap}

Comparing which records each model flags is exploratory evidence about
whether \PTC{}-instance agreement is associated more with architectural
lineage or with apparent training-data recency. On
the $8{,}734$ records common to all four runs, the \PTC{}-positive
sets are largely disjoint (\cref{fig:overlap}): only $6$ records are
flagged by all four models (union $1{,}287$), and the
Mistral~$\cap$~Llama overlap ($193$) exceeds the two-Qwen overlap ($38$)
by more than $5\times$ despite the two Qwens sharing architecture and
training data. Among the pairs we can compare, \PTC{}-instance agreement is
thus more closely associated with shared pre-training recency than with
architectural lineage; this is an observational comparison over four models,
so suggestive rather than conclusive. Cross-family transfer is in Appendix~N.

\begin{figure}[t]
    \centering
    \includegraphics[width=\columnwidth]{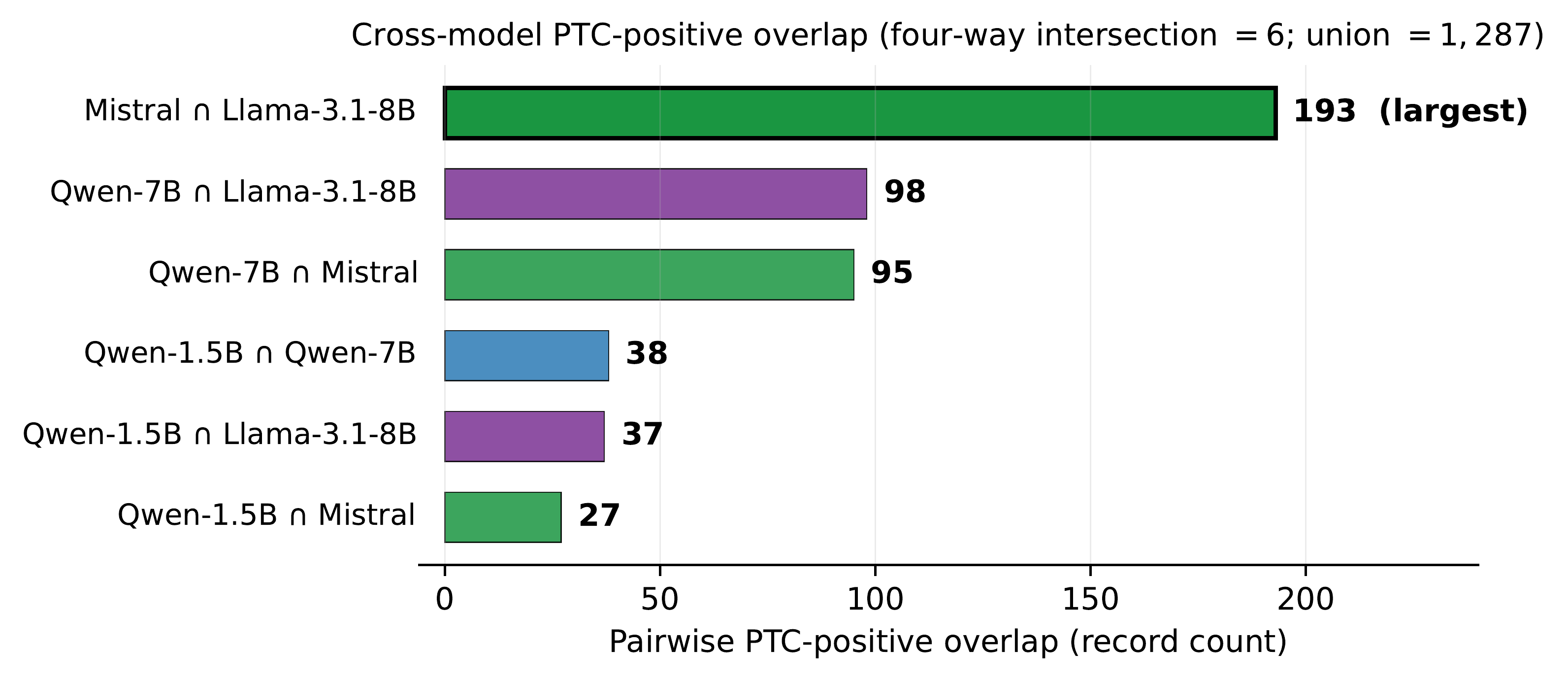}
    \caption{Pairwise and four-way \PTC{}-positive overlap on the
    $8{,}734$ records common to all four runs. The Mistral $\cap$
    Llama-3.1-8B intersection ($193$) is the largest pairwise overlap and
    substantially exceeds the two-Qwen overlap ($38$), even though the two
    Qwens share architecture and training data.}
    \label{fig:overlap}
\end{figure}

\paragraph{Free-generation robustness.}
Under unrestricted greedy decoding, base-model output is frequently
dominated by definitional or non-entity continuations, and no intervention
consistently produced the newer entity. We therefore treat candidate-level
scoring as the primary mechanistic evaluation, and report free generation in
Appendix~M.

\section{Related Work}
\label{sec:related}

\paragraph{Temporal QA and knowledge updating.}
LMs degrade on time-anchored recall
\citep{chen2021dataset,zhang2021situatedqa,kasai2023realtime} and temporal
reasoning \citep{chu2024timebench,tan2024towards,zhu2023question,
zhu2025evolvebench,zhang2024mrag,uddin2025unseentimeqa,fatemi2025test}.
Retrieval and year-conditioning \citep{vu2024freshllms,zhao2024set,
herel2024time} and continual learning \citep{dhingra2022time,zhao2022can,
kim2024carpe} only partially close the gap, and unlearning
\citep{sarwar-etal-2026-multimodal} removes rather than reprioritises.
None isolates the
\emph{parametric coexistence} targeted by \PTC{}.

\paragraph{Knowledge conflicts and steering.}
Temporal misalignment is the primary driver of intra-memory conflict
\citep{xu2024knowledge}, and hidden states are known to encode
temporal or mutability structure
\citep{fierro2024mulan,dey2026temporal,marjanovic2024dynamicqa,
bajpai2024temporally}, though none of this work targets outdated
\emph{preference}. \citet{nylund2024time} structure temporal signals in
\emph{fine-tuned} weight space; we ask the analogous question of the
\emph{pretrained} residual stream. Closest to us,
\citet{pham2026knowledge} concurrently locate intra-memory conflicts by
patching synthetic biographies, but provide no recovery filter and no
verified real-world evaluation. Existing steering methods either intervene
unconditionally, without separating absence from conflict
\citep{govindan2025temporal}, or target different conflicts entirely
\citep{li2025taming,kang2026model,li2026costom}, or reweight attention by
uncertainty rather than conflict \citep{hossain2026uat}.
ITI~\citep{li2023inference} is our matched conflict-agnostic baseline, and
ROME/MEMIT~\citep{meng2022locating,meng2022mass} edit weights via causal
tracing \citep{geva2021transformer,hase2023does} rather than at query time.
Against these, our contribution is
the phenomenon itself: we characterize \PTC{} on \emph{real superseding
facts} rather than synthetic patches, and localise it with an edit built
from \emph{verified per-instance conflicts}.

\section{Conclusion}
\label{sec:conclusion}

Parametric Temporal Conflict exposes a previously uncharacterized failure
mode in parametric knowledge retrieval: among outdated answers that pass our
recoverability criterion, the newer fact is present and merely dispreferred
rather than absent. We separate what our experiments settle from what they
merely favour.

\paragraph{Established.}
Under our benchmark definition, in the four evaluated models, \PTC{} exists
and is measurable: where the newer fact is verifiably elicitable, the model
still prefers the outdated one, and a date-prefix prompt recovers it in
$61$--$81\%$ of those cases. Two causal interventions localize it.
Single-layer patching flips $72$--$85\%$ of cases within a model-specific
upper-layer region. With magnitude held fixed, a single-direction edit moves
$29$--$62\%$ of oracle-identified conflicts against $18$--$31\%$ for a
norm-matched random direction, implicating the direction rather than generic
perturbation.

\paragraph{Supported but not proven.}
The filtered-\PTC{} trend and its association with pre-training recency are
consistent with capacity and data recency mattering, but four models
spanning three families cannot separate those factors. That \PTC{} extends
beyond single-holder position-holder relations is plausible but untested.

\paragraph{Open.}
Practical conflict detection is unresolved (our detector is near chance
under subject-disjoint evaluation, which is why every recovery figure above
is oracle-conditioned), as is whether such an intervention can be made
both selective and competitive.

The line between engineering utility and scientific evidence is therefore
deliberate. Our intervention loses to a date-prefix prompt, ties a matched
ITI baseline, and cannot be gated reliably: real negative results about
\emph{deployment} that leave the mechanistic findings untouched, because
the causal claims rest on internally controlled contrasts, not on the edit
being the best available fix. A probe need not outperform a prompt to reveal
where a preference lives.

\section*{Limitations and Ethical Considerations}

Because the benchmark is conditioned on verified superseding transitions
rather than sampled user queries, our rates are not estimates of real-world
\PTC{} prevalence. Beyond the scope and detection limits above, our analysis
needs residual-stream access (open-weight, not API-only models), and held-out
\PTC{} counts are small for the smaller models ($n{=}32$ for Qwen-2.5-1.5B),
leaving per-model comparisons imprecise. Ethically, activation edits alter
representations without provenance, so the mechanism that corrects an
outdated fact could install a false belief leaving no trace; we treat the
edit strictly as an analysis instrument (auditability in \cref{app:complexity-analysis}).

\section*{Author Contributions}

Elias Hossain, Sourav Saha, and Tasfia Nuzhat Ornee contributed equally.
They jointly developed the method and the experimental pipeline, debugged
and improved the implementation, ran the evaluations, and wrote the
technical content of the paper. Sanjeda Sara Jennifer carried out the
literature review. Shubhashis Roy Dipta reviewed the manuscript,
identified issues in the analysis and presentation, and contributed
revisions that improved the work. Umesh Chandra Biswas reviewed the
manuscript. Niloofar Yousefi and Rajib Rana supervised the project and
provided feedback throughout.

\bibliography{ref,roydipta}

\begin{thebibliography}{37}
\providecommand{\natexlab}[1]{#1}

\bibitem[{Bajpai et~al.(2024)Bajpai, Goyal, Anwer, and Chakraborty}]{bajpai2024temporally}
Ashutosh Bajpai, Aaryan Goyal, Atif Anwer, and Tanmoy Chakraborty. 2024.
\newblock Temporally consistent factuality probing for large language models.
\newblock In \emph{Proceedings of the 2024 Conference on Empirical Methods in Natural Language Processing}, pages 15864--15881.

\bibitem[{Chen et~al.(2021)Chen, Wang, and Wang}]{chen2021dataset}
Wenhu Chen, Xinyi Wang, and William~Yang Wang. 2021.
\newblock \href {https://arxiv.org/abs/2108.06314} {A dataset for answering time-sensitive questions}.
\newblock \emph{CoRR}, abs/2108.06314.

\bibitem[{Chu et~al.(2024)Chu, Chen, Chen, Yu, Wang, Liu, and Qin}]{chu2024timebench}
Zheng Chu, Jingchang Chen, Qianglong Chen, Weijiang Yu, Haotian Wang, Ming Liu, and Bing Qin. 2024.
\newblock Timebench: A comprehensive evaluation of temporal reasoning abilities in large language models.
\newblock In \emph{Proceedings of the 62nd Annual Meeting of the Association for Computational Linguistics (Volume 1: Long Papers)}, pages 1204--1228.

\bibitem[{Dey et~al.(2026)Dey, Ounis, McDonald, and Moshfeghi}]{dey2026temporal}
Ritajit Dey, Iadh Ounis, Graham McDonald, and Yashar Moshfeghi. 2026.
\newblock Temporal fact conflicts in llms: reproducibility insights from unifying dynamicqa and mulan.
\newblock \emph{arXiv preprint arXiv:2603.15892}.

\bibitem[{Dhingra et~al.(2022)Dhingra, Cole, Eisenschlos, Gillick, Eisenstein, and Cohen}]{dhingra2022time}
Bhuwan Dhingra, Jeremy~R Cole, Julian~Martin Eisenschlos, Daniel Gillick, Jacob Eisenstein, and William~W Cohen. 2022.
\newblock Time-aware language models as temporal knowledge bases.
\newblock \emph{Transactions of the Association for Computational Linguistics}, 10:257--273.

\bibitem[{Fatemi et~al.(2025)Fatemi, Kazemi, Tsitsulin, Malkan, Yim, Palowitch, Seo, Halcrow, and Perozzi}]{fatemi2025test}
Bahare Fatemi, Seyed~Mehran Kazemi, Anton Tsitsulin, Karishma Malkan, Jinyeong Yim, John Palowitch, Sungyong Seo, Jonathan Halcrow, and Bryan Perozzi. 2025.
\newblock Test of time: A benchmark for evaluating llms on temporal reasoning.
\newblock In \emph{International Conference on Learning Representations}, volume 2025, pages 94426--94447.

\bibitem[{Fierro et~al.(2024)Fierro, Garneau, Bugliarello, Kementchedjhieva, and S{\o}gaard}]{fierro2024mulan}
Constanza Fierro, Nicolas Garneau, Emanuele Bugliarello, Yova Kementchedjhieva, and Anders S{\o}gaard. 2024.
\newblock Mulan: A study of fact mutability in language models.
\newblock In \emph{Proceedings of the 2024 Conference of the North American Chapter of the Association for Computational Linguistics: Human Language Technologies (Volume 2: Short Papers)}, pages 762--771.

\bibitem[{Geva et~al.(2021)Geva, Schuster, Berant, and Levy}]{geva2021transformer}
Mor Geva, Roei Schuster, Jonathan Berant, and Omer Levy. 2021.
\newblock Transformer feed-forward layers are key-value memories.
\newblock In \emph{Proceedings of the 2021 Conference on Empirical Methods in Natural Language Processing}, pages 5484--5495.

\bibitem[{Govindan et~al.(2025)Govindan, Pagnucco, and Song}]{govindan2025temporal}
Sanjay Govindan, Maurice Pagnucco, and Yang Song. 2025.
\newblock Temporal alignment of time sensitive facts with activation engineering.
\newblock \emph{arXiv preprint arXiv:2505.14158}.

\bibitem[{Hase et~al.(2023)Hase, Bansal, Kim, and Ghandeharioun}]{hase2023does}
Peter Hase, Mohit Bansal, Been Kim, and Asma Ghandeharioun. 2023.
\newblock Does localization inform editing? surprising differences in causality-based localization vs. knowledge editing in language models.
\newblock \emph{Advances in Neural Information Processing Systems}, 36:17643--17668.

\bibitem[{Herel et~al.(2024)Herel, Bartek, Jirak, and Mikolov}]{herel2024time}
David Herel, Vojtech Bartek, Jiri Jirak, and Tomas Mikolov. 2024.
\newblock Time awareness in large language models: benchmarking fact recall across time.
\newblock \emph{arXiv preprint arXiv:2409.13338}.

\bibitem[{Hossain et~al.(2026)Hossain, Roy~Dipta, Neupane, Rana, Shwartz-Ziv, Garibay, and Yousefi}]{hossain2026uat}
Elias Hossain, Shubhashis Roy~Dipta, Subash Neupane, Rajib Rana, Ravid Shwartz-Ziv, Ivan Garibay, and Niloofar Yousefi. 2026.
\newblock {UAT-LITE}: Inference-time uncertainty-aware attention for pretrained transformers.
\newblock \emph{arXiv preprint arXiv:2602.02952}.

\bibitem[{Jiang et~al.(2026)Jiang, Li, Roy~Dipta, Li, and Yang}]{jiang2026rock}
Yuxuan Jiang, Runchao Li, Shubhashis Roy~Dipta, Dawei Li, and Zhao Yang. 2026.
\newblock Cornerstones or stumbling blocks? deciphering the rock tokens in on-policy distillation.
\newblock \emph{arXiv preprint arXiv:2605.09253}.

\bibitem[{Kang et~al.(2026)Kang, Shi, and Chen}]{kang2026model}
Xinyue Kang, Diwei Shi, and Li~Chen. 2026.
\newblock Model whisper: Steering vectors unlock large language models’ potential in test-time.
\newblock In \emph{Proceedings of the AAAI Conference on Artificial Intelligence}, volume~40, pages 31392--31400.

\bibitem[{Kasai et~al.(2024)Kasai, Sakaguchi, Takahashi, Bras, Asai, Yu, Radev, Smith, Choi, and Inui}]{kasai2023realtime}
Jungo Kasai, Keisuke Sakaguchi, Yoichi Takahashi, Ronan~Le Bras, Akari Asai, Xinyan Yu, Dragomir Radev, Noah~A. Smith, Yejin Choi, and Kentaro Inui. 2024.
\newblock \href {https://arxiv.org/abs/2207.13332} {Realtime qa: What's the answer right now?}
\newblock \emph{Preprint}, arXiv:2207.13332.

\bibitem[{Kim et~al.(2024)Kim, Yoon, Ye, Bae, Ho, Hwang, and Yun}]{kim2024carpe}
Yujin Kim, Jaehong Yoon, Seonghyeon Ye, Sangmin Bae, Namgyu Ho, Sung~Ju Hwang, and Se-Young Yun. 2024.
\newblock Carpe diem: On the evaluation of world knowledge in lifelong language models.
\newblock In \emph{Proceedings of the 2024 Conference of the North American Chapter of the Association for Computational Linguistics: Human Language Technologies (Volume 1: Long Papers)}, pages 5401--5415.

\bibitem[{Li et~al.(2025)Li, Chen, and Tong}]{li2025taming}
Gaotang Li, Yuzhong Chen, and Hanghang Tong. 2025.
\newblock Taming knowledge conflicts in language models.
\newblock \emph{arXiv preprint arXiv:2503.10996}.

\bibitem[{Li et~al.(2023)Li, Patel, Vi{\'e}gas, Pfister, and Wattenberg}]{li2023inference}
Kenneth Li, Oam Patel, Fernanda Vi{\'e}gas, Hanspeter Pfister, and Martin Wattenberg. 2023.
\newblock Inference-time intervention: Eliciting truthful answers from a language model.
\newblock \emph{Advances in Neural Information Processing Systems}, 36:41451--41530.

\bibitem[{Li et~al.(2026)Li, Shi, and Deng}]{li2026costom}
Mengfan Li, Xuanhua Shi, and Yang Deng. 2026.
\newblock Costom: Causal-oriented steering for intrinsic theory-of-mind alignment in large language models.
\newblock In \emph{Proceedings of the 64th Annual Meeting of the Association for Computational Linguistics (Volume 1: Long Papers)}, pages 9302--9317.

\bibitem[{Marjanovi{\'c} et~al.(2024)Marjanovi{\'c}, Yu, Atanasova, Maistro, Lioma, and Augenstein}]{marjanovic2024dynamicqa}
Sara~Vera Marjanovi{\'c}, Haeun Yu, Pepa Atanasova, Maria Maistro, Christina Lioma, and Isabelle Augenstein. 2024.
\newblock Dynamicqa: Tracing internal knowledge conflicts in language models.
\newblock In \emph{Findings of the Association for Computational Linguistics: EMNLP 2024}, pages 14346--14360.

\bibitem[{Meng et~al.(2022{\natexlab{a}})Meng, Bau, Andonian, and Belinkov}]{meng2022locating}
Kevin Meng, David Bau, Alex Andonian, and Yonatan Belinkov. 2022{\natexlab{a}}.
\newblock Locating and editing factual associations in gpt.
\newblock \emph{Advances in neural information processing systems}, 35:17359--17372.

\bibitem[{Meng et~al.(2022{\natexlab{b}})Meng, Sharma, Andonian, Belinkov, and Bau}]{meng2022mass}
Kevin Meng, Arnab~Sen Sharma, Alex Andonian, Yonatan Belinkov, and David Bau. 2022{\natexlab{b}}.
\newblock Mass-editing memory in a transformer.
\newblock \emph{arXiv preprint arXiv:2210.07229}.

\bibitem[{Nazi and Roy~Dipta(2026)}]{nazi2026triage}
Zabir~Al Nazi and Shubhashis Roy~Dipta. 2026.
\newblock {TRIAGE}: Evaluating prospective metacognitive control in {LLMs} under resource constraints.
\newblock \emph{arXiv preprint arXiv:2605.13414}.

\bibitem[{Nazi et~al.(2026)Nazi, Roy~Dipta, and Parvez}]{nazi2026omni}
Zabir~Al Nazi, Shubhashis Roy~Dipta, and Md~Rizwan Parvez. 2026.
\newblock Omni-modal dissonance benchmark: Systematically breaking modality consensus to probe robustness and calibrated abstention.
\newblock \emph{arXiv preprint arXiv:2603.27187}.

\bibitem[{Nylund et~al.(2024)Nylund, Gururangan, and Smith}]{nylund2024time}
Kai Nylund, Suchin Gururangan, and Noah~A. Smith. 2024.
\newblock Time is encoded in the weights of finetuned language models.
\newblock In \emph{Proceedings of the 62nd Annual Meeting of the Association for Computational Linguistics (Volume 1: Long Papers)}, pages 2571--2587.

\bibitem[{Pham et~al.(2026)Pham, Borkakoty, and Hou}]{pham2026knowledge}
Minh~Vu Pham, Hsuvas Borkakoty, and Yufang Hou. 2026.
\newblock Where knowledge collides: A mechanistic study of intra-memory knowledge conflict in language models.
\newblock \emph{arXiv preprint arXiv:2601.09445}.

\bibitem[{Sarwar et~al.(2026)Sarwar, Roy~Dipta, Liu, and Patil}]{sarwar-etal-2026-multimodal}
Nobin Sarwar, Shubhashis Roy~Dipta, Zheyuan Liu, and Vaidehi Patil. 2026.
\newblock \href {https://doi.org/10.18653/v1/2026.findings-acl.1379} {Multimodal unlearning across vision, language, video, and audio: Survey of methods, datasets, and benchmarks}.
\newblock In \emph{Findings of the {A}ssociation for {C}omputational {L}inguistics: {ACL} 2026}, pages 27702--27730, San Diego, California, United States. Association for Computational Linguistics.

\bibitem[{Tan et~al.(2024)Tan, Ng, and Bing}]{tan2024towards}
Qingyu Tan, Hwee~Tou Ng, and Lidong Bing. 2024.
\newblock Towards robust temporal reasoning of large language models via a multi-hop qa dataset and pseudo-instruction tuning.
\newblock In \emph{Findings of the Association for Computational Linguistics: ACL 2024}, pages 6272--6286.

\bibitem[{Uddin et~al.(2025)Uddin, Saeidi, Handa, Seth, Son, Blanco, Corman, and Baral}]{uddin2025unseentimeqa}
Md~Nayem Uddin, Amir Saeidi, Divij Handa, Agastya Seth, Tran~Cao Son, Eduardo Blanco, Steven Corman, and Chitta Baral. 2025.
\newblock Unseentimeqa: Time-sensitive question-answering beyond llms’ memorization.
\newblock In \emph{Proceedings of the 63rd Annual Meeting of the Association for Computational Linguistics (Volume 1: Long Papers)}, pages 1873--1913.

\bibitem[{Vu et~al.(2024)Vu, Iyyer, Wang, Constant, Wei, Wei, Tar, Sung, Zhou, Le, and Luong}]{vu2024freshllms}
Tu~Vu, Mohit Iyyer, Xuezhi Wang, Noah Constant, Jerry Wei, Jason Wei, Chris Tar, Yun-Hsuan Sung, Denny Zhou, Quoc Le, and Thang Luong. 2024.
\newblock \href {https://doi.org/10.18653/v1/2024.findings-acl.813} {{FreshLLMs}: Refreshing large language models with search engine augmentation}.
\newblock In \emph{Findings of the Association for Computational Linguistics: ACL 2024}, pages 13697--13720, Bangkok, Thailand. Association for Computational Linguistics.

\bibitem[{Xu et~al.(2024)Xu, Qi, Guo, Wang, Wang, Zhang, and Xu}]{xu2024knowledge}
Rongwu Xu, Zehan Qi, Zhijiang Guo, Cunxiang Wang, Hongru Wang, Yue Zhang, and Wei Xu. 2024.
\newblock Knowledge conflicts for llms: A survey.
\newblock In \emph{Proceedings of the 2024 Conference on Empirical Methods in Natural Language Processing}, pages 8541--8565.

\bibitem[{Zhang and Choi(2021)}]{zhang2021situatedqa}
Michael J.~Q. Zhang and Eunsol Choi. 2021.
\newblock \href {https://arxiv.org/abs/2109.06157} {Situatedqa: Incorporating extra-linguistic contexts into {QA}}.
\newblock \emph{CoRR}, abs/2109.06157.

\bibitem[{Zhang et~al.(2024)Zhang, Xue, Zhang, Wu, Luu, and Zhao}]{zhang2024mrag}
Siyue Zhang, Yuxiang Xue, Yiming Zhang, Xiaobao Wu, Anh~Tuan Luu, and Chen Zhao. 2024.
\newblock Mrag: A modular retrieval framework for time-sensitive question answering.
\newblock \emph{arXiv preprint arXiv:2412.15540}.

\bibitem[{Zhao et~al.(2024)Zhao, Brumbaugh, Wang, Hajishirzi, and Smith}]{zhao2024set}
Bowen Zhao, Zander Brumbaugh, Yizhong Wang, Hannaneh Hajishirzi, and Noah~A Smith. 2024.
\newblock Set the clock: Temporal alignment of pretrained language models.
\newblock In \emph{Findings of the Association for Computational Linguistics: ACL 2024}, pages 15015--15040.

\bibitem[{Zhao et~al.(2022)Zhao, Zhao, Xu, Zhang, and Jin}]{zhao2022can}
Ruilin Zhao, Feng Zhao, Guandong Xu, Sixiao Zhang, and Hai Jin. 2022.
\newblock Can language models serve as temporal knowledge bases?
\newblock In \emph{Findings of the Association for Computational Linguistics: EMNLP 2022}, pages 2024--2037.

\bibitem[{Zhu et~al.(2023)Zhu, Yang, Chen, Li, Lou, and Yang}]{zhu2023question}
Xinyu Zhu, Cheng Yang, Bei Chen, Siheng Li, Jian-Guang Lou, and Yujiu Yang. 2023.
\newblock Question answering as programming for solving time-sensitive questions.
\newblock In \emph{Proceedings of the 2023 Conference on Empirical Methods in Natural Language Processing}, pages 12775--12790.

\bibitem[{Zhu et~al.(2025)Zhu, Liao, Chen, Wang, Guan, Wang, and Wang}]{zhu2025evolvebench}
Zhiyuan Zhu, Yusheng Liao, Zhe Chen, Yuhao Wang, Yunfeng Guan, Yanfeng Wang, and Yu~Wang. 2025.
\newblock Evolvebench: A comprehensive benchmark for assessing temporal awareness in llms on evolving knowledge.
\newblock In \emph{Proceedings of the 63rd Annual Meeting of the Association for Computational Linguistics (Volume 1: Long Papers)}, pages 16173--16188.

\end{thebibliography}

\appendix
\crefalias{section}{appendix}
\crefalias{subsection}{appendix}
\onecolumn
\section*{Appendix}

This supplementary material provides the mathematical foundations,
algorithmic details, extended empirical results, ablations, and
reproducibility information that support the main paper. To help readers
navigate, the appendices are grouped thematically rather than listed flat.

\paragraph{Method foundations.}
\begin{itemize}\setlength\itemsep{1pt}
    \item \textbf{\cref{app:mathematical-details}}: standing
    assumptions, theorem statements, and local-margin proof outlines
    for \ours.
    \item \textbf{\cref{app:algorithm}}: the inference-time
    algorithm with per-stage rationale.
    \item \textbf{\cref{app:baseline-metrics}}: baselines
    (standard / temporal / RAG / ITI / contrastive decoding), metric
    definitions, and the ITI construction.
\end{itemize}

\paragraph{Benchmark, screening, and detection protocol.}
\begin{itemize}\setlength\itemsep{1pt}
    \item \textbf{\cref{app:splits}}: the subject-disjoint
    train / validation / test split and corrected detector protocol:
    calibration, test-set discrimination, operating points
    (\cref{tab:detector-ops}), and false-alarm analysis.
    \item \textbf{\cref{app:filter}}: knowledge-recovery filter
    mechanics and per-model dropout (visualised in
    \cref{fig:filter-capacity}), with the $\tau_{\mathrm{rec}}$
    sensitivity sweep (\cref{tab:tau-rec-sweep}).
    \item \textbf{\cref{app:benchmark}}: benchmark construction,
    SPARQL extraction, prompt templates, admission conditions, and the
    manual-audit protocol.
    \item \textbf{\cref{app:phase1-extended}}: extended Phase~1
    screening: prob-space \EG{}, per-relation breakdown
    (\cref{fig:relation-breakdown}), per-domain table.
\end{itemize}

\paragraph{Localization, variants, and end-to-end diagnostics.}
\begin{itemize}\setlength\itemsep{1pt}
    \item \textbf{\cref{app:paired}}: paired held-out
    \TAS{}-vs-ITI evaluation under the corrected protocol
    (\cref{tab:paired-construction}), with the oracle full-data
    diagnostic (\cref{tab:oracle-iti-diag}).
    \item \textbf{\cref{app:localization}}: layer localization by
    activation patching vs.\ direct editing
    (\cref{fig:layer-localization}).
    \item \textbf{\cref{app:locator-details}}: per-layer AFR
    profiles (\cref{fig:afr-profile}) and plateau geometry.
    \item \textbf{\cref{app:variants}}: V1\,/\,V2\,/\,V3
    steering-vector ablation.
    \item \textbf{\cref{app:tas-extended}}: detector PR curves
    (\cref{fig:detector-pr}), oracle $\alpha$-sweep
    (\cref{fig:oracle-alpha}), $\tau$-Pareto frontier
    (\cref{tab:tau-rec-sweep}), false-alarm diagnostics.
    \item \textbf{\cref{app:freegen}}: free-generation evaluation:
    protocol, alias resolution, main results
    (\cref{tab:fg-categories}), and generation-vs-scoring
    disagreement (\cref{tab:fg-disagreement}).
    \item \textbf{\cref{app:appendix_cross_model}}: cross-family
    generalisation at matched scale.
\end{itemize}

\paragraph{Reproducibility and ethics.}
\begin{itemize}\setlength\itemsep{1pt}
    \item \textbf{\cref{app:complexity-analysis}}: hardware,
    compute, storage, and reproducibility provisions.
    \item \textbf{\cref{app:ai-usage}}: AI usage statement.
\end{itemize}

\begin{insightblue}{Appendix at a glance}
Five new figures distil the appendix tables into single-glance views:
per-layer AFR (\cref{fig:afr-profile}), detector PR curves
(\cref{fig:detector-pr}), $\tau$-Pareto frontier
(\cref{tab:tau-rec-sweep}), oracle $\alpha$-sweep
(\cref{fig:oracle-alpha}), per-relation \PTC{} rate
(\cref{fig:relation-breakdown}), and the knowledge-recovery filter
trajectory (\cref{fig:filter-capacity}). Every figure is rendered
directly from the released per-model JSON outputs; no aggregated number
in this appendix is independent of the released data.
\end{insightblue}

\newpage

\section{Mathematical Foundations of \TAS{}}
\label{app:mathematical-details}

\noindent\emph{
These results establish local sufficient conditions for \TAS{}.
They are not intended as global guarantees for arbitrary transformer
dynamics.
}

\begin{insightblue}{Geometric intuition behind \TAS{}}
\TAS{} treats temporally outdated and temporally recovered behaviours
as two locally separable residual-stream trajectories at a single
intermediate layer $\ell^{*}$. The temporal-cue forward pass exposes
the direction along which the newer-fact trajectory diverges from the
outdated-preference trajectory; \TAS{} adds a scaled multiple of that
direction back into the standard-prompt forward pass so the residual
state crosses the local decision boundary between the two basins,
without modifying any model parameter.
\end{insightblue}


\subsection{Standing Assumptions}

\begin{assumption}[Parametric Co-encoding]
\label{ass:coencoding}

For a verified \PTC{} quadruple
$\mathcal{Q}=(q,a_{\mathrm{old}},a_{\mathrm{new}},t_{\mathrm{update}})$,
both $a_{\mathrm{old}}$ and $a_{\mathrm{new}}$ are encoded in
parametric memory.

Formally, there exists a prompting variant $c$ such that
\[
\overline{\log P}_{\theta}(a_{\mathrm{new}}\mid q,c)
\;\ge\;
\tau_{\mathrm{rec}},
\]
where scores are measured using length-normalized log-probability
(i.e., the per-token mean of $\log P_{\theta}$ over the tokens of
$a_{\mathrm{new}}$).

In our benchmark,
$c=c_{\mathrm{temp}}$
and
$\tau_{\mathrm{rec}}=-3$.
\end{assumption}

\begin{assumption}[Local Separability]
\label{ass:separability}

At a chosen intervention layer $\ell^*$ (a peak within the family-specific
causal plateau; the result requires only that Local Separability hold at
the selected layer, not that $\ell^*$ be unique), outdated-preference and
newer-fact trajectories are locally separable within the conflict region.

Formally, there exists a unit vector
$u\in\mathbb{R}^d$
and margin
$\delta>0$
such that
\[
\Bigl\langle
h_{\ell^*}(q,c_{\mathrm{temp}})
-
h_{\ell^*}(q),
u
\Bigr\rangle
\ge \delta
\]
for every verified \PTC{} instance $q$.

Geometrically, this assumption states that temporal elicitation moves
the residual trajectory along a locally identifiable direction in
representation space.
\end{assumption}

\begin{assumption}[Bounded Steering]
\label{ass:bounded}

For an appropriate steering scale $\alpha$,
the perturbation
$\alpha c \Delta$
modifies the conflict-relevant residual-stream trajectory without
causing catastrophic distortion of downstream computation.
\end{assumption}


\subsection{Theoretical Results}

We first state a structural proposition characterizing sufficient
conditions for \PTC{}, followed by the two principal \TAS{}
results.

\begin{proposition}[Sufficient Condition for \PTC{}]
\label{prop:ptc-sufficient}

Assume Parametric Co-encoding.
If there exists a layer $\ell^*$ such that
\[
h_{\ell^*}(q)
\]
lies in a region favoring
$a_{\mathrm{old}}$
while
\[
h_{\ell^*}(q,c_{\mathrm{temp}})
\]
lies in a region favoring
$a_{\mathrm{new}}$,
then $\theta$ exhibits \PTC{} on $\mathcal{Q}$ as defined in
the PTC definition in the main paper.
\end{proposition}

\begin{proof}[Proof sketch]
The standard-prompt trajectory yields outdated preference,
satisfying the first \PTC{} condition.
The temporally elicited trajectory yields current-fact recovery,
satisfying the second condition.
Together these establish the defining two-part structure of \PTC{}.
\end{proof}

\setcounter{theorem}{0}

\noindent The main paper states a shortened, self-contained form of the
following theorem (TAS Correctness, local sufficient form); the statement
below is the complete version and is mathematically identical.

\begin{theorem}[\TAS{} Correctness]
\label{thm:tas-full}

Under Assumptions
\ref{ass:coencoding}--\ref{ass:bounded},
there exists a finite steering scale
$\alpha^*>0$
such that, for any verified \PTC{} instance in the local separability
regime,
applying the \TAS{} update from
the main-paper TAS update
at layer $\ell^*$
yields
\[
\Ptheta(a_{\mathrm{new}}\mid q,\TAS)
>
\Ptheta(a_{\mathrm{old}}\mid q,\TAS).
\]
\end{theorem}

\begin{proof}[Proof outline with local margin derivation]
Let $s_{\mathrm{new}}(h)$ and $s_{\mathrm{old}}(h)$ denote the
length-normalized log-probability scores (as defined in
the main-paper scoring equation) assigned to
$a_{\mathrm{new}}$ and $a_{\mathrm{old}}$ when the layer-$\ell^*$
residual stream is $h$, and define the local \emph{margin}
\[
m(h) \;=\; s_{\mathrm{new}}(h) - s_{\mathrm{old}}(h).
\]
For a verified \PTC{} instance with hidden state $h := h_{\ell^*}(q)$
and temporally elicited hidden state
$h_{\mathrm{tmp}} := h_{\ell^*}(q,c_{\mathrm{temp}})$, the \PTC{}
definition together with Parametric Co-encoding gives
$m(h) < 0$ and $m(h_{\mathrm{tmp}}) > 0$.

Apply the \TAS{} update
$h' = h + \alpha c \Delta$
(the main-paper TAS update). Under a local linearization of $m$ around $h$,
\[
m(h') \;\approx\; m(h) \;+\; \alpha c\,\nabla m(h)^{\!\top} \Delta.
\]
Local Separability (\cref{ass:separability}) ensures that $\Delta$
shares the separating direction $u$ with $h_{\mathrm{tmp}} - h$ in
expectation, so $\nabla m(h)^{\!\top} \Delta > 0$ in the local
neighborhood. For detector-positive queries
($c > \tau$ and therefore $c > 0$), choosing any
\[
\alpha \;>\; \frac{-\,m(h)}{c \,\nabla m(h)^{\!\top} \Delta} \;>\; 0
\]
yields $m(h') > 0$, hence
$\Ptheta(a_{\mathrm{new}}\mid q,\TAS) > \Ptheta(a_{\mathrm{old}}\mid q,\TAS)$.
Bounded Steering (\cref{ass:bounded}) guarantees that this transition
occurs within a regime where downstream computation is not
catastrophically distorted, so the sign of the local margin is
preserved through the remaining $L - \ell^*$ layers.

This is a \emph{local} sufficient condition that holds in a
neighborhood of $h$ where the linearization is valid; it is not a
global guarantee over arbitrary transformer dynamics. The existence of
$\alpha^*$ is established empirically by the oracle $\alpha$ sweep
(the main-paper end-to-end results).
\end{proof}

Empirically, this local transition corresponds to the observed
reversal from outdated-preference behavior to newer-fact recovery in
verified \PTC{} instances.

\begin{theorem}[No-Op Preservation]
\label{thm:preserve-full}

For any query satisfying
$c\le\tau$,
\TAS{} applies no perturbation.
Consequently, the output distribution is identical to that of the
unsteered model.
\end{theorem}

\begin{proof}[Proof sketch]
From the main-paper TAS update,
the steering term vanishes whenever
$c\le\tau$.
The forward pass therefore reduces exactly to the original model
computation.
\end{proof}

\noindent\textbf{Limitation of scope.}
These results are intentionally local:
they establish sufficient conditions for successful steering in the
verified \PTC{} regime, rather than global guarantees over arbitrary
transformer dynamics.

\begin{insightteal}{From the theorem to a testable condition}
Local Separability (\cref{ass:separability}) is witnessed by the
Stage~2 locator: a successful single-layer single-token activation
patch is exactly the linear separator the assumption requires
(\cref{fig:afr-profile}). The existence of $\alpha^{*}$ in
\cref{thm:tas-full} is witnessed by the oracle $\alpha$-sweep on the
V2 vector (\cref{fig:oracle-alpha}). The No-Op Preservation theorem
(\cref{thm:preserve-full}) is operationalised by preservation accuracy
on matched non-conflict records (the main-paper threshold-sweep table).
\end{insightteal}

\section{Algorithm}
\label{app:algorithm}

\TAS{} is built around a single design principle.

\begin{insightblue}{Intervening only on positive evidence}
\TAS{} is a \emph{conditional} residual intervention: the steering
signal is applied only when a calibrated detector identifies the
current trajectory as a likely \PTC{} regime. Detector-negative
queries therefore receive a bit-identical forward pass to the
unmodified model (\cref{thm:preserve-full}), and detector-positive
queries pay only a single residual-stream addition at layer
$\ell^{*}$.
\end{insightblue}

All Stage~2 and Stage~3 components (the conflict-critical layer
$\ell^{*}$, the steering vector $\Delta$, and the calibrated detector
$s(\cdot)$) are precomputed once per model family. The deployment-time
forward pass therefore differs from the unmodified model only at a
single residual-stream location, and only on queries the detector
flags. The full procedure is given in \cref{alg:tas}.

\begin{algorithm}[H]
\caption{\TAS{}: \textsc{Temporal Attractor Steering}
(inference-time, single residual-stream addition at $\ell^{*}$,
gated by a per-query detector $s(\cdot)$ trained once per model
family).}
\label{alg:tas}
\begin{algorithmic}[1]
\Require
    Query $q$; model $\theta$ with $L$ layers; conflict-critical layer
    $\ell^{*}\!\in\!\{1,\dots,L\}$; steering vector $\Delta\!\in\!\mathbb{R}^{d}$
    (V2 per-relation; \cref{app:variants}); detector $s\!:\!\mathbb{R}^{d}\!\to\![0,1]$;
    threshold $\tau$; scale $\alpha\!>\!0$.
\Ensure
    Token distribution $P_{\theta}^{\TAS}(\,\cdot\mid q)$.
\Statex
\State $h_{\ell^{*}}(q) \gets \textsc{Forward}_{\theta}^{\,1:\ell^{*}}(q)$
       \hfill \textcolor{gray!75}{\small // run forward pass up to $\ell^{*}$}
\State $c \gets s\!\bigl(h_{\ell^{*}}(q)\bigr)$
       \hfill \textcolor{gray!75}{\small // calibrated conflict score $\in[0,1]$}
\If{$c > \tau$}
    \State $h_{\ell^{*}}(q) \gets h_{\ell^{*}}(q) + \alpha\,c\,\Delta$
       \hfill \textcolor{gray!75}{\small // confidence-scaled steering update (the main-paper TAS update)}
\Else
    \State \textbf{no-op}
       \hfill \textcolor{gray!75}{\small // identical to unmodified $\theta$ (\cref{thm:preserve-full})}
\EndIf
\State $\mathit{logits} \gets \textsc{Forward}_{\theta}^{\,\ell^{*}\!+1:L}\!\bigl(h_{\ell^{*}}(q)\bigr)$
       \hfill \textcolor{gray!75}{\small // resume forward pass to logits}
\State \Return $\mathrm{softmax}(\mathit{logits})$
\end{algorithmic}
\end{algorithm}

\paragraph{Per-line design rationale.}
\begin{itemize}\setlength\itemsep{2pt}
    \item \textbf{Detector $s(\cdot)$ (line 2).} A calibrated linear
    probe on $h_{\ell^{*}}$ that prevents unnecessary intervention on
    non-conflict queries; this conditional gate is what allows
    preservation accuracy to remain in $[0.85, 0.99]$ across models
    (the main-paper threshold-sweep table).
    \item \textbf{Steering layer $\ell^{*}$ (line 1).} Identified via
    activation patching (the main-paper locator table, \cref{fig:afr-profile})
    because temporal preference reversals emerge sharply at specific
    intermediate locations rather than uniformly across depth.
    \item \textbf{Steering vector $\Delta$ (line 4).} Constructed from
    activation differences $h_{\ell^{*}}(q, c_{\mathrm{temp}}) -
    h_{\ell^{*}}(q)$ averaged within each Wikidata relation.
    Relation-conditioned averaging (V2) isolates the shared geometric
    structure of temporal recovery within a relation family
    (\cref{app:variants}).
    \item \textbf{Confidence-scaled magnitude $\alpha\,c\,\Delta$ (line~4).}
    Detector confidence modulates the perturbation strength so the
    largest interventions land on the most-confident conflicts; this
    is what enables the favourable Recovery~/~PA trade-off in
    \cref{fig:oracle-alpha}.
\end{itemize}

\begin{insightteal}{\TAS{} as a single residual-stream addition}
On a detector-positive query, the only operation \TAS{} performs
beyond the unmodified forward pass is a single rank-zero residual
addition at one layer. There is no decoding-time loop, no
beam-rescoring, no retrieval call, and no extra forward pass over
$\theta$.
\end{insightteal}

\section{Subject-Disjoint Split and Corrected Detector Protocol}
\label{app:splits}

\paragraph{Split protocol.}
All detector learning, direction construction, operating-point selection,
and final evaluation use one versioned, subject-disjoint split
(a single versioned split, seed $20260712$). The grouping unit is the
Wikidata subject: every record of a subject is assigned to exactly one of
\emph{train} ($60\%$), \emph{validation} ($15\%$), \emph{calibration}
($10\%$), \emph{test} ($15\%$). Subjects that share a normalized surface
label (homonyms, e.g.\ distinct places both labeled ``Fântânele'') are
merged into one grouping component, so neither a subject nor an identical
prompt string can cross a partition. Subjects are stratified by their
rarest relation so the small relations (P35, P169) spread evenly, then
allocated deterministically by largest remainder. Roles are disjoint by
construction: \emph{train} fits the probe and constructs $\Delta_{\TAS}$
and $\Delta_{\mathrm{ITI}}$; \emph{validation} selects $\alpha$ and
operating points; \emph{calibration} fits the isotonic calibrator and
detector thresholds; \emph{test} reports all final metrics.

\begin{table}[htbp]
\centering\small
\setlength{\tabcolsep}{4pt}
\begin{tabular}{lccccc}
\toprule
\hdrrow
\hdr{Split} & \hdr{Records} & \hdr{Subjects} &
\multicolumn{3}{c}{\hdr{PTC positives (per model)}} \\
\hdrrow
 & & & \hdr{Q1.5B} & \hdr{Mist.} & \hdr{Llama} \\
\midrule
Train       & 5240 & 3270 & 123 & 311 & 399 \\
Validation  & 1306 & 822  & 35  & 83  & 91  \\
Calibration & 863  & 541  & 19  & 59  & 65  \\
Test        & 1337 & 818  & 32  & 71  & 87  \\
\bottomrule
\end{tabular}
\caption{Subject-disjoint split counts and per-model PTC-positive counts
(Qwen-2.5-7B counts $189/58/32/39$ for train/val/calib/test are omitted
for space). Every held-out partition has $\ge 19$ positives for every
model. A leakage audit confirms zero cross-split overlap of record IDs,
subject QIDs, normalized subject labels, subject--relation pairs, and
exact/normalized prompts; answer-entity QIDs may recur across unrelated
subjects and are reported descriptively only.}
\label{tab:split-counts}
\end{table}

\paragraph{Corrected detector metrics.}
Refitting the probe on \emph{train}, calibrating on \emph{calibration},
and reporting on \emph{test} (\cref{tab:detector-corrected}) yields
materially lower separability than the legacy instance-level split, which
had let records of one subject cross folds \emph{and} fit the isotonic
calibrator on the same held-out set used to report AUPRC/AUROC. Because
AUROC is base-rate invariant, the AUROC drop ($-0.03$ to $-0.13$) isolates
the leakage effect; the corrected AUPRC is additionally lower because it is
computed at the true test base rate (PTC prevalence $2$--$7\%$) rather than
on a subsampled, artificially balanced test set. The effect is largest for
the smallest model (Qwen-2.5-1.5B, corrected test AUROC $0.47$, at chance),
indicating its detector does not generalize to unseen subjects.

\begin{table}[htbp]
\centering\small
\setlength{\tabcolsep}{4pt}
\begin{tabular}{lcccc}
\toprule
\hdrrow
\hdr{Model} & \hdr{AUROC} & \hdr{AUROC} & \hdr{AUPRC} & \hdr{Brier} \\
\hdrrow
 & \hdr{(corrected)} & \hdr{(legacy)} & \hdr{(true base)} & \\
\midrule
Qwen-2.5-1.5B   & 0.467 & 0.601 & 0.028 & 0.024 \\
Qwen-2.5-7B     & 0.571 & 0.628 & 0.056 & 0.029 \\
Mistral-7B-v0.3 & 0.657 & 0.687 & 0.092 & 0.053 \\
Llama-3.1-8B    & 0.598 & 0.626 & 0.097 & 0.061 \\
\bottomrule
\end{tabular}
\caption{Detector metrics on \emph{test} under the corrected
subject-disjoint protocol vs.\ the legacy leaked instance-level split.
Probe fit on train, isotonic calibration on calibration, thresholds on
validation, metrics on test. AUROC (base-rate invariant) falls on every
model; corrected AUPRC uses the deployment-real test base rate. These
supersede the leaked detector AUPRC values in the main-paper end-to-end
table, whose detector-gated operating points are regenerated on this split
in the paired evaluation.}
\label{tab:detector-corrected}
\end{table}

\subsection{Calibration Procedure}
\label{app:detector-calib}
The detector is an $\ell_2$-regularized logistic probe
(\texttt{sklearn}, $C{=}1$, \texttt{class\_weight=balanced}) on the
layer-$\ell^{*}$ residual stream, fit on \emph{train} only; features are
standardized with train statistics and the scaler is folded back into a
raw-space linear probe. An isotonic regressor fit on \emph{calibration}
only maps the raw score to $p(\PTC)\in[0,1]$. Seeds are fixed
($20260712$ for the split; probe and calibrator are deterministic given
data). Calibration positives are the binding constraint: Qwen-2.5-1.5B has
only $19$ calibration positives, so its isotonic fit is high-variance and
its calibrated estimates are reported \emph{descriptively}; the other
models have $32$--$65$ (\cref{tab:split-counts}). Because $\Delta$ and the
probe use \emph{train} only, validation/calibration could be enlarged by
reallocating subjects from train (not from test) without contaminating
direction construction; we did not, to preserve train positives, and note
that repeated subject-disjoint cross-validation would be more appropriate
for detector diagnostics than a single small calibration fold. Validation
and test are never merged.

\subsection{Test-Set Discrimination}
\label{app:detector-discrim}
Precision-recall, ROC, and calibrated-score distributions on the
\emph{held-out} test subjects (\cref{fig:detector-corrected-pr,fig:detector-scoredist})
confirm the corrected assessment: separability is weak and model-dependent,
and the PTC and non-conflict score distributions overlap heavily.

\begin{figure}[htbp]
\centering
\includegraphics[width=0.49\columnwidth]{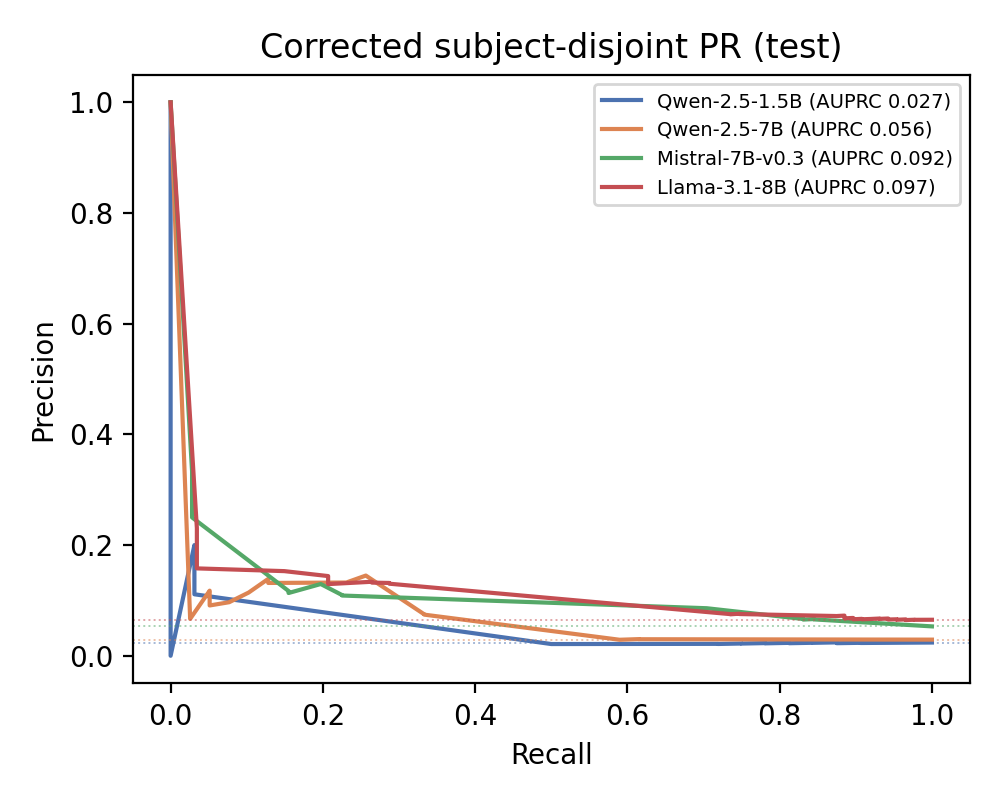}\hfill
\includegraphics[width=0.49\columnwidth]{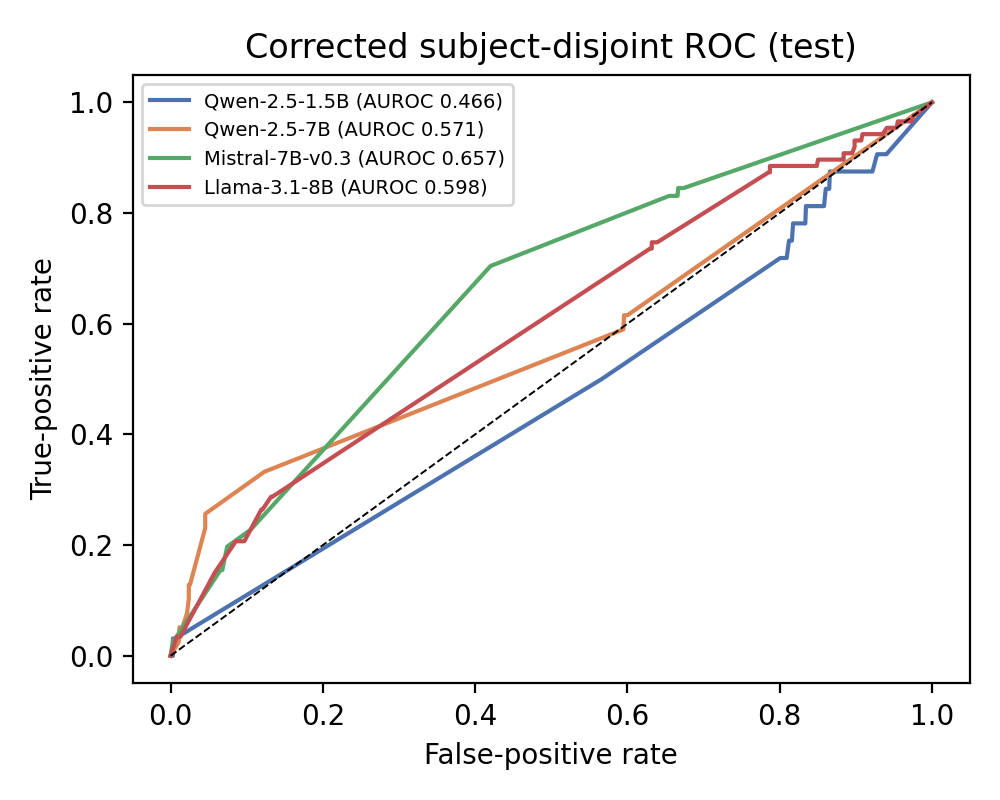}
\caption{Corrected subject-disjoint precision-recall (left) and ROC (right)
curves on held-out test subjects. Dotted lines mark each model's PTC base
rate (PR) and chance (ROC). These \emph{replace} the legacy leaked-split PR
curves (\cref{fig:detector-pr}, retained only as a historical diagnostic).}
\label{fig:detector-corrected-pr}
\end{figure}

\begin{figure}[htbp]
\centering
\includegraphics[width=0.98\columnwidth]{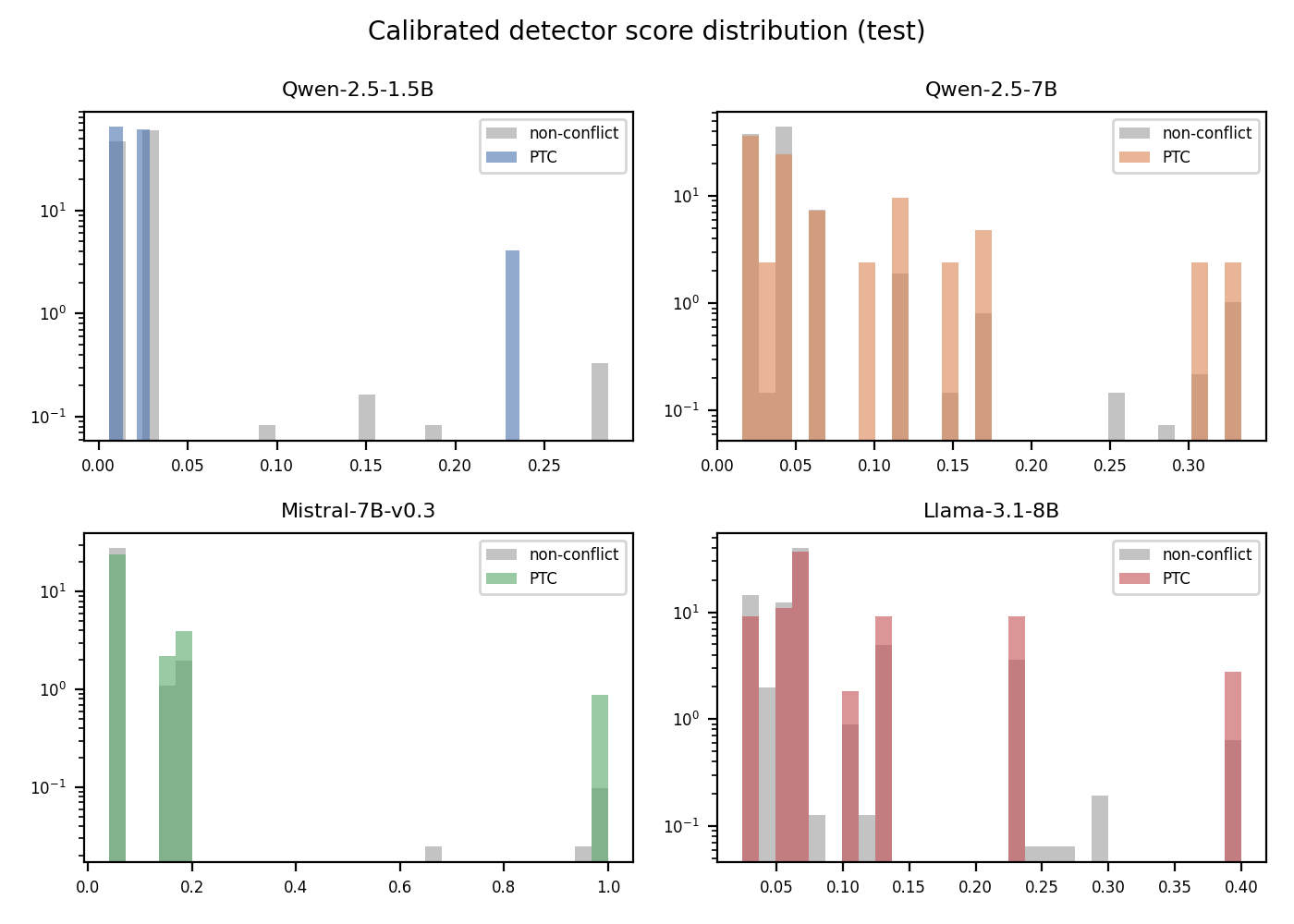}
\caption{Calibrated detector-score distributions on test (log density):
verified \PTC{} (colour) vs.\ non-conflict (grey). The distributions
overlap heavily for every model, which is why no threshold cleanly
separates conflicts.}
\label{fig:detector-scoredist}
\end{figure}

\subsection{Calibration Quality}
\label{app:detector-calibquality}
Reliability curves (\cref{fig:detector-reliability}) and the low Brier
($0.024$--$0.061$) and ECE ($<0.02$) in \cref{tab:detector-corrected}
indicate the probabilities are \emph{well calibrated} to the low \PTC{}
base rate. Calibration, however, does not imply \emph{ranking}: a probe
that outputs the base rate for every query is perfectly calibrated yet
useless for gating. The corrected AUROC ($0.47$--$0.66$) is the honest
measure of separability, and it is near chance for Qwen-2.5-1.5B.

\begin{figure}[htbp]
\centering
\includegraphics[width=0.62\columnwidth]{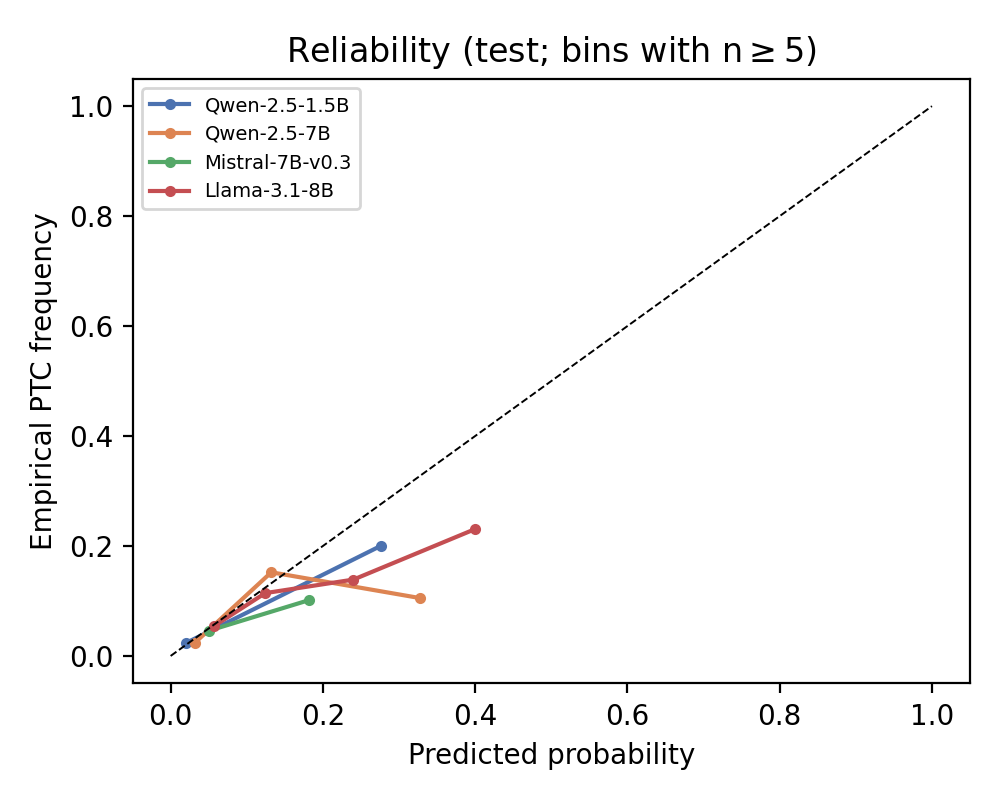}
\caption{Reliability (test; bins with $n\ge5$). Well-calibrated
probabilities do not imply discriminative ranking.}
\label{fig:detector-reliability}
\end{figure}

\subsection{Operating-Point Analysis}
\label{app:detector-ops}
\cref{tab:detector-ops} reports the primary threshold $\tau{=}0.15$ with
full confusion counts and subject-level bootstrap $95\%$ CIs
($B{=}10{,}000$, seed $0$; subjects resampled since records are not
independent within a subject). Thresholds are selected on \emph{validation}
only. Two facts stand out. First, at $\tau{=}0.15$ the gate fires on
$0.5$--$10.9\%$ of queries at FPR $0.5$--$10.3\%$ while its TPR (fraction of
true \PTC{} gated) is only $3$--$23\%$. Second, no validation-selected
threshold yields a usable operating point: the validation-$F_1$-optimal
threshold gives test TPR $0.00$ for Qwen-2.5-1.5B (its best threshold
catches \emph{no} test conflict), $0.21$--$0.33$ for the others; a
conservative validation-FPR-$\le5\%$ threshold is attainable for all
models but only by further lowering TPR (threshold sweeps for TPR/FPR and
gated Recovery/preservation are plotted in \cref{fig:detector-threshold}).
Full TP/FP/TN/FN at every threshold and operating point are released with the code and data.

\begin{table}[htbp]
\centering\small
\setlength{\tabcolsep}{3pt}
\begin{tabular}{lcccccc}
\toprule
\hdrrow
\hdr{Model} & \hdr{TP/FP/FN/TN} & \hdr{TPR [95\% CI]} &
\hdr{FPR [95\% CI]} & \hdr{\%gated} & \hdr{gRec} & \hdr{PA} \\
\midrule
Qwen-2.5-1.5B   & 1/6/31/1299   & 0.03 [.00,.11] & 0.005 [.002,.009] & 0.5\%  & 0.031 & 1.00 \\
Qwen-2.5-7B     & 5/33/34/1265  & 0.13 [.03,.25] & 0.025 [.014,.040] & 2.8\%  & 0.077 & 0.97 \\
Mistral-7B-v0.3 & 16/130/55/1136& 0.23 [.11,.35] & 0.103 [.074,.135] & 10.9\% & 0.056 & 1.00 \\
Llama-3.1-8B    & 13/72/74/1178 & 0.15 [.07,.24] & 0.058 [.040,.077] & 6.4\%  & 0.069 & 0.98 \\
\bottomrule
\end{tabular}
\caption{Detector operating point at the primary $\tau{=}0.15$ on
held-out test subjects. TPR = fraction of true \PTC{} gated (\% PTC
gated); FPR = fraction of non-conflict queries falsely gated; \%gated =
all queries gated; gRec = gated-\TAS{} Recovery; PA = preservation. CIs are
subject-level bootstraps; Qwen-2.5-1.5B (32 test positives) is descriptive.
Validation-selected $F_1$/Youden and low-FPR thresholds are analysed in the
text; all are reported in the released operating-point tables.}
\label{tab:detector-ops}
\end{table}

\begin{figure}[htbp]
\centering
\includegraphics[width=0.49\columnwidth]{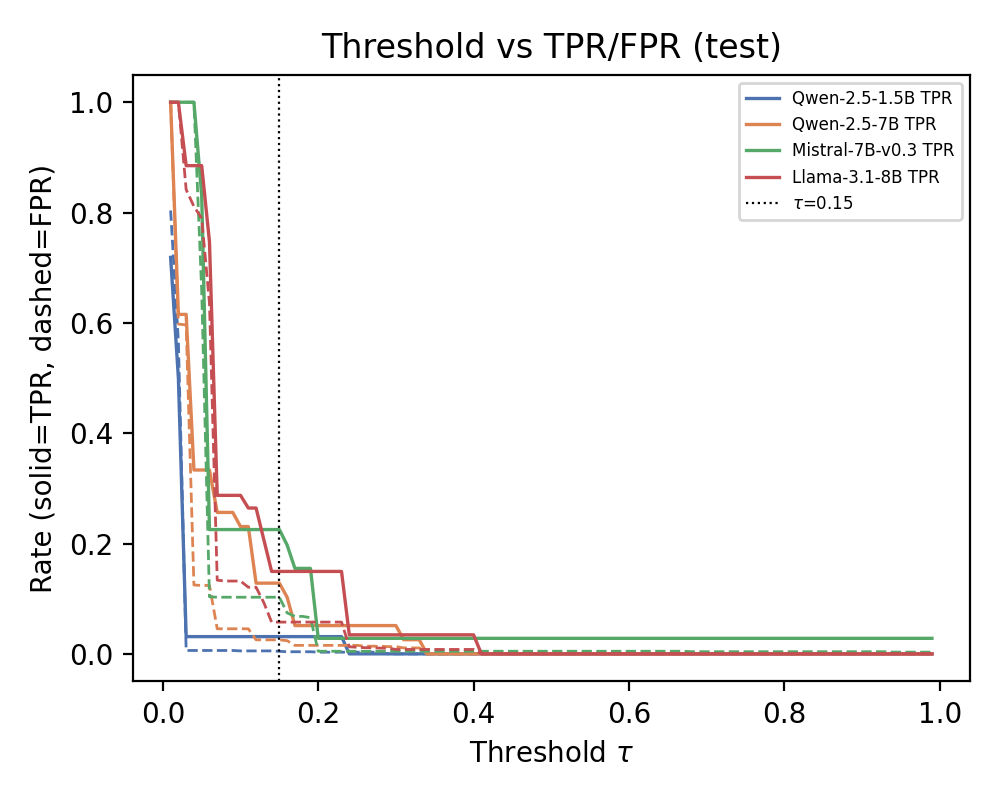}\hfill
\includegraphics[width=0.49\columnwidth]{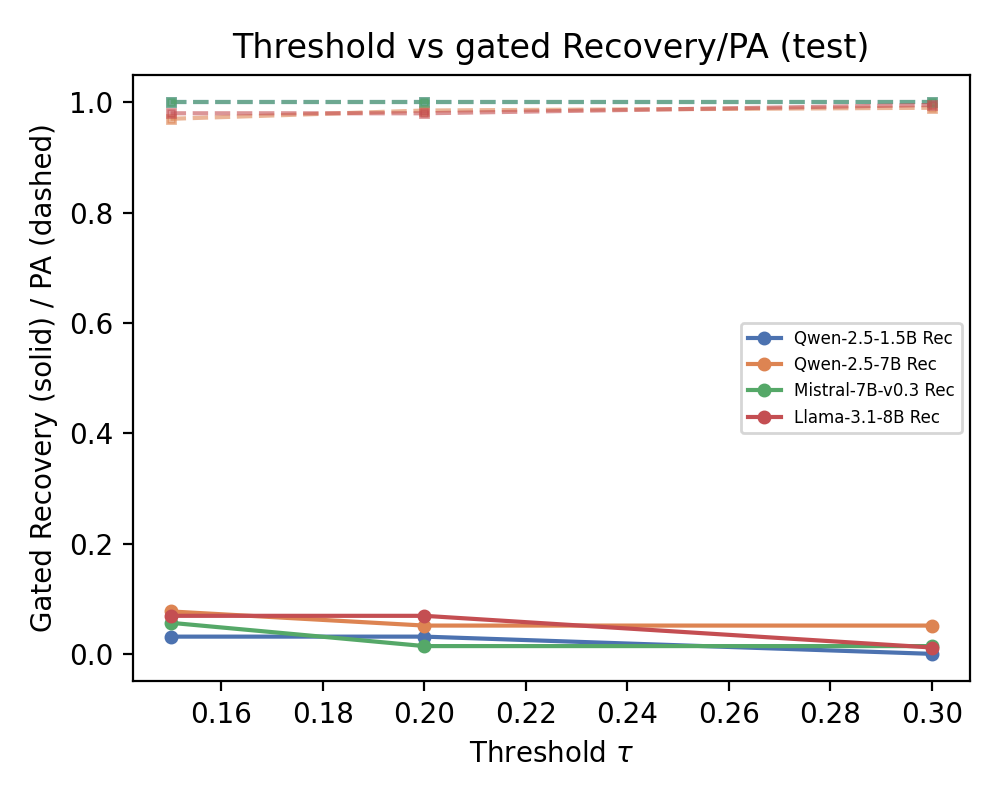}
\caption{Threshold sweeps on test. \emph{Left}: TPR (solid) and FPR
(dashed) vs.\ threshold $\tau$; the dotted line marks the primary
$\tau{=}0.15$. \emph{Right}: gated-\TAS{} Recovery (solid) and preservation
(dashed) at the deployed thresholds $\tau\in\{0.15,0.20,0.30\}$. Raising
$\tau$ trades already-low recall for fewer false alarms; no threshold makes
the gated pipeline competitive with always-on steering.}
\label{fig:detector-threshold}
\end{figure}

\subsection{False-Alarm Analysis}
\label{app:detector-falsealarm}
False positives (non-conflict queries gated at $\tau{=}0.15$) concentrate
on the \emph{rare} relations. For Mistral-7B-v0.3 the per-relation FPR is
$0.29$ for head of state (P35) and $0.26$ for CEO (P169) versus
$0.06$--$0.13$ for the three high-volume relations; Qwen-2.5-7B shows the
same pattern (P35 FPR $0.19$). Because P35/P169 carry the strongest
per-record \PTC{} signal, the probe over-fires on exactly the relations
where conflict is a priori most likely, inflating false alarms rather than
catching more true conflicts. Full per-relation, per-model false-positive
counts are in the released artifacts.

\subsection{Detector Limitations}
\label{app:detector-limits}
The detector is limited by (i) small positive counts (32--87 test
positives; 19--65 calibration positives), which widen operating-point CIs,
especially for Qwen-2.5-1.5B; (ii) near-chance separability for the
smallest model (AUROC $0.47$); and (iii) a low true base rate
($2$--$7\%$) that makes any fixed-threshold gate either miss most conflicts
or accept many false alarms. We therefore treat conflict \emph{detection}
under subject-disjoint evaluation as unresolved, and report the always-on
verified-\PTC{} evaluation as the mechanistic result.

\section{Baselines, Metrics, and Reporting Conventions}
\label{app:baseline-metrics}

This appendix collects the operational definitions of the metrics
reported throughout the paper and documents the baselines we compare
against. The exposition prioritizes \emph{reproducibility}: every
metric can be recomputed from the per-record screening output, and
the ITI baseline construction is given in enough detail that an
independent implementation should produce numerically identical
results to ours.

\subsection{Baseline Methods}
We compare against: \textbf{(i) Standard prompting} (no intervention);
\textbf{(ii) Temporal prompting} (\emph{``As of $\langle t \rangle$,
$\langle q\rangle$''}, a strong baseline since \PTC{} is defined
relative to it); \textbf{(iii) RAG}; \textbf{(iv) ITI}
\citep{li2023inference} as an activation-engineering comparator;
\textbf{(v) Contrastive decoding}. ROME / MEMIT
\citep{meng2022locating,meng2022mass} are discussed as conceptual
weight-editing comparators only (they modify model parameters rather
than resolve conflicts at query time).

\subsection{Metric Definitions}

\paragraph{\OPR{} ($\downarrow$).}
$\OPR = |\{q \in \mathcal{B} : \Ptheta(a_{\mathrm{old}}\mid q) > \Ptheta(a_{\mathrm{new}}\mid q)\}| / |\mathcal{B}|$.

\paragraph{Recovery ($\uparrow$).}
Fraction of records where the intervention (temporal cue or \TAS{})
makes $\overline{\log P}_\theta(a_{\mathrm{new}}\mid \cdot)
> \overline{\log P}_\theta(a_{\mathrm{old}}\mid \cdot)$.

\paragraph{\PTC{} rate ($\uparrow$).}
Conjunction of \OPR{}$=1$ under standard prompting AND Recovery$=1$
under temporal cue (the PTC definition in the main paper).

\paragraph{Filtered \PTC{} rate ($\uparrow$).}
\PTC{} rate computed only on records that pass the model-side
knowledge-recovery filter (\cref{app:filter}).

\paragraph{Elicitation Gap (\EG{}, $\uparrow$).}
$\EG = \mathbb{E}_q[\overline{\log P}(a_{\mathrm{new}}\mid q, \mathrm{intervention})
- \overline{\log P}(a_{\mathrm{new}}\mid q)]$. Reported in log-probability
and per-token probability.

\paragraph{\PA{} ($\uparrow$).}
Fraction of matched non-conflict records whose argmax over
$\{a_{\mathrm{old}}, a_{\mathrm{new}}\}$ is unchanged by the
intervention.

\paragraph{Detector \AUPRC{} / AUROC ($\uparrow$).}
Area under the precision--recall / ROC curve on a held-out test split
of the binary \PTC{} prediction~task.

\subsection{ITI Baseline Construction}
For each model we construct an ITI direction at the same $\ell^*$ as
\TAS{}, using a prefers-new vs.\ prefers-old contrast on the full
screening set under the standard prompt:
\[
\Delta_{\mathrm{ITI}} =
\mathrm{mean}\bigl(h_{\ell^*}\!(q) \mid \theta \text{ prefers } a_{\mathrm{new}}\bigr)
- \mathrm{mean}\bigl(h_{\ell^*}\!(q) \mid \theta \text{ prefers } a_{\mathrm{old}}\bigr).
\]
This is intentionally a strong baseline: the contrast pair is already
\PTC{}-relevant (not the generic truthful/deceptive contrast), so the
comparison isolates the value added by \TAS{}'s \emph{verified-conflict}
construction over a population-level \emph{behavioral} one. Each side
of the contrast is approximated by a stratified random sample of at
most $500$ screening records (seed $0$); the prefers-old side is the
verified \PTC{} subset itself, and the prefers-new side is sampled
from screening records where the model's argmax over $\{a_{\mathrm{old}},
a_{\mathrm{new}}\}$ under the standard prompt is $a_{\mathrm{new}}$.
Each sampled instance contributes one forward pass through the standard
prompt at $\ell^{*}$. We then sweep $\alpha \in \{0.5,1,2,3,4,6,8\}$
identically to \TAS{}, against the same $200$-instance non-conflict
control set (seed $0$), and select the operating point that maximises
$J = \mathrm{Recovery} - \lambda(1 - \mathrm{PA})$ with $\lambda = 1.0$.
Only the construction of $\Delta$ differs between \TAS{} and ITI: every
other setting ($\ell^{*}$, $\alpha$ grid, control set, $J$ objective) is
shared.

\section{Paired Held-Out TAS-vs-ITI Evaluation}
\label{app:paired}

\paragraph{Reproducible protocol.}
The paired comparison runs entirely on the authoritative subject-disjoint
split (\cref{app:splits}). Both directions are constructed on \emph{train}
only: $\Delta_{\TAS}$ is the V2 per-relation mean of
$h_{\ell^*}(q,c_{\mathrm{temp}})-h_{\ell^*}(q)$ over train filtered-\PTC{}
records; $\Delta_{\mathrm{ITI}}$ is the raw mean-difference of
standard-prompt $h_{\ell^*}$ between a prefers-new positive group
($\le\!500$, seed $0$) and the train verified-\PTC{} negative group, read
from the cached (benchmark-ordered) standard-prompt hidden states so ITI
requires no forward passes. Neither vector is unit-normalized or
norm-matched; the shared $\alpha$ grid $\{0.5,1,2,3,4,6,8\}$ explores
scale independently for each method, and $\alpha$ is selected on
\emph{validation} by $J=\mathrm{Recovery}-(1-\mathrm{PA})$ (ties broken
toward smaller $\alpha$). Each method is evaluated \emph{once} on the
identical test verified-\PTC{} records and the identical $200$ clean
controls; per-instance decisions are joined one-to-one by record id for
the exact McNemar tests. Directions and their train-record ids are
asserted train-only before evaluation. Construction counts, $\Delta$
norms, selected $\alpha$, and pinned model snapshot revisions are given in
\cref{tab:paired-construction}:

\begin{table}[htbp]
\centering\small\setlength{\tabcolsep}{3.5pt}
\begin{tabular}{lccccccl}
\toprule
\hdrrow
\hdr{Model} & \hdr{ITI$^{+}$} & \hdr{ITI$^{-}$} & \hdr{$\lVert\Delta_{\mathrm{ITI}}\rVert$} &
\hdr{TAS $n$} & \hdr{$\alpha_{\TAS}$} & \hdr{$\alpha_{\mathrm{ITI}}$} & \hdr{snapshot} \\
\midrule
Qwen-2.5-1.5B   & 500 & 123 & 19.9 & 64  & 2 & 6 & 8faed761 \\
Qwen-2.5-7B     & 500 & 189 & 52.3 & 149 & 2 & 8 & d1497293 \\
Mistral-7B-v0.3 & 500 & 311 & 5.9  & 290 & 4 & 8 & caa1feb0 \\
Llama-3.1-8B    & 500 & 399 & 12.8 & 373 & 6 & 6 & d04e592b \\
\bottomrule
\end{tabular}
\caption{Train-only direction construction and selected operating points.
ITI$^{+}$/ITI$^{-}$: positive (prefers-new) / negative (verified-\PTC{})
group sizes; TAS $n$: filtered-\PTC{} train records; snapshot: leading
Hugging Face commit (fp16, torch 2.11, transformers 5.5).}
\label{tab:paired-construction}
\end{table}

\paragraph{Paired significance.}
On held-out test data, no Recovery difference is significant after Holm
correction across the four models: discordant (TAS-wins/ITI-wins) and
raw$\to$Holm McNemar $p$ are $5/3$ ($0.73\to1$), $2/2$ ($1\to1$), $6/10$
($0.45\to1$), $19/13$ ($0.38\to1$) for Qwen-1.5B, Qwen-7B, Mistral,
Llama. For preservation, ITI preserves more clean answers, significantly
so on Qwen-2.5-1.5B (discordant $3/18$, $p_{\mathrm{raw}}=0.0015$,
$p_{\mathrm{Holm}}=0.006$) and numerically on three of four models. The
earlier oracle, full-data comparison (\cref{tab:oracle-iti-diag}) reported
large Qwen \TAS{} advantages ($+0.142$, $+0.152$); these do not survive
held-out, validation-selected evaluation and are retained only as a
diagnostic of construction under test-set-optimal $\alpha$.

\begin{table}[htbp]
\centering\small\setlength{\tabcolsep}{4pt}
\begin{tabular}{lcccc}
\toprule
\hdrrow
\hdr{Model} & \hdr{ITI Rec.} & \hdr{TAS Rec.} & \hdr{$\Delta$} & \hdr{note} \\
\midrule
Qwen-2.5-1.5B   & 0.373 & 0.515 & $+0.142$ & oracle \\
Qwen-2.5-7B     & 0.465 & 0.617 & $+0.152$ & oracle \\
Mistral-7B-v0.3 & 0.349 & 0.293 & $-0.056$ & oracle \\
Llama-3.1-8B    & 0.332 & 0.382 & $+0.050$ & oracle \\
\bottomrule
\end{tabular}
\caption{\textbf{Diagnostic only (superseded).} Oracle, full-data
\TAS{}-vs-ITI with test-set-optimal $\alpha$ and directions built from
the whole benchmark. Included to document why the held-out main-paper result differs: removing oracle $\alpha$ and full-data
directions, and evaluating on subject-disjoint held-out records, shrinks
the Qwen gaps to non-significance.}
\label{tab:oracle-iti-diag}
\end{table}

\section{Layer Localization: Patch vs.\ Edit}
\label{app:localization}

\paragraph{Split usage and reproduction.}
The per-layer answer-flip rate (AFR) is recomputed from the cached
per-instance patching outputs, mapped one-to-one onto the subject-disjoint
manifest (\cref{app:splits}); the full-data AFR peaks exactly reproduce
the main-paper locator table ($0.723/0.851/0.824/0.816$). Held-out
\emph{test} AFR reproduces these peaks ($0.60/0.86/0.87/0.82$ for
Qwen-1.5B/Qwen-7B/Mistral/Llama; Qwen-1.5B has only $n{=}10$ test
positives, so its localization estimate is descriptive). The
$0.05$-plateau widths ($6/6/14/9$) are unchanged under $0.03$ and $0.10$
thresholds, so ``Qwen tightest, Mistral broadest, Llama intermediate'' is
threshold-stable.

\paragraph{Three distinct analyses.}
We keep separate: \textbf{(A)} activation \emph{patching} at layer $\ell$
(donor $h_\ell(q,c_{\mathrm{temp}})$ into the standard run); \textbf{(B)}
\emph{portability} (applying the same numeric train-only
$\Delta_{\ell^*}$ at other layers); and \textbf{(C)} \emph{layer-specific}
train-only $\Delta_\ell=\mathrm{mean}[h_\ell(q,c_{\mathrm{temp}})-h_\ell(q)]$.
We never compare $\Delta_\ell$ and $\Delta_{\ell^*}$ by cosine (different
residual spaces); (B) is described only as a portability experiment.

\paragraph{Geometry and correlations (held-out test).}
The mean cosine of each test instance's temporal shift with the train-only
$\Delta_\ell$ is high at early layers ($0.91$--$0.96$) and low at $\ell^*$
($0.36$--$0.51$). Held-out steering Recovery correlates \emph{negatively}
with patching AFR across evaluated layers (Spearman
$\rho=-0.67,-0.82,-0.98,0.0$ for the four models) and \emph{positively}
with shift coherence ($\rho=0.89^{*},0.98^{**},0.98^{**},0.0$; ${}^{*}p<0.05$).
Llama's steering Recovery is nearly flat across layers, hence its zero
correlation. Paired McNemar (off-peak vs.\ $\ell^*$) is significant only
for Mistral, where early layers \emph{beat} $\ell^*$ ($p{=}0.02$); the
small test-PTC counts ($10$--$87$) limit power elsewhere. Full per-layer
AFR curves (with bootstrap bands), steering-by-layer Recovery/PA, and the
geometry table are released with the code and data; the multi-panel figure
is \cref{fig:layer-localization}.

\begin{figure}[htbp]
    \centering
    \includegraphics[width=\columnwidth]{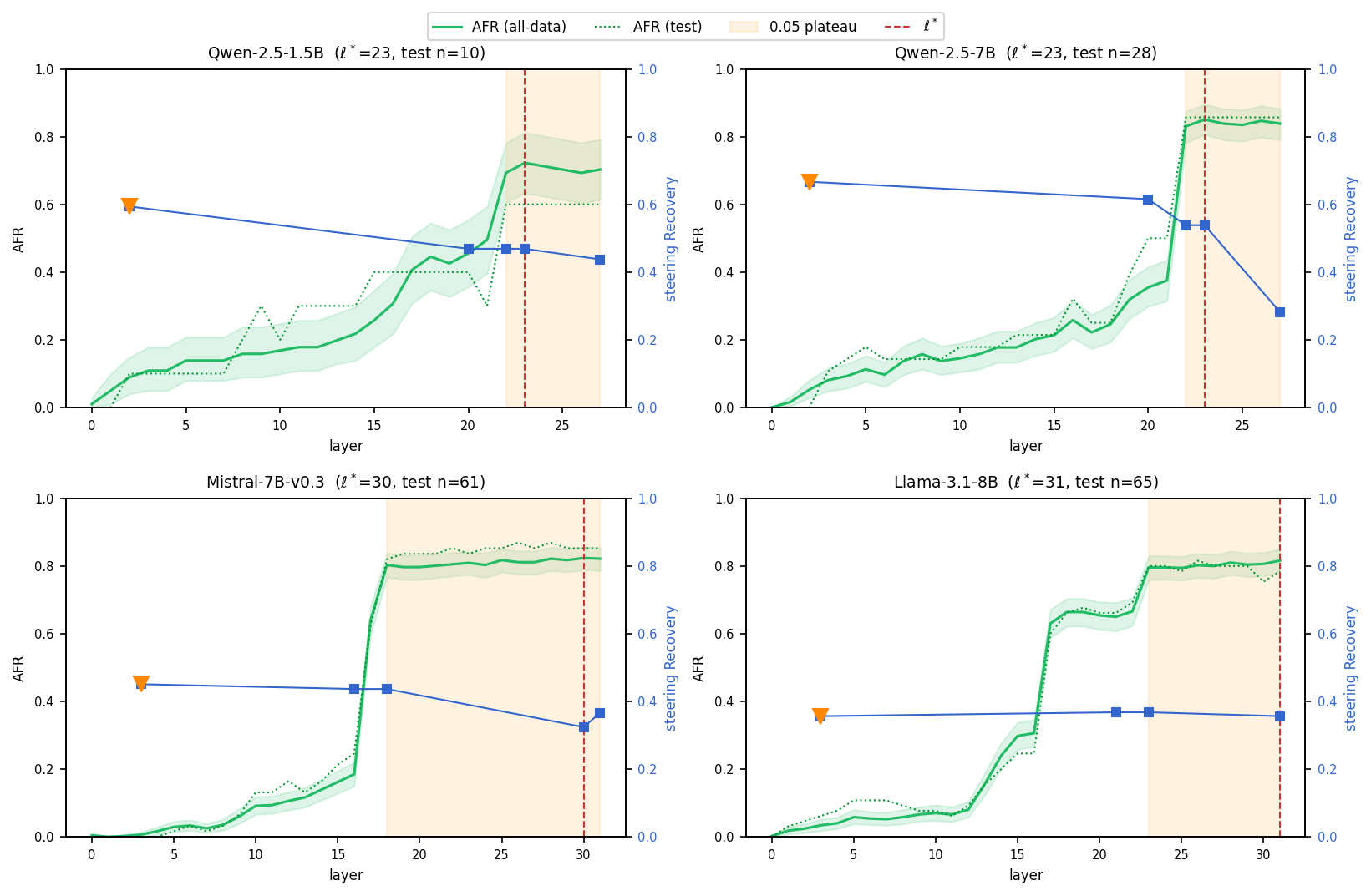}
    \caption{Per-model layer localization. Green: patching AFR by layer
    (all-data solid, test dotted) with bootstrap band; orange shading:
    $0.05$ plateau; red dashed: $\ell^*$; blue: held-out steering Recovery
    at evaluated layers; triangle: off-peak control. Steering Recovery is
    high at early layers where patching AFR is near zero, illustrating that
    patch-localization and edit-efficacy have different layer profiles.}
    \label{fig:layer-localization}
\end{figure}

\section{Knowledge-Recovery Filter}
\label{app:filter}

The knowledge-recovery filter is the key methodological device that
distinguishes \emph{true \PTC{}} from \emph{knowledge absence}. Without
this filter, the headline \PTC{} metric would silently aggregate two
failure modes that look identical under raw measurements but have
opposite implications for any test-time intervention.

\paragraph{Mechanics.} A record is retained for a given model only if
\[
\overline{\log P}_{\theta}(a_{\mathrm{new}} \mid q, c_{\mathrm{temp}})
\;\ge\; \tau_{\mathrm{rec}},
\]
that is, only if the model assigns at least $e^{\tau_{\mathrm{rec}}}$
per-token probability to the newer answer when given the temporal
elicitation cue. We use $\tau_{\mathrm{rec}}=-3$ throughout the paper,
corresponding to roughly a $5\%$ per-token probability floor on the
newer answer. Records that fail this threshold are flagged as
\emph{knowledge-absent for that model} and excluded from filtered
\PTC{}-rate computations (but retained in the raw \PTC{} computations
for honesty about the benchmark's composition).

\paragraph{Model-specificity.} Because the filter depends on the
model's actual probability assignment, the kept-record set is
\emph{different for each model}. A record may belong to the \PTC{}
regime for one model but the knowledge-absence regime for another. We
therefore always report raw and filtered \PTC{} as a pair
(\cref{fig:filter-capacity}), never collapsing them into a single number.

\paragraph{Reading the filter as a knowledge-breadth measure.} The
fraction of the benchmark each model retains as recoverable is itself
an indirect measure of that model's parametric knowledge breadth on the
benchmark's relation distribution: a larger or more recent model
retains a larger fraction. \cref{fig:filter-capacity} visualises the
joint trajectory of the four models in (kept-fraction, filtered \PTC{})
space.

\begin{figure}[htbp]
    \centering
    \includegraphics[width=0.78\textwidth]{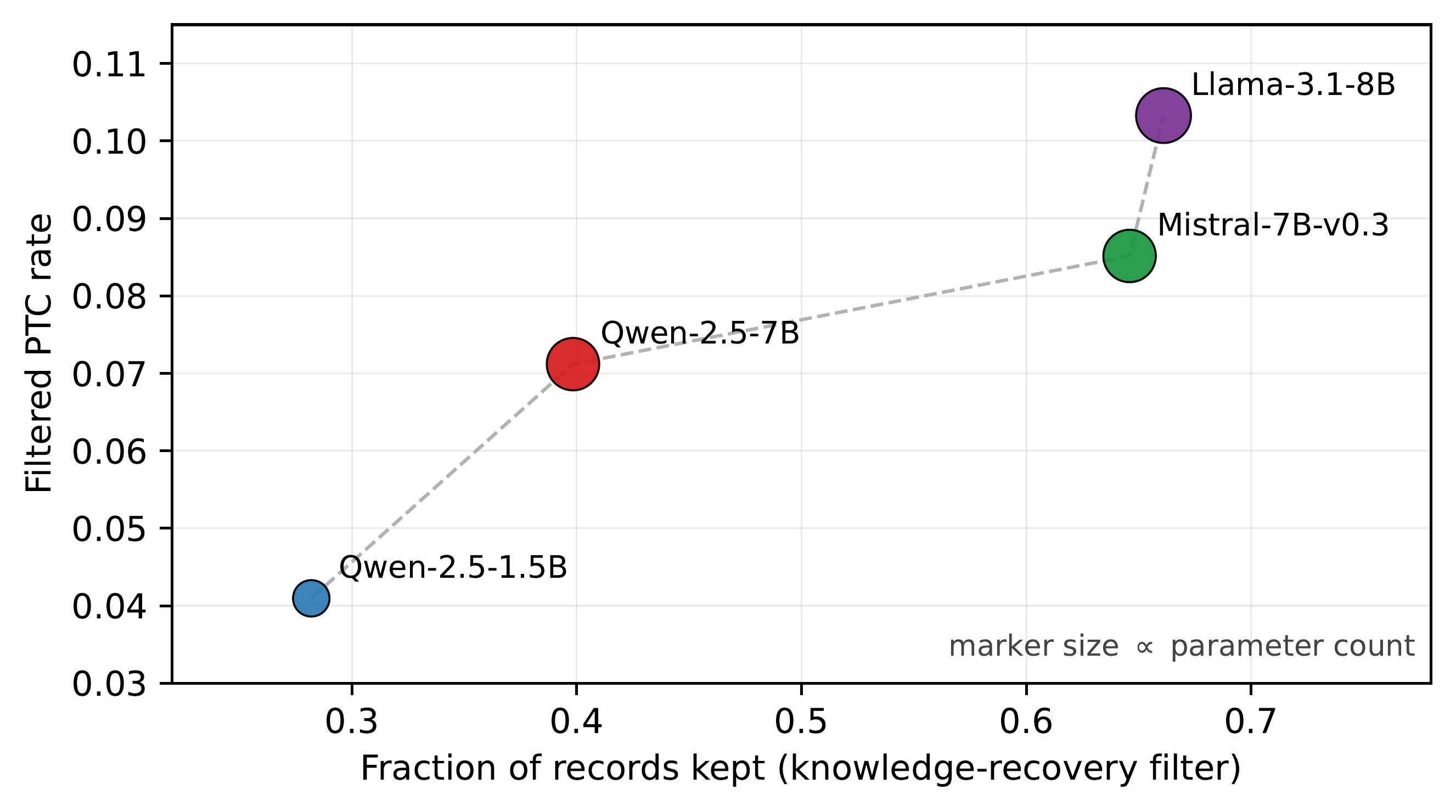}
    \caption{The knowledge-recovery filter across the evaluated models.
    Each point is a model placed at its measured
    (kept-fraction, filtered \PTC{}) coordinate, where
    \emph{kept-fraction} is the share of the $8{,}746$-record benchmark
    that passes
    $\overline{\log P}_{\theta}(a_{\mathrm{new}}\mid q, c_{\mathrm{temp}})
    \ge \tau_{\mathrm{rec}}$ at $\tau_{\mathrm{rec}}=-3$, and
    \emph{filtered \PTC{}} is the \PTC{} rate on that retained
    subset. Marker size scales with parameter count. Both axes climb
    together across the evaluated models: those retaining less of the
    benchmark as recoverable also accumulate a smaller
    outdated-preference margin on what they retain, while those that
    retain more accumulate a proportionally larger
    margin. Because the models differ in family, tokenizer, and cutoff
    as well as size, this is an observed trend rather than a controlled
    effect of capacity. It is the empirical pattern that motivates
    reporting filtered \PTC{} alongside raw \PTC{} throughout the
    paper. Numbers are recomputed directly from per-instance
    log-probabilities and reproduce the main-paper Phase~1 figure in the
    main body.}
    \label{fig:filter-capacity}
\end{figure}

\begin{insightgreen}{Increase of filtered \PTC{} across the evaluated models}
Kept fraction climbs from $28.2\%$ (Qwen-1.5B) to $66.1\%$
(Llama-3.1-8B) while filtered \PTC{} rises alongside it
$0.041 \to 0.071 \to 0.085 \to 0.103$. The two move together across
the four evaluated models, which differ in family, tokenizer, and
cutoff as well as size; this is an observed trend, not a controlled
scaling result.
\end{insightgreen}

\begin{insightorange}{Reporting raw and filtered \PTC{} together}
Because the filter depends on the model's own probability assignment,
the kept-record set differs across models. Reporting only \emph{raw}
\PTC{} silently aggregates true conflict with knowledge absence;
reporting only \emph{filtered} \PTC{} hides the benchmark-composition
shift across models. We therefore report both throughout the paper
(\cref{fig:filter-capacity}).
\end{insightorange}

\paragraph{Threshold rationale.}
$\tau_{\mathrm{rec}}$ is applied to the \emph{length-normalized} answer
log-probability $\overline{\log P}_\theta(a_{\mathrm{new}}\mid
q,c_{\mathrm{temp}})$ (mean per answer token), so $e^{\tau_{\mathrm{rec}}}$
is a per-token geometric-mean probability floor, \emph{not} a raw sequence
probability. A stricter (higher) floor increases confidence that the newer
fact is parametrically recoverable but reduces coverage; a looser floor
increases coverage but admits less plausible completions.
$\tau_{\mathrm{rec}}=-3$ is the code default fixed before any downstream
experiment (\texttt{recovery\_threshold} in the activation-cache stage) and
was not selected to maximize TAS Recovery.

\paragraph{Sensitivity to $\tau_{\mathrm{rec}}$.}
We regenerate the filter at $\tau_{\mathrm{rec}}\in\{-2,-2.5,-3,-3.5,-4\}$
from the cached phase-1 scores (\cref{tab:tau-rec-sweep}; every cell
reproduces exactly, and \cref{fig:tau-rec} plots it). The filtered-\PTC{}
rate (\PTC{} count over \emph{retained} records) shifts smoothly with the
threshold, but the conclusions are stable across the reviewer's
$\{-2.5,-3,-3.5\}$ window and beyond. First, the four-model ordering on
filtered \PTC{} (Qwen-1.5B $<$ Qwen-7B $<$ Mistral-7B $<$ Llama-8B) is
preserved at every $\tau_{\mathrm{rec}}\le-2.5$ (Spearman $\rho=1$ vs.\
$-3$); only at the most permissive $-2$ do Qwen-7B and Mistral swap, by
$<\!0.005$. We describe this as an \emph{empirical cross-model} ordering,
not controlled capacity scaling: only the within-Qwen pair
($1.5\text{B}<7\text{B}$) isolates parameter count, as the four models
differ in family, corpus, and cutoff. Second, filtered-\PTC{} sets are
highly stable (Jaccard $0.86$--$0.98$ vs.\ $-3$; the sets are nested by
construction). Third, holding the Bucket-5 train-only direction and
validation-selected $\alpha$ \emph{fixed} and varying only the evaluation
filter, held-out TAS Recovery on the filtered test-\PTC{} subset moves by
$\le0.09$ with overlapping CIs (\cref{tab:tau-rec-split}); this is a
fixed-intervention re-filtering, not a retrained sensitivity study (a
fully retrained sweep would rebuild the direction and re-select $\alpha$
per threshold, which we do not run). Relation ordering is also stable:
P35 is the strongest per-record relation for three of four models at every
threshold (P169 the runner-up for the $7$--$8$B models; on Qwen-1.5B P169
carries little signal). No conclusion in the paper hinges on the exact
value $-3$.

\begin{table}[htbp]
\centering\small\setlength{\tabcolsep}{5pt}
\begin{tabular}{llccccc}
\toprule
\hdrrow
\hdr{Model} & \hdr{$\tau_{\mathrm{rec}}$} & \hdr{train} & \hdr{val} & \hdr{calib} & \hdr{test} & \hdr{test kept \%} \\
\midrule
Qwen-2.5-1.5B & $-2.5$ & 56 & 13 & 8 & 10 & 17 \\
Qwen-2.5-1.5B & $-3.0$ & 64 & 18 & 9 & 10 & 26 \\
Qwen-2.5-1.5B & $-3.5$ & 73 & 20 & 10 & 15 & 40 \\
\midrule
Qwen-2.5-7B & $-2.5$ & 142 & 42 & 25 & 28 & 29 \\
Qwen-2.5-7B & $-3.0$ & 149 & 44 & 27 & 28 & 40 \\
Qwen-2.5-7B & $-3.5$ & 159 & 45 & 27 & 32 & 52 \\
\midrule
Mistral-7B-v0.3 & $-2.5$ & 277 & 73 & 53 & 59 & 52 \\
Mistral-7B-v0.3 & $-3.0$ & 289 & 76 & 55 & 61 & 63 \\
Mistral-7B-v0.3 & $-3.5$ & 298 & 78 & 56 & 66 & 75 \\
\midrule
Llama-3.1-8B & $-2.5$ & 358 & 83 & 57 & 81 & 55 \\
Llama-3.1-8B & $-3.0$ & 372 & 85 & 57 & 83 & 65 \\
Llama-3.1-8B & $-3.5$ & 378 & 87 & 60 & 85 & 73 \\
\bottomrule
\end{tabular}
\caption{Split-aware filtered-\PTC{} counts (is-\PTC{} $\wedge$
knowledge-present at $\tau_{\mathrm{rec}}$) by partition, and the test
kept fraction. Held-out test filtered-\PTC{} counts are small for
Qwen-2.5-1.5B ($\le15$), so its threshold-level estimates are descriptive;
the other three models retain $\ge28$ test positives throughout. Bucket-5
Recovery uses the (threshold-independent) is-\PTC{} test set as its
denominator; these filtered counts govern train-side direction
construction.}
\label{tab:tau-rec-split}
\end{table}

\begin{figure}[htbp]
    \centering
    \includegraphics[width=\columnwidth]{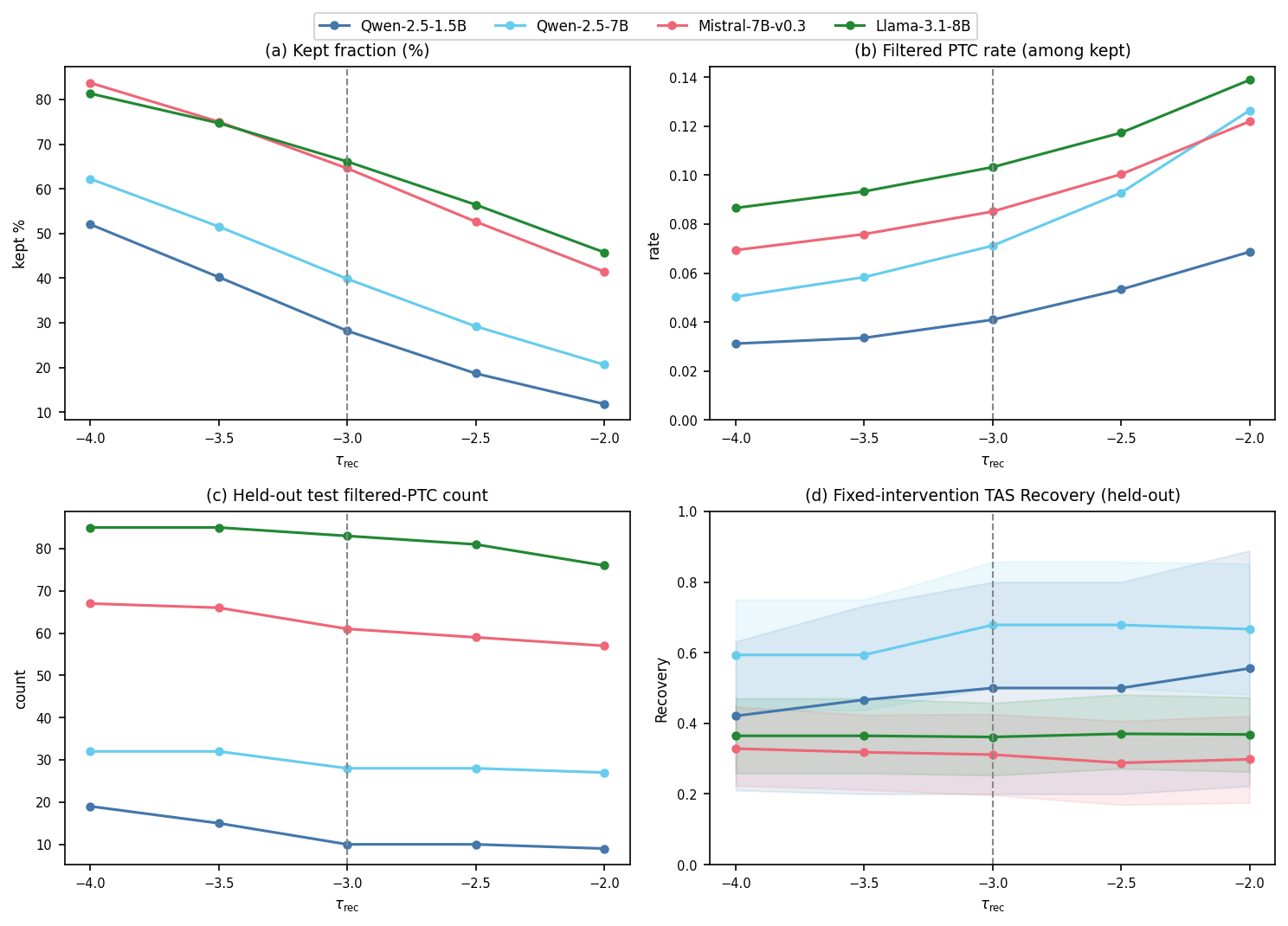}
    \caption{$\tau_{\mathrm{rec}}$ sensitivity (dashed line: primary
    $-3.0$): (a) kept fraction, (b) filtered-\PTC{} rate among kept, (c)
    held-out test filtered-\PTC{} count, (d) fixed-intervention held-out
    TAS Recovery with $95\%$ bootstrap bands. Axes are not truncated.}
    \label{fig:tau-rec}
\end{figure}

\begin{table}[htbp]
\centering
\rowcolors{2}{tablezebra}{white}
\small
\setlength{\tabcolsep}{6pt}
\renewcommand{\arraystretch}{1.25}
\begin{tabular}{lccccc}
\toprule
\hdrrow
\hdr{Model} & \hdr{$-2$} & \hdr{$-2.5$} & \hdr{$-3$} &
\hdr{$-3.5$} & \hdr{$-4$} \\
\midrule
\multicolumn{6}{l}{\emph{Filtered \PTC{} rate (\PTC{} among retained records)}} \\
Qwen-2.5-1.5B   & 0.069 & 0.053 & \cellcolor{tablehighlight}\textbf{0.041} & 0.034 & 0.031 \\
Qwen-2.5-7B     & 0.126 & 0.093 & \cellcolor{tablehighlight}\textbf{0.071} & 0.058 & 0.050 \\
Mistral-7B-v0.3 & 0.122 & 0.100 & \cellcolor{tablehighlight}\textbf{0.085} & 0.076 & 0.069 \\
Llama-3.1-8B    & 0.139 & 0.117 & \cellcolor{tablehighlight}\textbf{0.103} & 0.093 & 0.087 \\
\midrule
\multicolumn{6}{l}{\emph{Kept fraction (\% of $8{,}746$ benchmark passing the filter)}} \\
Qwen-2.5-1.5B   & 11.8 & 18.7 & \textbf{28.2} & 40.2 & 52.0 \\
Qwen-2.5-7B     & 20.6 & 29.2 & \textbf{39.8} & 51.5 & 62.2 \\
Mistral-7B-v0.3 & 41.4 & 52.6 & \textbf{64.6} & 75.0 & 83.7 \\
Llama-3.1-8B    & 45.8 & 56.4 & \textbf{66.1} & 74.7 & 81.3 \\
\bottomrule
\end{tabular}
\caption{Sensitivity of the knowledge-recovery filter to the recovery
threshold $\tau_{\mathrm{rec}}$ (column headers), on the full
$8{,}746$-record benchmark. \emph{Top panel:} filtered \PTC{} rate, the
\PTC{} rate among the records the filter retains. \emph{Bottom panel:}
kept fraction, the share of the benchmark passing
$\overline{\log P}_\theta(a_{\mathrm{new}}\mid q,c_{\mathrm{temp}})
\ge\tau_{\mathrm{rec}}$. The cream-highlighted column is the operating
point $\tau_{\mathrm{rec}}=-3$ used throughout the paper; its values
reproduce \cref{fig:filter-capacity}. The model ordering on filtered
\PTC{} is invariant for $\tau_{\mathrm{rec}}\le-2.5$, and filtered
\PTC{} and kept-fraction continue to increase together across the
evaluated models over the full range, so the filter's qualitative
conclusions do not depend on the exact threshold.}
\label{tab:tau-rec-sweep}
\end{table}

\section{Benchmark Construction and Verification}
\label{app:benchmark}

This appendix documents how the $8{,}746$-record benchmark was
constructed and verified. Every record traces back to specific Wikidata
statements, and every admission condition is a deterministic filter over
the resulting record. \emph{``Verified'' here means: every admitted record
passes deterministic Wikidata consistency checks}, not that any record was
hand-checked. A stratified $200$-record sample with annotation guidelines is
released for independent auditing; that annotation is not yet complete, so
we report no audit precision (\cref{app:bench-audit-results}).
Released artifacts: the SPARQL queries,
the per-record verification manifest and rejection logic, and the
manual-audit sample and annotation guide.

\subsection{Wikidata Sources and SPARQL Extraction}
\label{app:bench-sparql}
We mine five single-holder position-holder relations: head of state (P35),
CEO (P169), head coach (P286), chairperson (P488), and head of government
(P6). The exact, runnable extraction queries are released with the code and data: one per relation, plus a
parameterized template documenting endpoint, retrieval date, rank
and timestamp handling, and batching. Each query pulls every holder
statement with its value, start (P580) and end (P582) qualifiers, and rank,
against \texttt{query.wikidata.org}. Deprecated-rank statements are
excluded; preferred and normal ranks are both retained; missing P580/P582
are returned unbound and rejected downstream; English labels are requested
via the label service (items lacking an English label are rejected). The
raw candidate set is regenerable from these queries; the released repo
retains only the admitted records.

\subsection{Candidate Transition Construction}
\label{app:bench-candidates}
Within each subject, holder statements are ordered by start date and every
pair of \emph{consecutive} holders forms one candidate
$(a_{\mathrm{old}}, a_{\mathrm{new}}, t_{\mathrm{update}})$ transition, with
$t_{\mathrm{update}}=t_{\mathrm{new,start}}$; a subject with $k$ consecutive
holders yields up to $k{-}1$ candidates. The relation distribution is
intentionally unbalanced (\cref{fig:relation-dist}).

\begin{figure}[htbp]
    \centering
    \includegraphics[width=0.92\columnwidth]{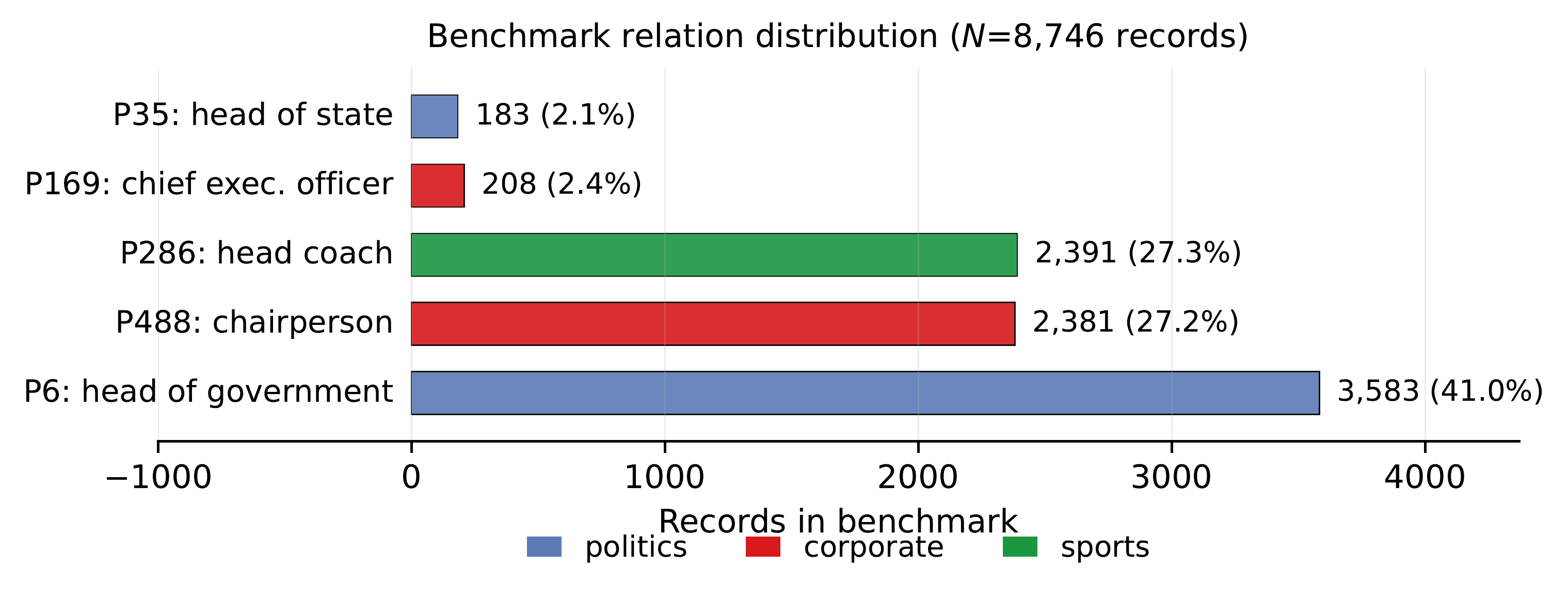}
    \caption{Benchmark relation distribution ($N{=}8{,}746$ records,
    five Wikidata relations, three domains). The two smallest relations,
    P35 (head of state, $n{=}183$) and P169 (CEO, $n{=}208$), carry the
    strongest per-record \PTC{} signal across all four models.}
    \label{fig:relation-dist}
\end{figure}

\paragraph{Temporal-era coverage.}
Bucketing by update year (\cref{fig:by-era}) gives an inverted-U with
$[2020,2021]$ as the peak for the three $7$--$8$B models; Llama-3.1-8B
uniquely retains a substantial $2022$--$2024$ rate, consistent with its
more recent training cutoff.

\begin{figure}[htbp]
    \centering
    \includegraphics[width=0.92\columnwidth]{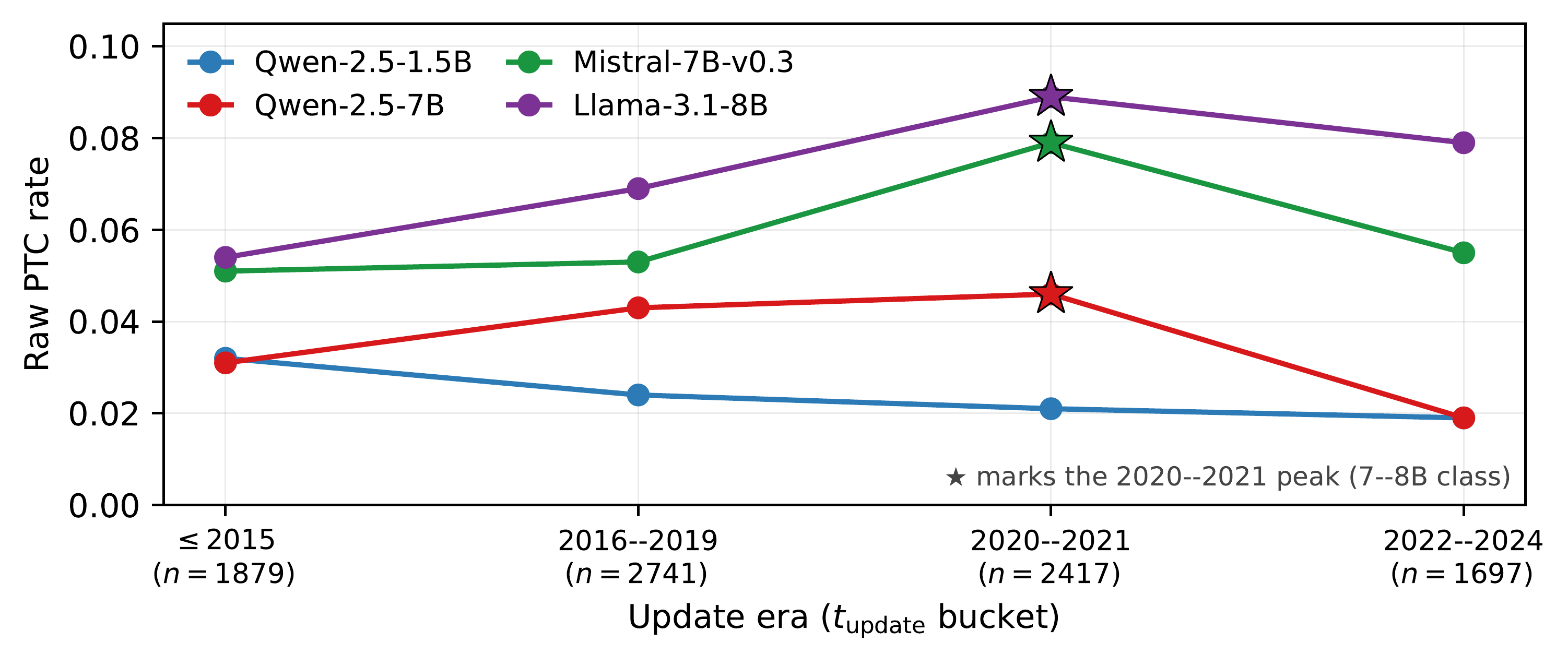}
    \caption{Per-era raw \PTC{} rate. The $2020$--$2021$ bucket is the
    peak for the three $7$--$8$B-class models (stars); Llama-3.1-8B
    uniquely retains a substantial $2022$--$2024$ rate, consistent with
    its more recent training cutoff.}
    \label{fig:by-era}
\end{figure}

\paragraph{Prompt construction.}
For each candidate $(q, a_{\mathrm{old}}, a_{\mathrm{new}},
t_{\mathrm{update}})$ tuple, we generate a standard prompt
(\emph{``Who is the $\langle$role$\rangle$ of
$\langle$subject$\rangle$?''}) and a temporal prompt
(\emph{``As of $\langle t\rangle$, who was the $\langle$role$\rangle$
of $\langle$subject$\rangle$?''}) where $t$ is chosen inside the new
fact's validity window. The temporal prompt uses past-tense
``who was'' for grammatical alignment with a specified past year.
\paragraph{Verification conditions.} A tuple is admitted only if all of
the following hold: (1) under standard prompting the model prefers
$a_{\mathrm{old}}$; (2) under temporal elicitation the model recovers
$a_{\mathrm{new}}$; (3) both answers are grammatically and factually
plausible completions; (4) the update is verified to be superseding
rather than additive; (5) knowledge-recoverability: the model satisfies
the recovery filter (\cref{app:filter}). Condition~(5) is model-specific
by construction; we report both \emph{raw} (all records) and
\emph{filtered} (model-recoverable) \PTC{} rates as a pair, never
collapsing them into a single number.

\subsection{Deterministic Admission Rules}
\label{app:bench-rules}
Independent of any model, a candidate transition is admitted to the
benchmark only if it passes all of these deterministic checks on the
Wikidata timestamps and labels:
\begin{compactitem}
\item \textbf{Abutting windows:} $|t_{\mathrm{old,end}}-t_{\mathrm{new,start}}|\le 60$
days (an approximate hand-off; overlaps up to $60$ days are permitted to
absorb Wikidata's date imprecision, larger overlaps/vacancies are rejected).
\item \textbf{Minimum old tenure:} $t_{\mathrm{old,end}}-t_{\mathrm{old,start}}\ge 180$
days, so the outdated fact had time to establish in pre-training text.
\item \textbf{Distinct holders:} $a_{\mathrm{old}}\neq a_{\mathrm{new}}$ (rejects
re-elections / re-appointments).
\item \textbf{Single-holder relation:} only P6/P35/P169/P286/P488, which admit
one holder at a time (rejects co-holder relations).
\item \textbf{Canonical English labels:} both answers have non-empty English
Wikidata labels (bare-QID values rejected).
\item \textbf{Year range:} $t_{\mathrm{update}}\in[2010,2024]$.
\item \textbf{Deduplication:} unique $(\text{relation},\text{subject},a_{\mathrm{old}},a_{\mathrm{new}},t_{\mathrm{update}})$.
\end{compactitem}
Recomputing all seven quantities from the stored timestamps for every
released record confirms that all $8{,}746$ admitted records satisfy every
rule (gap $\in[-60,60]$ d, tenure $\in[180,6209]$ d, $0$ same-holder, $0$
missing labels, $0$ duplicates; recorded in a verification manifest
regenerated by our verification pipeline).

\subsection{Rejection Breakdown}
\label{app:bench-rejection}
The pipeline assigns each rejected candidate a stable code:
MISSING\_ENGLISH\_LABEL, MISSING\_TIME\_BOUNDARY, OVERLAPPING\_HOLDERS,
LONG\_VACANCY, SHORT\_OLD\_TENURE, SAME\_HOLDER, COHOLDER\_CASE,
NON\_CONSECUTIVE\_TRANSITION, OUTSIDE\_YEAR\_RANGE, DUPLICATE\_RECORD,
ADDITIVE\_NOT\_SUPERSEDING, AMBIGUOUS\_WIKIDATA\_STATEMENT. The released
repository retains only the admitted set (over which every code has count
$0$ by construction); the full per-code candidate breakdown is regenerable
by running the released SPARQL (\cref{app:bench-sparql}) through the same
rule logic. \cref{tab:bench-examples} gives real Wikidata cases for the
principal rejection codes.

\subsection{Manual-Audit Protocol}
\label{app:bench-audit-protocol}
We draw a stratified random sample of $200$ admitted records (seed
$20260711$), with a per-relation floor of $20$ so the two rare relations
(P35, P169) can be inspected separately; the realized allocation is
P6:$61$, P35:$22$, P169:$22$, P286:$48$, P488:$47$. Every sampled
\texttt{record\_id} exists in the final benchmark (verified). The annotation
guide poses
yes/no/uncertain questions per record: does the subject match the intended
entity; was $a_{\mathrm{old}}$ valid before and $a_{\mathrm{new}}$ after the
transition; is this a true supersession (not additive/co-holder); are the
validity intervals consistent; are both English labels usable in the prompt;
is the generated question well formed; should the record remain. Annotators
record a label (VALID / INVALID / UNCERTAIN), the supporting Wikidata
statement ID or URL, an error category, and notes; UNCERTAIN is never
collapsed into VALID.

\subsection{Manual-Audit Results}
\label{app:bench-audit-results}
\emph{Pending.} The stratified $200$-record sample is released for
annotation; manual-audit precision and its $95\%$ confidence interval,
per-relation and per-era validity rates, and error-category counts will be
reported here once annotation is complete
(our released code computes the agreement statistic and its
Wilson/bootstrap CI from the annotated file). We deliberately report no
precision figure rather than a placeholder, and the main paper makes no
quantitative audit claim until the annotation is finished. This is a
single-audit design; no inter-annotator agreement is implied.

\subsection{Accepted and Rejected Examples}
\label{app:bench-examples-sec}
\begin{table}[htbp]
\centering\scriptsize\setlength{\tabcolsep}{3pt}
\begin{tabular}{@{}p{0.9cm}p{1.0cm}p{1.4cm}p{1.5cm}p{1.5cm}p{2.6cm}@{}}
\toprule
\hdrrow
\hdr{Status} & \hdr{Relation} & \hdr{Subject} & \hdr{Old $\to$ New} & \hdr{Time evidence} & \hdr{Reason} \\
\midrule
Accept & head of state & France (Q142) & Sarkozy (Q329) $\to$ Hollande (Q157) & old 2007-05-16$\to$2012-05-15; new from 2012-05-15 & abutting, tenure $1826$d, distinct \\
Accept & CEO & Deutsche Telekom (Q9396) & Obermann $\to$ H\"ottges & old $\to$2013-12-31; new from 2014-01-01 & abutting (gap $1$d), tenure $2604$d \\
Reject & head of gov. & 28th Canadian Ministry (Q220542) & Harper $\to$ Harper & three terms 2006/\allowbreak{}2008/\allowbreak{}2011 & \textbf{SAME\_HOLDER} (re-appointment) \\
Reject & CEO & Google Nest (Q2119882) & Fadell $\to$ N/A & end 2016-06-03, no start (P580) & \textbf{MISSING\_\allowbreak{}TIME\_\allowbreak{}BOUNDARY} \\
Reject & CEO & GNOME Foundation (Q1056660) & two ``current'' CEOs & since 2008 and since 2024, no end dates & \textbf{OVERLAPPING\_\allowbreak{}HOLDERS} / ambiguous \\
\bottomrule
\end{tabular}
\caption{Accepted (from the released benchmark) and rejected (real Wikidata
cases, retrieved live; QIDs shown) examples illustrating why each admission
rule is needed. The full list is released with the code and data.}
\label{tab:bench-examples}
\end{table}

\subsection{Remaining Benchmark Limitations}
\label{app:bench-limits}
(i) Wikidata timestamps are incomplete and sometimes imprecise (the
$60$-day abutment tolerance absorbs this but cannot fix missing dates); (ii)
entity labels can be ambiguous across languages (we require an English
label but do not disambiguate homonyms beyond the surface-form merge of
\cref{app:splits}); (iii) some real transitions are historically complex
(caretakers, interim holders) and are conservatively rejected rather than
modeled; (iv) the single-valued-relation assumption excludes multi-holder
offices by design; (v) the released audit sample covers $200$ records, not
the full benchmark, and its annotation is not yet complete, so no
manual-audit precision is available for any part of the benchmark.

\section{Extended Phase~1 Results}
\label{app:phase1-extended}

This appendix reports two extended Phase~1 views that the main paper's
headline tables compress: the \emph{full headline table} including the
elicitation gap in both log-probability and per-token-probability form
(\cref{tab:phase1-extended-full}), and the \emph{per-domain breakdown}
that aggregates relations into broad domains
(\cref{tab:phase1-by-domain-app}). The per-relation and per-era tables
are already in the main body
(\cref{fig:relation-breakdown} and the main-paper per-era figure); the per-domain view here completes
the three-way stratification.

\paragraph{Reading the elicitation gap.} The elicitation gap (EG)
measures how much probability mass the temporal cue moves toward
$a_{\mathrm{new}}$ on each record, in log-probability space and again
expressed as a per-token probability difference. A positive EG is the
\emph{direction-of-effect} signal that motivates \TAS{}: the temporal
cue moves the model in the right direction in every one of the four
models, even on the records where it fails to fully flip the
argmax. The log-probability gap is in a narrow $+0.09$ to $+0.18$ nat
band; the per-token probability gap ranges from $+0.015$ (Qwen-2.5-1.5B,
where the model's prior on $a_{\mathrm{new}}$ is low enough that the
cue's effect remains modest in absolute terms) to $+0.049$
(Llama-3.1-8B, where the model's better coverage means the same shift
translates to a larger absolute change).

\begin{table}[htbp]
\centering
\rowcolors{2}{tablezebra}{white}
\small
\setlength{\tabcolsep}{6pt}
\renewcommand{\arraystretch}{1.25}
\begin{tabular}{lccccc}
\toprule
\hdrrow
\hdr{Model} & \hdr{OPR} & \hdr{Recovery} &
\hdr{\PTC{}} & \hdr{EG (log-prob)} & \hdr{EG (prob)} \\
\midrule
Qwen-2.5-1.5B    & 0.530 & 0.475 & 0.024 & $+0.180$ & $+0.015$ \\
Qwen-2.5-7B      & 0.544 & 0.476 & 0.036 & $+0.090$ & $+0.018$ \\
Mistral-7B-v0.3  & 0.540 & 0.497 & 0.060 & $+0.148$ & $+0.040$ \\
Llama-3.1-8B     & 0.540 & 0.505 & \cellcolor{tablehighlight}\textbf{0.073} & $+0.179$ & \cellcolor{tablehighlight}$+0.049$ \\
\bottomrule
\end{tabular}
\caption{Extended Phase~1 screening on the full $8{,}746$-record
benchmark across all four models, with the elicitation gap (EG)
reported in both log-probability and per-token probability form;
the main body reports only the log-probability form. \emph{OPR}
is the fraction of records on which the model's argmax over
$\{a_{\mathrm{old}}, a_{\mathrm{new}}\}$ under the standard prompt
is $a_{\mathrm{old}}$; \emph{Recovery} is the fraction on which the
same argmax flips to $a_{\mathrm{new}}$ when the temporal cue
$c_{\mathrm{temp}}$ is added to the prompt; raw \emph{\PTC{}}
is the conjunction of OPR and Recovery on a per-instance basis
(the PTC definition in the main paper); \emph{EG~(log-prob)} is the mean
$\overline{\log P}(a_{\mathrm{new}}\mid q, c_{\mathrm{temp}})
- \overline{\log P}(a_{\mathrm{new}}\mid q)$ over the benchmark,
and \emph{EG~(prob)} is the same quantity exponentiated back to
per-token probability space. Cream-highlighted cells mark
Llama-3.1-8B's leading raw \PTC{} and per-token EG values, both
consistent with its more recent training cutoff. The narrow band
of OPR values ($0.530$--$0.544$) across all four models, despite
substantial differences in capacity, family, and training corpus, is
what we mean when we describe the outdated-preference rate as
``similar across evaluated models'' in the main body: even though larger models
score \emph{both} candidates more confidently in absolute terms,
the \emph{relative} preference for $a_{\mathrm{old}}$ over
$a_{\mathrm{new}}$ under the standard prompt is nearly the
same across the four models.}
\label{tab:phase1-extended-full}
\end{table}

\begin{figure}[htbp]
    \centering
    \includegraphics[width=0.92\textwidth]{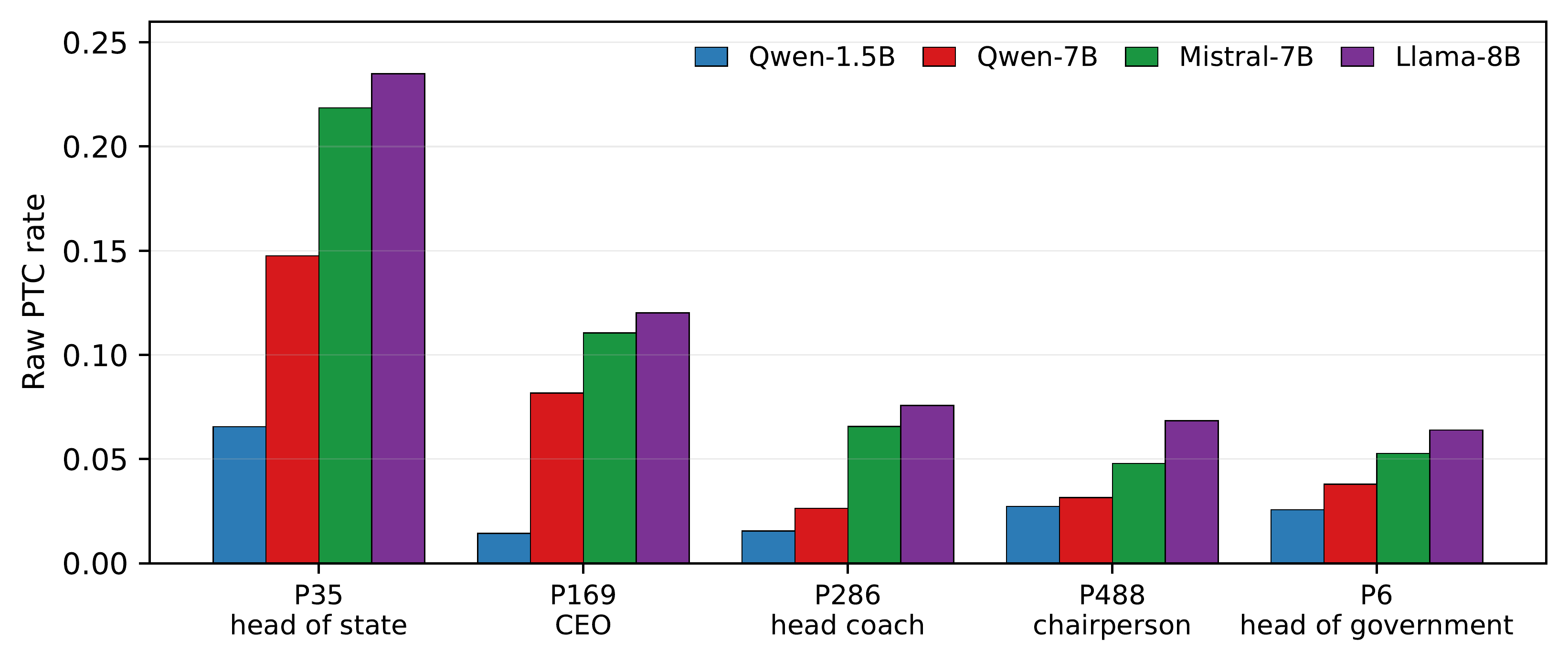}
    \caption{Per-relation raw \PTC{} rate across the four models.
    Each bar shows the fraction of records of a given Wikidata
    relation on which the model both prefers $a_{\mathrm{old}}$
    under standard prompting and recovers $a_{\mathrm{new}}$ under
    the temporal cue. The ordering
    P35~$>$~P169~$>$~\{P286,~P488,~P6\} is preserved across all four
    models, making the small-but-dense P35 (head of state, $n{=}183$)
    and P169 (CEO, $n{=}208$) relations the dominant per-record
    \PTC{} signal: together they hold $4.5\%$ of the benchmark but
    contribute the strongest conflict rate, while the three
    high-volume relations contribute most of the absolute record
    count yet only moderate per-record signal. This is what
    motivates the V2 (per-relation) construction of $\Delta$ in the
    main body: a single global average would dilute the strong P35
    and P169 signal with noisier high-volume contributions, while a
    per-relation $\Delta$ recovers it. The same cross-model ordering is
    visible within each relation: Llama-3.1-8B has the highest
    per-relation rate on three of five relations.}
    \label{fig:relation-breakdown}
\end{figure}

\begin{insightgreen}{Concentration of per-record \PTC{} signal in small relations}
Head of state (P35, $n{=}183$) and CEO (P169, $n{=}208$) together
hold $4.5\%$ of the benchmark but exhibit the largest per-record
\PTC{} rates in every model (P35 reaches $0.235$ on Llama-3.1-8B,
roughly $3\times$ the rate of any high-volume relation). The
unbalanced relation distribution is therefore a feature, not a
confound.
\end{insightgreen}

\paragraph{Reading the domain view.} \cref{tab:phase1-by-domain-app}
aggregates the five Wikidata relations into three broad domains:
\emph{politics} (P35 head of state + P6 head of government),
\emph{corporate\_leadership} (P169 CEO + P488 chairperson), and
\emph{sports} (P286 head coach). The domain ordering matches the
per-relation ordering reported in the main body
(\cref{fig:relation-breakdown}): politics carries the strongest per-record
\PTC{} signal in every model, corporate leadership is intermediate,
and sports is lowest. The domain ordering is preserved across all four
models, suggesting the effect is driven primarily by the underlying
relation distribution (P35 alone accounts for most of the political
\PTC{} signal at the $7$--$8$B scale) rather than by any
domain-specific representational structure.

\begin{table}[htbp]
\centering
\rowcolors{2}{tablezebra}{white}
\small
\setlength{\tabcolsep}{6pt}
\renewcommand{\arraystretch}{1.25}
\begin{tabular}{lrcccc}
\toprule
\hdrrow
\hdr{Domain} & \hdr{$n$} & \hdr{Qwen-1.5B} & \hdr{Qwen-7B} &
\hdr{Mistral-7B} & \hdr{Llama-3.1-8B} \\
\midrule
politics                  & 3{,}766 & \cellcolor{tablehighlight}0.028 & \cellcolor{tablehighlight}0.043 & 0.061 & 0.072 \\
corporate\_leadership     & 2{,}589 & 0.026 & 0.036 & 0.053 & \cellcolor{tablehighlight}0.073 \\
sports                    & 2{,}391 & 0.015 & 0.026 & \cellcolor{tablehighlight}0.066 & \cellcolor{tablehighlight}0.076 \\
\bottomrule
\end{tabular}
\caption{Per-domain raw \PTC{} rate on the full $8{,}746$-record
benchmark, aggregating the five Wikidata relations into three
domains: politics (P35 head-of-state $+$ P6 head-of-government),
corporate leadership (P169 CEO $+$ P488 chairperson), and sports
(P286 head-coach). Politics is the strongest \PTC{} domain across
all four models, corporate leadership is intermediate, and sports
is lowest. Cream highlights mark the per-model leading domain.
The ordering politics~$>$~corporate~$>$~sports is preserved across
all four models. Combined with the per-relation analysis
(\cref{fig:relation-breakdown}), this indicates that the
apparent domain effect is driven primarily by the underlying
relation composition (P35 and P6 are both political and together
contribute the dominant per-record \PTC{} signal) rather than by
any domain-specific representational structure in the residual
stream.}
\label{tab:phase1-by-domain-app}
\end{table}

\section{Layer Localization Details}
\label{app:locator-details}

This appendix expands on the Stage~2 locator results summarized in the
main paper (the main-paper locator table). The headline table reports only peak
AFR and plateau width; here we additionally report the full per-layer
AFR profile (\cref{fig:afr-profile}) and the per-instance mean
log-probability shifts the patch induces on $a_{\mathrm{new}}$ and
$a_{\mathrm{old}}$ at $\ell^*$, which together characterize
\emph{how} the patched activation flips the answer.

\begin{figure}[htbp]
    \centering
    \includegraphics[width=0.86\textwidth]{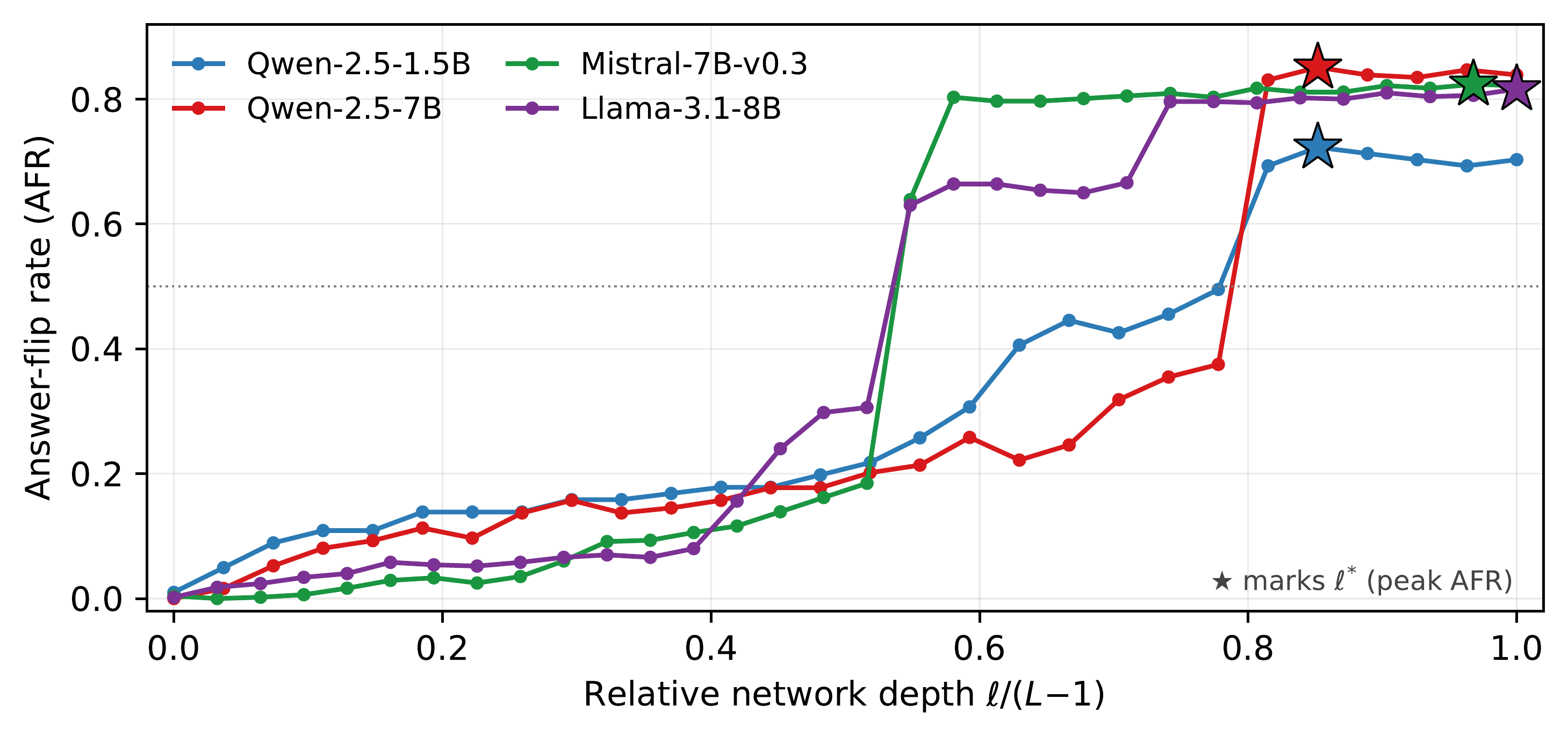}
    \caption{Per-layer answer-flip rate (AFR) for all four models,
    with the $x$-axis normalised by network depth ($\ell / (L - 1)$)
    so the four architectures (28, 28, 32, 32 layers) share a common
    frame. AFR$(\ell)$ is the fraction of verified \PTC{} instances on
    which a single-token activation patch at layer $\ell$ from the
    temporal-prompt forward pass into the standard-prompt forward
    pass flips the answer argmax. Stars mark each model's
    conflict-critical layer $\ell^{*} = \arg\max_{\ell}
    \mathrm{AFR}(\ell)$. Three observations: (i) all four curves are
    near-zero below relative depth $\approx\!0.6$, indicating that
    conflict-relevant information accumulates only in the second half
    of the network; (ii) each curve shows a sharp transition followed
    by a relatively flat plateau, supporting the Local Separability
    precondition of \cref{thm:tas-full}; (iii) the two Qwens
    transition at the same relative depth, while Mistral and
    Llama-3.1-8B transition slightly later, consistent with the
    family-specific $\ell^{*}$ values in the main-paper locator table.}
    \label{fig:afr-profile}
\end{figure}

\begin{insightgreen}{Sharp layerwise transition in answer-flip rate}
On every model the AFR profile is nearly flat below
$\ell/(L{-}1)\!\approx\!0.6$ and then climbs to its peak inside a
narrow late-network band. This is the \emph{causal} signature
\cref{thm:tas-full} relies on: a single, localised separating direction
that single-token activation patching can exploit.
\end{insightgreen}

\paragraph{Mean $\Delta$ log-probability shifts.} At $\ell^*$,
$\Delta \overline{\log P}(a_{\mathrm{new}})$ is the average gain in
length-normalized log-probability the patched activation produces for
the newer answer, and $\Delta \overline{\log P}(a_{\mathrm{old}})$ is
the corresponding shift for the older answer (typically negative when
the patch successfully redirects probability mass). The pattern across
models is informative. On Qwen-2.5-1.5B, the patch yields a large
positive shift for $a_{\mathrm{new}}$ ($+0.880$ nats), but also a
positive shift for $a_{\mathrm{old}}$ ($+0.263$ nats), suggesting that
the patched representation increases probability on \emph{both}
candidates relative to baseline (more strongly on the newer one).
On all three $7$--$8$B-class models the $a_{\mathrm{old}}$ shift is
negative ($-0.148$ to $-0.302$), indicating a cleaner zero-sum
redirection of probability mass between the two candidates at $\ell^*$.

\paragraph{Plateau geometry.} The plateau is the contiguous block of
layers with AFR within $0.05$ of peak that contains $\ell^*$. The two
Qwens share both $\ell^*=23$ and a $6$-layer plateau, consistent with
shared architecture and tokenizer. Mistral's plateau is the widest in
our study ($14$ layers, $\ell \in [18, 31]$), suggesting that the
conflict-resolution computation in Mistral is distributed across
roughly the upper half of the network. Llama-3.1-8B's plateau is a
$9$-layer end-network band ($\ell \in [23, 31]$), with the peak at the
final decoder layer.

A sublayer (attention vs.\ MLP) decomposition is a natural follow-up
that we do not include here.

\section{Steering Vector Variants (V1 / V2 / V3)}
\label{app:variants}

The \TAS{} framework defines three steering-vector averaging granularities for estimating the conflict direction $\Delta$ at layer $\ell^*$. In the first variant, V1, a single global direction is computed by averaging $\Delta$ over all verified \PTC{} instances. This is the simplest construction, since the same direction is applied to every query regardless of relation or domain.

The second variant, V2, computes a separate direction for each Wikidata relation in the benchmark. Each relation-specific $\Delta$ is averaged only over the verified \PTC{} instances associated with that relation. This design is intended to preserve relation-conditioned activation geometry, which may otherwise be diluted when all instances are collapsed into a single global average.

The third variant, V3, uses an intermediate level of granularity by computing one direction per broad domain, such as politics, corporate leadership, and sports. Thus, V3 lies between the global construction of V1 and the fine-grained relation-specific construction of V2.

We ran the full V1/V2/V3 comparison under the corrected held-out protocol
(train-only directions, validation-selected $\alpha$, one evaluation on the
subject-disjoint test set; \cref{tab:variants-app}, plotted in
\cref{fig:variant-ablation}). Relations or domains
with fewer than $10$ train-\PTC{} records fall back to the global direction
(same rule for validation and test; never tuned on test). The pilot's
``V1${<}$V3${<}$V2'' ordering \emph{does not survive}: on Mistral-7B-v0.3
(the pilot's own model) the held-out ordering is V1${>}$V2${>}$V3, and V1
gives the highest Recovery on two of four models. No V2-vs-V1 or V2-vs-V3
Recovery difference is significant after Holm correction (paired McNemar,
per-model discordant counts released with the code); the
three variants are statistically comparable. On PA, V2 is not better than
V1 and is significantly \emph{worse} on Qwen-2.5-1.5B (discordant $2/26$,
$p_{\mathrm{Holm}}{=}1.2\!\times\!10^{-5}$). The smallest relations (P35,
P169) have too few held-out test instances ($n{=}1$--$6$) to compare
per-relation. We therefore retain V2 as a \emph{relation-conditioned
mechanistic} construction rather than a performance-superior one; the
simpler global V1 is an equally good, better-preserving alternative. The
earlier $500$-record single-model oracle pilot is superseded by this table.

\begin{table}[htbp]
\centering\small\setlength{\tabcolsep}{4pt}
\begin{tabular}{lccccccc}
\toprule
\hdrrow
& \multicolumn{3}{c}{\hdr{Recovery}} & \multicolumn{3}{c}{\hdr{PA}} & \\
\hdrrow
\hdr{Model} & \hdr{V1} & \hdr{V2} & \hdr{V3} & \hdr{V1} & \hdr{V2} & \hdr{V3} & \hdr{$\alpha$ 1/2/3} \\
\midrule
Qwen-2.5-1.5B & 0.375 & 0.469 & 0.438 & \textbf{0.940} & 0.820 & 0.815 & 1/2/2 \\
Qwen-2.5-7B & \textbf{0.615} & 0.538 & 0.564 & 0.830 & 0.800 & 0.805 & 2/2/2 \\
Mistral-7B-v0.3 & \textbf{0.394} & 0.324 & 0.310 & 0.830 & 0.850 & 0.850 & 8/4/4 \\
Llama-3.1-8B & 0.345 & 0.356 & \textbf{0.368} & 0.790 & 0.785 & 0.790 & 4/6/6 \\
\bottomrule
\end{tabular}
\caption{\textbf{Held-out V1/V2/V3 ablation} on the subject-disjoint test
set: train-only directions, validation-selected $\alpha$ (shown as
1/2/3), evaluated once on identical test \PTC{} records and the identical
$200$ test controls. Best Recovery per model in bold. V1/V2/V3 are
statistically comparable (no significant Recovery difference after Holm);
V2 is not consistently best, and V1 leads on Qwen-2.5-7B and
Mistral-7B-v0.3. This held-out table supersedes the earlier
$500$-record, single-model, oracle-$\alpha$ pilot.}
\label{tab:variants-app}
\end{table}

\begin{figure}[htbp]
    \centering
    \includegraphics[width=\columnwidth]{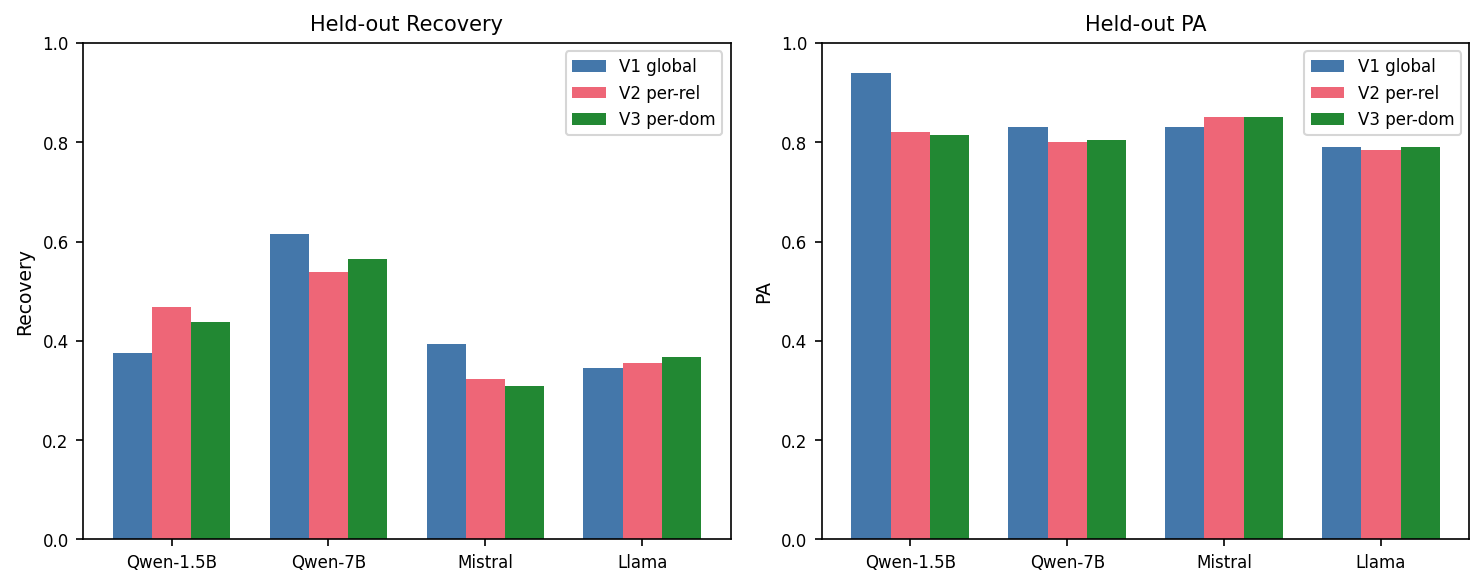}
    \caption{Held-out Recovery (left) and PA (right) for V1 (global),
    V2 (per-relation), V3 (per-domain). The variants are statistically
    comparable; V2 is not consistently best.}
    \label{fig:variant-ablation}
\end{figure}

\section{Extended \TAS{} Diagnostics}
\label{app:tas-extended}

This appendix expands the headline end-to-end \TAS{} results with three
complementary views drawn from the same per-model JSONs as
the main-paper TAS and threshold-sweep tables: detector PR curves
(\cref{fig:detector-pr}), the oracle steering-scale sweep
(\cref{fig:oracle-alpha}), and the $\tau$-Pareto frontier
(the main-paper threshold-sweep table). Together they trace each model's behaviour
across the three controls (detector quality, $\alpha$, $\tau$) that govern
\TAS{}'s operating point.

\paragraph{Why Qwen-2.5-7B shows the largest effect.}
Among the four models, Qwen-2.5-7B is where a single-direction edit moves
the most verified-\PTC{} cases. Three properties coincide in it:
strong-but-not-saturated coverage (filtered \PTC{} $0.071$), a tight
$6$-layer causal plateau, and a small steering norm, so a modest
$\alpha^{*}{=}2$ carries the state across the argmax boundary without
overshooting into the preservation loss that larger $\alpha$ incurs on
Mistral and Llama. Relatedly, the gap between always-on and detector-gated
Recovery tracks detector quality rather than the edit: it is smallest on
Mistral-7B-v0.3 ($0.27$), whose detector separates best, and largest on the
two Qwens ($0.44$ and $0.46$), whose detectors are weakest. This is a statement about the gate, not about the
direction; the always-on figures are the mechanistic quantity, and the
gated figures inherit the detector's failure.

\begin{figure}[htbp]
    \centering
    \includegraphics[width=0.92\textwidth]{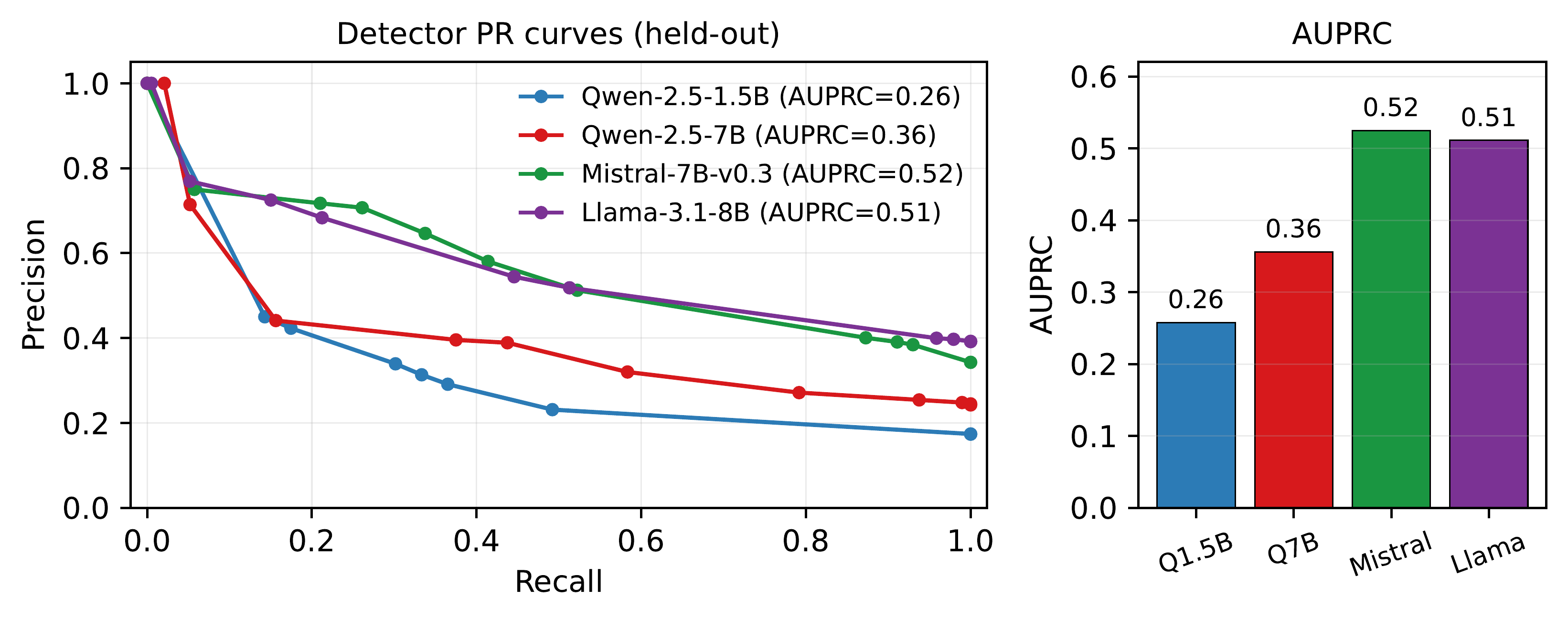}
    \caption{\textbf{Legacy instance-level split; diagnostic only;
    superseded by the subject-disjoint evaluation.} Stage~1 detector
    precision-recall curves computed on the legacy instance-level split,
    which \emph{overstates} detector quality. The corrected subject-disjoint
    curves are \cref{fig:detector-corrected-pr} and the corrected metrics
    \cref{tab:detector-corrected}: held-out test AUPRC falls to
    $0.028$--$0.097$ and Qwen-2.5-1.5B AUROC is below chance ($0.47$). Even
    on this optimistic legacy split the relative ordering already shows
    Mistral-7B-v0.3 and Llama-3.1-8B training stronger detectors than the
    two Qwens; the corrected metrics confirm that conflict detection under
    honest subject-disjoint evaluation remains unsolved, which is why the
    detector-gated pipeline is reported as a negative result rather than a
    deployable component. This figure is retained only for historical
    comparison.}
    \label{fig:detector-pr}
\end{figure}

\begin{figure}[htbp]
    \centering
    \includegraphics[width=0.95\textwidth]{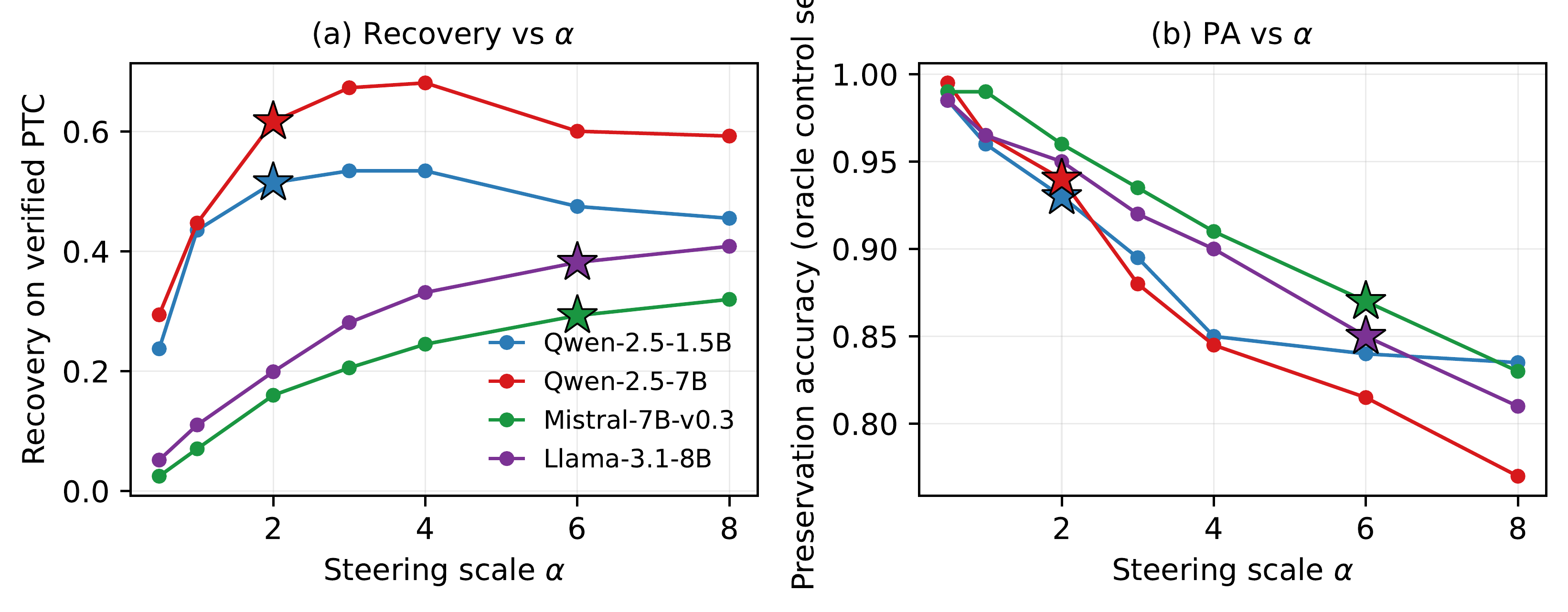}
    \caption{Oracle $\alpha$-sweep at $\ell^{*}$ with the V2
    per-relation $\Delta$ and no detector gate (steering is applied to
    every instance). \textbf{(a)} Recovery on the verified \PTC{}
    subset rises monotonically with the steering scale $\alpha$ on all
    four models and saturates by $\alpha \approx 4$ on the two Qwens;
    Mistral and Llama-3.1-8B keep climbing more slowly. Stars mark
    each model's $\alpha^{*} = \arg\max_{\alpha}\,[\mathrm{Recovery}
    - \lambda(1 - \mathrm{PA})]$ with $\lambda = 1.0$. \textbf{(b)}
    Preservation accuracy on the matched non-conflict control set
    declines smoothly with $\alpha$, as expected because a stronger
    steering perturbation distorts more of the residual-stream
    trajectory. The $\alpha^{*}$ values lie at the knee of each
    model's Recovery-vs-$\alpha$ curve where the marginal Recovery
    gain begins to be outweighed by the marginal PA loss; this is
    the empirical witness for the existence of a finite
    $\alpha^{*}>0$ required by \cref{thm:tas-full}.}
    \label{fig:oracle-alpha}
\end{figure}

\paragraph{False-alarm rate at headline operating points.}
The in-pipeline false-alarm rate (steered fraction of non-\PTC{}
records) at $\tau\in\{0.15,0.20,0.30\}$ is, per model:
Qwen-2.5-1.5B $0.189$\,/\,$0.143$\,/\,$0.055$;
Qwen-2.5-7B   $0.971$\,/\,$0.376$\,/\,$0.183$;
Mistral-7B-v0.3 $0.772$\,/\,$0.771$\,/\,$0.637$;
Llama-3.1-8B  $0.996$\,/\,$0.995$\,/\,$0.955$.
The two large models' high false-alarm rates at all swept $\tau$
reflect the calibrator-cliff issue discussed in
the main paper.

\begin{insightgreen}{Distinct roles of the three operating parameters}
The PR curves (\cref{fig:detector-pr}) set the ceiling on what the
detector can sort. The $\alpha$-sweep (\cref{fig:oracle-alpha})
shows where the steering vector itself starts overshooting (Recovery
flattens). The $\tau$-frontier (the main-paper threshold-sweep table) is where
operators trade Recovery for PA at deployment. Qwen-2.5-7B is the
strongest model on all three because its detector AUPRC ($0.36$) is
adequate, $\alpha^{*}{=}2$ is on the rising part of its Recovery
curve, and its $\tau$-frontier sits clearly upper-right of every
other model.
\end{insightgreen}

\paragraph{Per-relation Recovery.}
A per-relation breakdown of end-to-end \TAS{} Recovery is a useful
future diagnostic for testing whether the P35-dominant screening
pattern (\cref{fig:relation-breakdown}) also
appears after steering, and would inform the relation-conditional
$\alpha^*$ extension noted in the main paper.

\section{Free-Generation Evaluation}
\label{app:freegen}
As a complementary robustness and limitation analysis, we evaluate every
intervention under \emph{unrestricted} greedy decoding, in addition to the
candidate-probability scoring used throughout the paper. This answers whether the
mechanistic (candidate-level) conclusion also holds when the model answers naturally.
The headline result is behaviourally distinct and negative: on these base
(non-instruction-tuned) models, free generation is dominated by definitional /
non-entity continuations, and \emph{no intervention consistently produces the newer
entity}.

\subsection{Free-Generation Protocol}
\label{app:fg-protocol}
For every held-out test verified-\PTC{} record we generate a natural answer with
identical decoding for every method: greedy (\texttt{do\_sample=False},
\texttt{num\_beams=1}), $\le 25$ new tokens, truncated to the first non-empty line.
Methods: \emph{standard}, \emph{date\_prefix} (``In $\langle$year$\rangle$:'' prefix),
\emph{instruction} (``Answer with the current, most up-to-date fact.''),
\emph{TAS} (always-on), and \emph{ITI}. Steering reuses the \emph{same} raw $\Delta$
vectors, layer $\ell^{*}$, and validation-selected $\alpha$ as the probability
evaluation, applied at the conflict-critical query position during the prompt forward
pass (applying it at every decode step over-applies the scoring-selected $\alpha$ and
destroys fluency). Preservation is measured on up to $200$ non-conflict control
records. Test-\PTC{} sizes: $32/39/71/87$ (Qwen-1.5B/7B, Mistral, Llama).

\subsection{Alias and Entity Resolution}
\label{app:fg-alias}
A generation is matched to the newer/outdated entity by deterministic alias-aware
matching: Unicode NFKC, lowercasing, diacritic and punctuation stripping, whitespace
normalization, then whole-word matching against the entity's Wikidata English label
and aliases \emph{and} the benchmark's stored gold label (the latter recovers
entities whose current Wikidata label is empty or has drifted). Redirects are
resolved by the Wikidata API. Multi-entity generations match if any gold surface form
appears; unresolved outputs are logged for manual review.

\subsection{Main Free-Generation Results}
\label{app:fg-results}
\cref{tab:fg-categories} gives, per model and method, the rate of each output
category. Definitional/non-entity continuations dominate for every method and model;
the new-entity rate never exceeds $0.21$ and \emph{the best method differs by model}
(standard for Qwen-7B and Llama; date-prefix for Mistral; all near zero for
Qwen-1.5B). TAS never exceeds standard on any model, and on Llama steering collapses
to empty output ($72\%$).

\begin{table}[htbp]
\centering\scriptsize\setlength{\tabcolsep}{3pt}
\begin{tabular}{llccccc}
\toprule
\hdrrow
\hdr{Model} & \hdr{Method} & \hdr{new} & \hdr{old} & \hdr{other} & \hdr{defin.} & \hdr{empty} \\
\midrule
Qwen-1.5B & standard & 0.03 & 0.00 & 0.50 & 0.47 & 0.00 \\
          & date\_prefix & 0.06 & 0.09 & 0.41 & 0.44 & 0.00 \\
          & instruction & 0.06 & 0.06 & 0.47 & 0.41 & 0.00 \\
          & TAS & 0.00 & 0.06 & 0.44 & 0.50 & 0.00 \\
          & ITI & 0.03 & 0.00 & 0.53 & 0.44 & 0.00 \\
\midrule
Qwen-7B & standard & 0.21 & 0.10 & 0.33 & 0.36 & 0.00 \\
        & date\_prefix & 0.15 & 0.26 & 0.44 & 0.15 & 0.00 \\
        & instruction & 0.21 & 0.10 & 0.51 & 0.18 & 0.00 \\
        & TAS & 0.18 & 0.10 & 0.38 & 0.31 & 0.03 \\
        & ITI & 0.15 & 0.08 & 0.31 & 0.46 & 0.00 \\
\midrule
Mistral & standard & 0.00 & 0.07 & 0.56 & 0.37 & 0.00 \\
        & date\_prefix & 0.15 & 0.11 & 0.30 & 0.44 & 0.00 \\
        & instruction & 0.04 & 0.01 & 0.01 & 0.93 & 0.00 \\
        & TAS & 0.00 & 0.07 & 0.56 & 0.37 & 0.00 \\
        & ITI & 0.01 & 0.07 & 0.55 & 0.37 & 0.00 \\
\midrule
Llama & standard & 0.05 & 0.05 & 0.25 & 0.66 & 0.00 \\
      & date\_prefix & 0.01 & 0.02 & 0.17 & 0.79 & 0.00 \\
      & instruction & 0.00 & 0.02 & 0.60 & 0.36 & 0.02 \\
      & TAS & 0.00 & 0.00 & 0.08 & 0.20 & 0.72 \\
      & ITI & 0.01 & 0.05 & 0.77 & 0.17 & 0.00 \\
\bottomrule
\end{tabular}
\caption{Free-generation output categories (held-out test \PTC{}). \emph{new/old}:
alias-matches the newer/outdated entity; \emph{other}: names a different entity;
\emph{defin.}: role/relational continuation with no entity answer; \emph{empty}:
no output. Rows sum to $1$. Definitional non-answers dominate; no method consistently
produces the newer entity.}
\label{tab:fg-categories}
\end{table}

\subsection{Candidate-Scoring vs.\ Generation Disagreement}
\label{app:fg-disagreement}
Candidate scoring and free generation disagree substantially
(\cref{tab:fg-disagreement}), confirming they measure \emph{different} things.
Candidate scoring reports date-prefix Recovery of $0.61$--$0.81$, but the
free-generation new-entity rate is only $0.01$--$0.15$. Conversely, standard
prompting has candidate Recovery $0$ by \PTC{} construction (it must prefer the
outdated answer), yet free generation already produces the newer entity up to $0.21$
of the time, because greedy decoding is not restricted to the two candidates. For
TAS/ITI the candidate Recovery ($0.29$--$0.54$) likewise does not translate into a
comparable free-generation new-entity rate.

\begin{table}[htbp]
\centering\scriptsize\setlength{\tabcolsep}{4pt}
\begin{tabular}{llcc}
\toprule
\hdrrow
\hdr{Model} & \hdr{Method} & \hdr{cand.\ Recovery} & \hdr{free-gen new-rate} \\
\midrule
Qwen-1.5B & date\_prefix & 0.61 & 0.06 \\ & TAS & 0.47 & 0.00 \\ & standard & 0.00 & 0.03 \\
\midrule
Qwen-7B & date\_prefix & 0.78 & 0.15 \\ & TAS & 0.54 & 0.18 \\ & standard & 0.00 & 0.21 \\
\midrule
Mistral & date\_prefix & 0.79 & 0.15 \\ & TAS & 0.32 & 0.00 \\ & standard & 0.00 & 0.00 \\
\midrule
Llama & date\_prefix & 0.81 & 0.01 \\ & TAS & 0.36 & 0.00 \\ & standard & 0.00 & 0.05 \\
\bottomrule
\end{tabular}
\caption{Candidate-scoring held-out Recovery vs.\ free-generation new-entity rate.
The two measurements diverge sharply: high candidate Recovery does not yield a
correspondingly high generated-answer rate, and standard generation produces the
newer entity despite candidate scoring preferring the outdated one.}
\label{tab:fg-disagreement}
\end{table}

\subsection{Manual Validation}
\label{app:fg-manual}
We manually inspected a stratified sample of $128$ generations (across all models,
methods, and output categories; $\sim\!90$ examined with full untruncated text
against the gold labels). New/old detection agreed with the automatic label in
$\approx\!100\%$ of inspected cases (alias handling verified on name reordering
``Abe Shinzo'', diacritics, and titles); full $5$-way agreement is $\approx\!90\%$,
with the residual disagreements confined to the \emph{other\_entity vs.\
definitional} boundary (organisation/office descriptions with no person answer),
which does not affect the new-answer rate. Two parser/resolver corrections were made
and re-run: a leading-newline parsing bug (Mistral generations were all truncated to
empty; fixed to take the first non-empty line) and empty-current-Wikidata-label
false negatives (fixed by also matching the benchmark's stored gold label). The
genuinely unresolvable outputs are the definitional/empty non-answers, which are
correctly identified as such.

\subsection{Error Analysis: Representative Outputs}
\label{app:fg-errors}
\begin{compactitem}
\item \emph{Correct newer entity:} ``The head coach of FC Shakhtar Donetsk is Roberto De Zerbi.''
\item \emph{Outdated entity:} ``Luciano Spalletti is the head coach of Inter Milan.'' (newer: Antonio Conte)
\item \emph{Definition instead of entity:} ``The head of state of Sri Lanka is the President.''
\item \emph{Unrelated entity:} ``The governor of Rio de Janeiro is Wilson Witzel.'' (newer: Eduardo Paes)
\item \emph{Multiple entities:} ``A. Jeremy Corbyn B. Ed Miliband C. Nick Clegg D. Nick Clegg''
\item \emph{Malformed continuation:} ``A. A. A. A. A. A. A. A.'' (degenerate, under steering)
\end{compactitem}

\subsection{Interpretation}
\label{app:fg-interp}
We state explicitly: (i) this experiment does \emph{not} validate \TAS{} as a
practical generation method: TAS never exceeds standard generation and degrades
fluency on the smallest model; (ii) it reveals a limitation of evaluating base
(non-instruction-tuned) models through unrestricted greedy generation, which is
dominated by definitional non-answers; (iii) it does \emph{not} invalidate the
candidate-level mechanistic result, because candidate scoring remains the appropriate
instrument for measuring a directional parametric effect; and (iv) candidate scoring
and free generation answer related but different questions. We therefore retain
candidate-level scoring as the primary mechanistic evaluation and report free
generation as complementary robustness and limitation evidence.

\section{Cross-Family Generalization}
\label{app:appendix_cross_model}

This appendix consolidates the cross-family evidence presented
piecemeal across the main-paper screening, locator, and overlap analyses
and frames it as a single \emph{generalization claim}: the \PTC{}
phenomenon and the \TAS{} intervention are not artefacts of any single
training pipeline.

\paragraph{Quantitative cross-family agreement at $7$--$8$B.}
At the $7$--$8$B parameter class, three different families (Qwen-2.5-7B,
Mistral-7B-v0.3, Llama-3.1-8B), trained on different corpora with
different tokenizers, agree on every headline metric within a few
percentage points: \OPR{} all $\in [0.540, 0.544]$, \EG{} all $\in
[+0.090, +0.179]$ nats, filtered \PTC{} all $\in [0.071, 0.103]$, peak
AFR all $\in [0.816, 0.851]$. The \PTC{} pattern is therefore not a
quirk of any single training pipeline.

Cross-family steering-vector transfer (applying a Mistral-derived
$\Delta$ to Llama-3.1-8B at matched relative depth) is left to future
work and is motivated by the relatively large pairwise
Mistral~$\cap$~Llama overlap in the main-paper overlap figure.

\begin{insightteal}{\PTC{} is not a single-pipeline artefact}
At the $7$--$8$B scale, three families with different corpora and
tokenizers agree on every headline metric within a few percentage
points (\OPR{} $\in[0.540,0.544]$, filtered \PTC{}
$\in[0.071,0.103]$, peak AFR $\in[0.816,0.851]$). The phenomenon is
real, measurable, and capacity-sensitive across pipelines; it is not
a quirk of any single corpus or tokenizer family.
\end{insightteal}

\section{Implementation and Reproducibility}
\label{app:complexity-analysis}

This appendix documents the hardware, compute footprint, and
reproducibility provisions of our implementation. Everything is
deterministic given the model checkpoints, the benchmark JSONL, and
the per-model $(\ell^*, \alpha^*, \tau^*)$ values reported throughout;
no random initialization beyond the detector's stratified train/test
split (which we seed) affects the headline numbers.

\paragraph{Hardware.} All Phase~1 screening and \TAS{} experiments were
run on a workstation with 2$\times$ NVIDIA TITAN RTX (24~GB each).
Phase~1 screening on the full $8{,}746$-record benchmark takes
approximately $25$--$30$ minutes per model in fp16 single-GPU~mode.
Full \TAS{} pipeline (cache activations $\to$ locate $\to$ steer $\to$
train detector $\to$ end-to-end TAS at three $\tau$ values) takes
approximately $90$--$120$ minutes per model. The four-model ITI
baseline sweep (the main-paper ITI comparison) runs in $4$--$15$
minutes per model on a single TITAN RTX, dominated by the $1000$
forward passes ($500$ prefers-old, $500$ prefers-new) used to construct
$\Delta_{\mathrm{ITI}}$ plus the $\alpha$-sweep over the $200$-instance
control set.

\paragraph{Default settings.} The released pipeline ships with the
following defaults, which were used unchanged for every reported
number: $\alpha$ grid $\{0.5, 1, 2, 3, 4, 6, 8\}$; detector thresholds
$\tau \in \{0.15, 0.20, 0.30\}$; $\lambda = 1.0$ in the $J$ objective;
$200$-instance non-conflict control set (seed $0$); $500$-instance
caps for the per-model prefers-new and prefers-old ITI samples
(seed $0$); $\tau_{\mathrm{rec}} = -3$ for the knowledge-recovery
filter; bfloat16 weights for Mistral-7B-v0.3 and Llama-3.1-8B, fp16
for the two Qwens. Pre-trained model checkpoints are loaded from the
Hugging Face Hub through a model registry.

\paragraph{Steering vector storage.} Per-model steering vectors are
$\le 150$~KB. Detector activations are $\sim 8$~MB per model. Full
per-layer activation caches for the locator are $\sim 30$~MB per model.

\paragraph{Code and data release.}
The benchmark loader, per-model outputs, figure-generation scripts,
and inference-time steering code will be released with anonymous
submission artifacts or at camera-ready, subject to the venue's
anonymity policy. Hyperparameters, model checkpoints, and the per-model
$(\ell^*, \alpha^*, \tau^*)$ values are deterministic and reported
throughout. The five appendix figures
(\cref{fig:afr-profile,fig:detector-pr,fig:oracle-alpha,fig:relation-breakdown,fig:filter-capacity})
are each produced by a single released plotting script
that reads the per-model result files and emits the corresponding plot.

\paragraph{License and usage terms.}
The source code will be released under a permissive open-source
license, and the \PTC{} benchmark under a research-use license, both
intended strictly for research on parametric-knowledge analysis and
inference-time intervention. Precise terms, sourcing notes, and
known-coverage limitations will be documented on the release artefacts
themselves, which will be authoritative where they disagree with this
paper. Underlying language-model checkpoints (Qwen-2.5, Mistral-7B-v0.3,
Llama-3.1-8B) remain subject to their respective publisher licenses
and are not redistributed.
Wikidata-sourced facts retain the public-domain dedication
(CC0~1.0) of their original statements; only the \PTC{} record schema
and the verified filtering decisions on top of those facts are the
new contribution. Identifying release endpoints will be substituted at camera-ready.

\paragraph{Known engineering items.} The largest remaining engineering
item is detector calibration. The isotonic calibrator trained on a
small held-out fold of the per-model detector training set generalizes
imperfectly to the long tail of the full $8{,}746$-record benchmark,
leading to high false-alarm rates at aggressive thresholds for the
larger models. A proper held-out calibration set of $\sim$$1{,}000$
records is expected to substantially improve the operating-point
selection.

\paragraph{Auditability of activation edits (deployment guidance).}
We use the residual-stream edit purely as an analysis instrument, but the
mechanism it exercises has a dual-use character worth stating explicitly:
because the edit alters internal representations without retrieval or
provenance, the same operation that corrects an outdated fact can install a
false or adversarially supplied belief while leaving no external trace in
the output. Nothing in the forward pass marks an answer as having been
steered. Any deployment built on this mechanism should therefore (i) source
$\Delta$ only from trusted, auditable update streams, since whoever
controls $\Delta$ controls the installed belief; (ii) use conservative
detector thresholds, given that our detector is near chance under
subject-disjoint evaluation; and (iii) log every steering decision
(checkpoint hash, $\Delta$ version, layer $\ell^{*}$, strength $\alpha$,
detector score, and timestamp), keyed by a privacy-preserving hash of the
query rather than the raw prompt, so that steered answers remain
reconstructible after the fact without retaining user text.

\section{AI Usage Statement}
\label{app:ai-usage}

\paragraph{Use of AI Assistance.}
AI-based tools such as ChatGPT and Claude were used only to refine the manuscript, including grammar checking, clarity improvement, and language polishing. These tools were also used to support limited code refactoring and formatting improvements. All scientific findings, experimental results, analyses, interpretations, and conclusions presented in this work were developed and verified exclusively by the authors and were not generated by AI systems.

\paragraph{Benchmark.} The $8{,}746$-record benchmark was constructed
from public Wikidata statements following the deterministic protocol
in \cref{app:benchmark}. No AI system was used to generate, label,
verify, or filter benchmark records; LLM-judge cross-verification is
noted as a planned extension in the main paper.

\end{document}